\documentclass[a4paper, oneside]{discothesis}

\usepackage[utf8]{inputenc}
\usepackage[T1]{fontenc}

\usepackage{graphicx} 

\usepackage{geometry}     
\usepackage{amsthm}
\usepackage{amssymb, amsmath}
\usepackage{url}
\usepackage{hyperref}

\usepackage{pdfpages}
\usepackage{bbm}

\usepackage{subcaption}

\usepackage{tikz}
\usetikzlibrary{patterns}
\usetikzlibrary{datavisualization}

\usepackage{pgfplots}
\pgfplotsset{%
    compat=newest, 
    tick label style={font=\footnotesize},
    label style={font=\footnotesize},
    legend style={font=\small},
    axis x line = center,
    axis y line = center,
    }

\usepackage{algorithm}
\usepackage{algpseudocode}

\usepackage{comment}

\usepackage{float}

\newcommand{\MCE}{\operatorname{MisclassificationError_T}}
\newcommand{\AbE}{\operatorname{AbstentionError_T}}
\newcommand{\VCdim}{\operatorname{VCdim}}
\newcommand{\Ldim}{\operatorname{Ldim}}
\newcommand{\pos}{\operatorname{pos}}
\newcommand{\DIS}{\operatorname{DIS}}
\newcommand{\er}{\operatorname{er}}
\newcommand{\XS}{0.6cm}

\thesistype{Master's Thesis} 
\title{Adversarial Resilience against Clean-Label Attacks in Realizable and Noisy Settings}

\author{Carolin Christin Heinzler}
\email{cheinzler@ethz.ch}

\institute{Mathematics MSc\\
D-MATH IFOR\\[2pt]
ETH Zürich}

\supervisors{Prof. Dr. Afonso Bandeira,\\  Daniil Dmitriev}

\date{\today}

\begin{document}

\frontmatter 
\maketitle

\cleardoublepage

\begin{abstract}

We investigate the challenge of establishing stochastic-like guarantees when sequentially learning from a stream of i.i.d. data that includes an unknown quantity of clean-label adversarial samples. We permit the learner to abstain from making predictions when uncertain. The regret of the learner is measured in terms of misclassification and abstention error, where we allow the learner to abstain for free on adversarial injected samples.
This approach is based on the work of Goel, Hanneke, Moran, and Shetty from \cite{goel2023AdversarialResilienceSequential}. We explore the methods they present and manage to correct inaccuracies in their argumentation.\\
However, this approach is limited to the realizable setting, where labels are assigned according to some function $f^*$ from the hypothesis space $\mathcal{F}$. Based on similar arguments, we explore methods to make adaptations for the agnostic setting where labels are random. Introducing the notion of a clean-label adversary in the agnostic context, we are the first to give a theoretical analysis of a disagreement-based learner for thresholds, subject to a clean-label adversary with noise.

\end{abstract}

\tableofcontents

\mainmatter 

\chapter{Introduction} \label{sec:intro}

\section{Introduction}

How can we learn from data when it is impossible to determine if it has been corrupted? When all labels are correct, verifiable even by an expert? One might naively wonder what could go wrong in this case if all labels are indeed correct. However, one should not overlook that in Machine Learning, we often rely on the assumption of independence and an identical distribution (i.i.d.) of samples to leverage statistical properties and issue performance guarantees.\\
But what happens now when unknown parts of the data do \textit{not} come from the distribution we are assuming? Especially in the case of sequential prediction, when we want to classify one sample after another, it would not be feasible to perform a one-sample outlier detection every time we receive a new sample. This question of how to provide stochastic-like guarantees when sequentially learning from a stream of i.i.d. data, which contains an unknown amount of clean-label adversarial samples, will be the topic of this thesis.\\

In a first step, we will focus on the realizable setting, where we can assume that samples are labeled according to some hypothesis $f^*$ from a binary function space $\mathcal{F}$. This setting has been studied by Goel, Hanneke, Moran, and Shetty in their 2023 paper \cite{goel2023AdversarialResilienceSequential}, which can be considered the basis and inspiration of this thesis. They derive sequential learning protocols that are resilient against a clean-label adversary, in the sense that we can recover stochastic-like error guarantees. The approach they propose relies on disagreement-based classification and shattering and takes inspiration from active learning. They manage to give results for binary function classes $\mathcal{F}$ with finite VC dimension. Assuming a fully adversarial stream of data would result in error bounds based on the Littlestone dimension instead of the VC dimension of the problem class. Therefore, their results can be seen as a beyond-worst case analysis of error.\\
One important aspect of their algorithm is that the learner is allowed to abstain from predicting in case he is not certain of his prediction, taking inspiration from selective classification. Having certainty associated with making a prediction is helpful for the learner, in order to have the possibility to learn from his mistakes. Furthermore, in real-life applications, particularly in critical domains such as medicine or self-driving cars, opting for an 'I don't know' prediction is preferable over making a wrong prediction with low certainty.\\

We build on these results and extend them to the agnostic setting, where not only samples but also labels are considered random. This scenario is more challenging to analyze; however, it is a better reflection of real-life data, allowing for noisy observations, a misspecified model, or general random labels. We find a novel way to define clean-label adversaries in this agnostic setting. To our knowledge, we are the first to study a clean-label adversary injecting samples in a classification task subject to noise from a theoretical perspective. As an introduction to this setting, we derive a disagreement-based learner for the class of threshold functions on an interval. Again, we manage to recover stochastic-like error guarantees. The primary challenge in this scenario is deciding when to update our knowledge, given the potential noise in individual labels. In the realizable setting, after observing each new labeled sample, we can be sure that we have learned something new about the target. However, the presence of noisy labels introduces an additional layer of uncertainty, as we cannot rely on individual labels.\\
Lastly, we explore ways to extend the results from \cite{goel2023AdversarialResilienceSequential} to go beyond disagreement-based methods when learning in the presence of noise.

\section{Thesis Overview}
In the first part of the thesis, we will rigorously introduce the setting of sequential prediction with adversarial clean-label injections in the realizable case. Furthermore, we will give a thorough overview of research fields in Section \ref{sec:lit}, which are connected to and have inspired the work of Goel et al. We will discuss learning in the stochastic and adversarial setting, as well as active learning and selective classification before providing a brief overview of the research conducted on clean-label adversaries. With this context, in Section \ref{sec:results_real}, we introduce the learning algorithms of \cite{goel2023AdversarialResilienceSequential} and give guarantees on expected misclassification and abstention error. We partially reformulate the proofs given in \cite{goel2023AdversarialResilienceSequential} to account for inaccuracies, and we present the revised version of the algorithm based on our suggestions. \\

The second part of the thesis focuses on the agnostic setting. We introduce the new concept of a clean-label adversary when labels are random. In Section \ref{sec:lit_agno}, we provide an overview of research in agnostic learning, specifically in agnostic active learning, which serves as an inspiration for our approach. We manage to derive a disagreement-based learner in the case of Random Classification Noise in Section \ref{sec:results_agno}, and we prove error guarantees on expected misclassification and abstention error. Lastly, we explore ways in which we can move beyond disagreement-based learning and adopt the techniques used in \cite{goel2023AdversarialResilienceSequential}. We pinpoint the challenges that emerge for this approach in the agnostic setting. To conclude, we outline potential steps to build a broader theory for clean-label adversaries in the agnostic setting.

\newpage

\chapter{Realizable Setting}

\section{Literature for the Realizable Setting} \label{sec:lit}

\subsection{Setting}\label{sub_sec:setting_realizable}

First, we consider the general case of sequentially learning some binary hypothesis. Let $\mathcal{X}\neq\emptyset$ be the instance space, equipped with a $\sigma$-algebra which defines measurable subsets and $\mathcal{Y}=\{0,1\}$ the corresponding label space. The goal of the learning problem is to find the best possible classifier from the hypothesis space $\mathcal{F}=\{f:\mathcal{X}\rightarrow \mathcal{Y}: f \text{ measurable}\}$ through samples $(x_i,y_i)\in \mathcal{X}\times\mathcal{Y}$. In the statistical model, data points $x_i$ are drawn i.i.d. from some fixed but generally unknown distribution $\mathcal{D}$ over $\mathcal{X}$ and are assigned a label according to a true labelling function $f^*\in \mathcal{F}$. We call this setting the \textit{realizable case}, as there exists a function which can deterministically assign the true label, given the point $x_i$. In Chapter \ref{chap:agno} we will focus on a generalization of this setting, which is called the \textit{agnostic case}. There, the labels are drawn at random from some conditional distribution given $x_i$.\\

We consider a time horizon $T\in\mathbb{N}^+$. In classical online (also referred to as sequential) learning, the algorithm receives at time $t\in \{1,\dots,T\}$ the new data point $x_t$ and based on the previously observed set of points $S_{t-1}=\{(x_i,y_i):1\leq i \leq t-1\}$ makes a prediction $\hat{y}_t$ for its label. Subsequently, the true label $f^*(x_t)=y_t$ is revealed and some loss might be incurred, depending on whether the prediction was correct or not. Then, $(x_t,y_t)$ is added to $S_{t-1}$ to inform the next prediction on $x_{t+1}$.\\
Any hypothesis $f\in \mathcal{F}$ can be assigned a corresponding true error $\er_{\mathcal{D}}(f)$ and an empirical error over a set $S_T$ of labeled points $\er_{S_T}(f)$, which are defined as follows
$$\text{true error }\er_{\mathcal{D}}(f):=\underset{X\sim \mathcal{D}}{P}(f(X)\neq f^*(X))\qquad \text{empirical error }\er_{S_T}(f):=\frac{1}{T}\sum_{t=1}^T \mathbbm{1}_{\{f(x_i)\neq y_i\}}.$$
Throughout this thesis, we will only consider the error with respect to the $0-1$ loss. 
The objective of the learning algorithm can be formulated such as to find a classifier $\Tilde{f}\in \mathcal{F}$ which minimizes the true error, this is called the true risk minimizer. In the realizable case, the true labelling function $f^*\in \mathcal{F}$ achieves the minimum with $\er_{\mathcal{D}}(f^*)=0$, as the label of every point $x\in\mathcal{X}$ can be determined trough $f^*\in \mathcal{F}$ and is given by $f^*(x)=y$ (with probability one). Given a data set $S_T$ we can calculate the corresponding empirical risk minimizer (ERM) $\hat{f}$:
$$\text{true risk minimizer }\Tilde{f}:= \underset{f\in \mathcal{F}}{\arg \inf} \er_{\mathcal{D}}(f)\qquad \text{empirical risk minimizer }\hat{f}:= \underset{f\in \mathcal{F}}{\arg \inf} \er_{S_T}(f).$$

In the following, we will adopt different methods than an ERM approach. We consider an adversary, trying to corrupt this learning task. More specifically, we will allow for clean-label injections: an adversary is allowed to modify the input $x_t\in \mathcal{X}$ at time $t$ without having access to the true labelling function/distribution, i.e. an adversary can sample from a different unknown distribution in every round, and the injected samples are not guaranteed to be independent.  This means we can no longer assume that the sample $x_t\in\mathcal{X}$ presented to the learner was drawn i.i.d..
However, the label $y_t\in\mathcal{Y}$, which is given to the learner afterwards, is generated according to the correct labelling function $f^*\in \mathcal{F}$. Note that the learner is not informed about which examples are adversarial (also not in hindsight) and there is no bound on the number of adversarial injections of this type. However, the adversary can choose in each round $t$ whether to inject an example that he chose (from an unknown distribution and not i.i.d.) or whether not to inject and present an i.i.d. example to the learner. The adversary makes his choice, without having observed the i.i.d. sample. \\
As a consequence, the learner cannot impose any distributional assumption on the observed samples. We will denote the sample presented to the learner in round $t$ by $\hat{x}_t$ to stress that we no longer can be sure, whether it is an uncorrupted $x_t$ drawn from $\mathcal{D}$ or whether it is an injection by an adversary.\\

The original paper \cite{goel2023AdversarialResilienceSequential} addresses two settings for the realizable case: first, the case for which the learner knows the true distribution $\mathcal{D}$ of samples from $\mathcal{X}$. This will help to make correct predictions and evaluate if an example might have come from an adversary. Secondly, we study the case when there is no information about the distribution of samples $x\in \mathcal{X}$. Then one needs to rely on structural assumptions of the learning problem. It is important to note that in both scenarios, the learner does not know the true labelling function; the objective is to acquire this knowledge through the learning process.\\

In order to improve performance in this setting of clean-label injections by an adversary, the algorithm is allowed to abstain on examples where it is uncertain. In literature this concept is known as the classification with a reject option or selective classification.\\
The predictions of the algorithm are within the set $\{\perp,0,1\}$, where the symbol $\perp$ indicates an abstention from prediction. However the problem is different compared to multiclass classification, as no point has an inherent label $\perp$.

\subsubsection*{Objective}

The performance of our algorithm is evaluated based on two different types of error. On one hand, whenever the learner chooses to predict, there is a possibility that the prediction is wrong. In order to ensure correct predictions, the learner suffers a \textit{misclassification error} for a wrong prediction in $\{0,1\}$. On the other hand, the learner is discouraged from abstaining excessively on i.i.d. samples, and is subject to an \textit{abstention error} for an abstention on what would have been an i.i.d. sample. Without this penalty, a trivial strategy to minimize classification error would be to always abstain. Consequently, we infer a zero-one loss in both scenarios and define the $\MCE$ and $\AbE$ at time $T$ accordingly
\begin{align*}
    \MCE&= \sum_{t=1}^T \mathbbm{1}_{\{\hat{y}_t=1-y_t\}}\\
    \AbE&= \sum_{t=1}^T\mathbbm{1}_{\{x_t \text{ was drawn i.i.d. from }\mathcal{D} \text{ and } \hat{y}_t=\perp\}}
\end{align*}
The objective is to minimize the sum of these two and comprehend the trade-off between misclassification and abstention error. Note that abstention does not suffer a cost in general (except for i.i.d. samples) such as to account for an arbitrary amount of injections without incurring a linear cost in the error for abstaining. \\
Furthermore, an interesting aspect to note is that the learner itself is not aware of how much error he suffers. This is due to the fact that it is never revealed which samples were adversarial and which were not, thus making the calculation of the true abstention error not possible for the learner. Theoretically one approach could be to reveal to the learner if it was an adversarial sample only after an abstention. A strategy could then be to abstain as long as a large enough sample of $n\in\mathbb{N}$ i.i.d. samples is collected and estimate the underlying distribution if it is not already known. The abstention error until this point would be $n$, as one only pays for abstentions on non-adversarial samples. In case the distribution is unknown, we will nevertheless adopt a different approach that is distribution-free and relies on structural properties of the hypothesis class.

\begin{definition}[Shattering and VC dimension \cite{vladimirvapnik1982EstimationDependencesBased,vapnik1971UniformConvergenceRelative}]
    Let $\mathcal{F}$ be a binary function class over $\mathcal{X}$ and $d\in \mathbb{N}$.\\
    The class $\mathcal{F}$ is said to shatter a set of $d$ points $\{x_1,\dots,x_d\}\subset\mathcal{X}$ if for all $y\in \{0,1\}^d$ there exists an $f\in \mathcal{F}$ such that $f(x_i)=y_i$ for all $1\leq i\leq d$. We define the Vapnik-Chernovenkis (VC) dimension of $\mathcal{F}$ to be the cardinality $d$ of the largest set in $\mathcal{X}$ that is shattered by $\mathcal{F}$ and we write $\VCdim(\mathcal{F})=d$.
\end{definition}

\begin{definition}[Littlestone dimension \cite{littlestone1988LearningQuicklyWhen}] \label{def:little_dim}
    Let $\mathcal{F}$ be a binary function class over $\mathcal{X}$ and $\ell\in \mathbb{N}$.\\
    First, we define a mistake tree by a full binary decision tree of depth $\ell$ with internal nodes labeled with elements of $\mathcal{X}$. We associate every root to leaf path in this mistake tree by a sequence of $((x_i,y_i))_{i=1}^\ell$ by assigning every element $x_i$ a label $y_i\in \{0,1\}$ depending if it is the left ($y_i=0$) or right ($y_i=1$) child of its parent.\\
    The function class $\mathcal{F}$ is said to shatter a mistake tree of depth $\ell$, if for any root to leaf path $((x_i,y_i))_{i=1}^\ell$, there exits an $f\in \mathcal{F}$ such that $f(x_i)=y_i$ for all $1\leq i\leq \ell$. We define the Littlestone dimension of $\mathcal{F}$ to be largest depth $l$ of a mistake tree which is shattered by $\mathcal{F}$ and we write $\Ldim(\mathcal{F})=\ell$.
\end{definition}

\begin{definition}[Set of shattered $k$-sets]
    Let $\mathcal{F}$ be a binary function class over $\mathcal{X}$ and $k\in \mathbb{N}^+$.\\
    The set of shattered $k$-sets $\mathcal{S}_k$ is defined as all subsets of $\mathcal{X}$ of size $k$, which are shattered by $\mathcal{F}$:
    $$\mathcal{S}_k=\{\{x_1,\dots, x_k\}\subset \mathcal{X}: \mathcal{F} \text{ shatters } \{x_1,\dots, x_k\}\}.$$
    Furthermore, let $\mathcal{D}$ be a probability distribution over $\mathcal{X}$.\\
    The $k$ shattering probability of $\mathcal{F}$ with respect to $\mathcal{D}$, $\rho_k(\mathcal{F}, \mathcal{D})$, is defined as 
    $$\rho_k(\mathcal{F}, \mathcal{D})=\mathcal{D}^{\otimes k}(\mathcal{S}_k(\mathcal{F}))= P_{\{x_1,\dots, x_k\}\sim \mathcal{D}^{\otimes k}}(\{x_1,\dots, x_k\} \text{ is shattered by } \mathcal{F}).$$
    In case the distribution $\mathcal{D}$ is clear from the context, we just write $\rho_k(\mathcal{F})$.
\end{definition}
The set of shattered $k$-sets for $k=1$ is of particular interest and deserves its own name.
\begin{definition}[Disagreement region]
    Let $\mathcal{F}$ be a binary function class over $\mathcal{X}$.\\
    The disagreement region of $\mathcal{F}$ is defined as the set of all points $x\in \mathcal{X}$ which can be shattered by $\mathcal{F}$:
    $$\mathcal{S}_1(\mathcal{F})= \{x\in \mathcal{X}: \exists f,g\in \mathcal{F} \text{ such that } f(x)\neq g(x)\}.$$
    We will also refer to this quantity as $\DIS(\mathcal{F})$.
    Furthermore, let $\mathcal{D}$ be a probability distribution over $\mathcal{X}$.\\
    Analogously, we define the disagreement probability of $\mathcal{F}$ with respect to $\mathcal{D}$ as 
    $$\rho_1(\mathcal{F},\mathcal{D})=\mathcal{D}(\mathcal{S}_1(\mathcal{F}))=P_{x\sim \mathcal{D}}(\exists f,g \text{ such that } f(x)\neq g(x))=P_{x\sim \mathcal{D}}(x\in \DIS(\mathcal{F})).$$
    We write $\rho_1(\mathcal{F})$ if the distribution $\mathcal{D}$ is clear from the context.
\end{definition}

\subsubsection*{Notation}

We explain some notation specific for this thesis. The set of natural numbers $\mathbb{N}$ contains 0, whilst $\mathbb{N}^+$ only contains positive integers. Similarly for $\mathbb{R}^+=(0,\infty)$.  \\
Let $\mathcal{X}$ be the input space and $\mathcal{D}$ a probability distribution over $\mathcal{X}$. We define the $k$ dimensional product measure to be $\mathcal{D}^{\otimes k}$, i.e. for $A\subseteq\mathcal{X}^k$ then $\mathcal{D}^{\otimes k}(A)=P((X_1,\dots,X_k)\in A)$ for $X_1,\dots, X_k$ i.i.d. random variables drawn from $\mathcal{D}$.
Given a set $S_t=\{(x_1,y_1),\dots,(x_t,y_t)\}$, we write $\mathcal{F}|_{S_t}=\{f\in\mathcal{F}:f(x_i)=y_i \text{ for all } 1\leq i\leq t\}$ to define the restriction of the function class $\mathcal{F}$ to the set $S_t$. Whenever we want to restrict $\mathcal{F}$ to a single point $S=\{(x,y)\}$, we can also write $\mathcal{F}|_S=\mathcal{F}^{x\rightarrow y}$. For a vector $x\in\mathbb{R}^n$, when writing $(x)_i$, we refer to the $i$-th coordinate of $x$. While we typically address the binary classification task with labels in $\mathcal{Y}=\{0,1\}$, we sometimes refer to a sample with the label 0 as a 'negative' sample and as a 'positive' sample when the label is 1. We define the indicator function  $\mathbbm{1}_{\{A\}}:\mathcal{X}\rightarrow \{0,1\}$ for an event $A\subseteq\mathcal{X}$ to be the following
\begin{equation*}
    \mathbbm{1}_{\{x\in A\}}=\begin{cases}
    1 & x\in A \\
    0 & x\notin A.
    \end{cases}
\end{equation*} 
Based on this, one can define the $0-1$ loss as a function $\ell_{0,1}:\mathcal{X}\times \mathcal{Y}\rightarrow\{0,1\}$ as $\ell_{0,1}(x,y)=\mathbbm{1}_{\{x\neq y\}}$.\\
Lastly, we use Landau notation, to describe asymptotics of a function $f$, in terms of another function $g$ where $f(x)\in\mathcal{O}(g(x))$ if $\lim\sup _{x\rightarrow\infty}\frac{|f(x)|}{g(x)}< \infty$ which is equivalent to $g(x)\in\Omega(f(x))$, and
$f(x)\in o(g(x))$ if $\lim_{x \rightarrow \infty} \frac{f(x)}{g(x)}=0$.



\subsection{Related Work}

\subsubsection{On Learning}\label{sec:learning}
For a first notion of what it means to learn, we look at PAC-learning. The cornerstone for this theory was first developed by Vladimir Vapnik and Alexey Chervonenkis. Together they were the founding fathers of a whole new subbranch of statistical learning theory between 1960 and 1990: the VC Theory (for Vapnik Chernovenkis).
Their fundamental work focused on establishing uniform laws of large numbers and high probability bounds for empirical risk \cite{vladimirvapnik1964ClassAlgorithmsPattern, vapnik1971UniformConvergenceRelative,vladimirvapnik1974TheoryPatternRecognition, vladimirvapnik1982EstimationDependencesBased}.\\
This so called \textit{uniform convergence} property justifies learning through empirical risk minimization (ERM) if the hypothesis class is not too rich (exactly for bounded VC dimension \cite{vapnik1971UniformConvergenceRelative}). Their results encapsulate fundamental properties of what one would associate with learning: when the goal is to discover general rules from observed data, we seek for our rules to generalize well, even to data we have not observed whilst for the process or learning we wish to be consistent, i.e. the more data we observe, the better we wish to approximate the true rules \cite{vapnik2000NatureStatisticalLearning}.\\

These principles were put in the context of learning by Valiant in 1984 \cite{valiant1984TheoryLearnable}, who introduced the notion of PAC-learning. By making minimal assumptions, he was able to put learning in the most general context, to establish distribution-free error bounds. For this, we first rigorously define what a learner is in statistical learning theory and we state the definition of PAC learnable:\\

\begin{definition}[Learner or Learning Algorithm]
    A learner $L$ (or for a less human version: a learning algorithm) is a mapping from $\bigcup_{k=1}^\infty (\mathcal{X}\times \mathcal{Y})^k\rightarrow \mathcal{F}$ which takes a sequence of any length of points and labels and maps it to a function in $\mathcal{F}$.
\end{definition}
\begin{definition}[PAC learnable (from \cite{shalev-shwartz2014UnderstandingMachineLearning})]\label{def:pac}
    Let $\mathcal{F}$ be a binary hypothesis class over $\mathcal{X}$ and $\mathcal{D}$ a probability distribution over $\mathcal{X}$.\\
    We call a hypothesis class $\mathcal{F}$ Probably Approximately Correct (PAC) learnable, if there exists a function $m_\mathcal{F}:(0,1)\times (0,1)\rightarrow \mathbb{N}$ and a learner $L$ such that the following holds:
    For any $\varepsilon, \delta\in (0,1)$ and any probability distribution $\mathcal{D}$ over $\mathcal{X}$ and every labelling function $f\in \mathcal{F}$, when running the learning algorithm on $m \geq m_\mathcal{F}(\varepsilon,\delta)$ i.i.d. samples drawn from $\mathcal{D}$ and labeled by $f$, the algorithm returns a hypothesis $h$. For this $h$ we have with probability of at least $1-\delta$ (over the randomness of the samples), that for the true risk the following bound holds: $\er_{\mathcal{D}}(h)\leq \varepsilon$. \\
    The function $m_\mathcal{F}$ is referred to as sample complexity for the learning algorithm $L$.
\end{definition}
Note that the bound has to hold for \textit{any} probability distribution $\mathcal{D}$, this can be interpreted a worst-case error. Furthermore, we usually assume that the data which is presented to the learner, is drawn i.i.d. from a distribution $\mathcal{D}$. We will refer to this as the \textit{stochastic} setting. \\

In 1989 Blumer, Ehrenfeucht, Warmuth and Haussler \cite{blumer1989LearnabilityVapnikChervonenkisDimension} put the VC dimension and uniform convergence results from Vapnik and Chernovenkis into the context of Valiant's PAC learning. This in turn lead to the following theorem which provides a unique characterization of those binary function classes which are PAC-learnable. It combines results on PAC-learnability, VC-dimension, uniform convergence and ERM:

\begin{theorem}[Fundamental Theorem of Statistical Learning Part I]\label{theo:fund_stat_learn}

 Let $\mathcal{F}$ be a binary hypothesis class over $\mathcal{X}$. Then the following are equivalent:
\begin{enumerate}
    \item The hypothesis class $\mathcal{F}$ is PAC-learnable
    \item The hypothesis class $\mathcal{F}$ has the uniform convergence property
    \item An ERM rule is a successful PAC learner for $\mathcal{F}$
    \item The hypothesis class $\mathcal{F}$ has finite VC dimension $d$
\end{enumerate}
Furthermore, if $\mathcal{F}$ is PAC-learnable, we have for the sample complexity there exist absolute constants $C_1,C_2$ such that:
$$C_1\frac{d+log(1/\delta)}{\varepsilon}\leq 
m_\mathcal{F}\leq C_2\frac{d\log(1/\varepsilon)+\log(1/\delta)}{\varepsilon}.$$
\end{theorem}

It is important to note that this holds only for the problem of binary classification. For example, when studying private learning, a finite VC dimension does not imply learnability \cite{alon2019PrivatePACLearning}.\\

The upper and lower bounds on sample complexity have been established by \cite{vapnik1971UniformConvergenceRelative, ehrenfeucht1989GeneralLowerBound,haussler1994PredictingFunctionsRandomly}, thus characterizing the rate of convergence of the true error of the hypothesis $\mathbb{E}(\er_{\mathcal{D}}(h))$ tending to 0 as $m\rightarrow\infty$ in the PAC model. Note that convergence is at most linear in $\frac{d}{m}$ where $m$ is the number of samples if $\VCdim(\mathcal{F})=d<\infty$ or the true error is bounded away from 0 for infinite VC dimension. \\

With these foundations of PAC-learnability and uniform convergence in mind, we now move on to a characterisation of learnability for binary classification in the online setting. Instead of receiving one batch of samples and finding the best classifier to describe the data, we will receive new information in a stream and wish to continuously improve our knowledge. Furthermore, we are tasked with predicting on every incoming sample and will only receive the true label afterwards. The goal of learning in this setting is to make the fewest mistakes possible. This was first rigorously studied in 1988 by Nick Littlestone \cite{littlestone1988LearningQuicklyWhen}. Note that in this setting, we make no assumption on the origin of the samples or labels (no i.i.d. assumption), and thus the quantity we study is the maximum number of mistakes a learner might incur. Again, we study the worst-case behavior in terms of the maximum number of mistakes an algorithm will make.

\begin{definition}[Online learnability (see \cite{shalev-shwartz2014UnderstandingMachineLearning}]
Let $\mathcal{F}$ be a binary function class over the domain $\mathcal{X}$ and $T\in \mathbb{N}^+$.\\
Define $M_L(S)$ to be the number of mistakes an online learning algorithm $L$ makes on any given sequence $S=((x_1,f^*(x_1)),\dots, (x_T,f^*(x_T)))$ of labeled samples for some $f^*\in \mathcal{F}$. We define $M_L(\mathcal{F})$ to be the maximum over all sequences of this form. Furthermore, all $B<\infty$ which upper bound $M_A(\mathcal{F})$ are called a mistake bound. \\
A hypothesis class $\mathcal{F}$ is said to be online learnable, if there exists a learner $L$ for which such a mistake bound is satisfied, i.e. $M_A(\mathcal{F})\leq B < \infty$ .
\end{definition}

Similar to Theorem \ref{theo:fund_stat_learn} characterizing PAC-learnability through the VC dimension of the hypothesis class $\mathcal{F}$, online-learnability can be established through a similar combinatorial property: the Littlestone dimension of $\mathcal{F}$ (see Definition \ref{def:little_dim}). A fundamental result proven by Littlestone \cite{littlestone1988LearningQuicklyWhen} says that the optimal mistake bound in online learning, is exactly the Littlestone dimension. This combinatorial property thus characterizes learnability (if it is finite), whilst exactly quantifying the mistake bound. 

\begin{remark}
    It is easy to see that $\VCdim(\mathcal{F})\leq\Ldim(\mathcal{F})$. Furthermore, note that for example the class thresholds on the interval $[0,1]$ has VC dimension 1 whereas the Littlestone dimension is already infinite.
\end{remark}

Lastly, to give an even more general viewpoint on learning, we discuss a more recent addition to the field of learning theory. We introduce the notion of \textit{Universal Learning} which was rigorously introduced by Bousquet et al. in 2021
\cite{bousquet2021TheoryUniversalLearning}. In learning theory, and for the context of the PAC model, research interest lies in understanding the behaviour of the true error for the classifier $h\in \mathcal{F}$ which the learner produces after observing $m$ i.i.d. samples. The Fundamental Theorem of Statistical Learning (see Theorem \ref{theo:fund_stat_learn}) established at most a linear rate of convergence in the PAC model. In this setting, the goal is to establish distribution-free error bounds, which gives bounds for the worst case error over all distributions.\\
For some learning methods however, exponential error decay has been observed in practice and under low noise assumptions these rates can be proven (e.g. for stochastic gradient descent and kernel methods \cite{koltchinskii2005ExponentialConvergenceRates,audibert2007FastLearningRates,pillaud-vivien2018ExponentialConvergenceTesting}), suggesting that the distribution-free bounds might be too pessimistic to give inference on the behaviour of empirical error in applications.\\

The new approach of universal learning is to move beyond this worst-case analysis in the distribution-free setting, where we considered rates uniform in the distribution $\mathcal{D}$. This new viewpoint weakens this uniformity assumption and instead requires universality. This means that a given property (like consistency) holds for all distributions $\mathcal{D}$, however not uniformly over all distributions. 

\begin{definition}[Universal learnability \cite{bousquet2021TheoryUniversalLearning}]
    Let $\mathcal{F}$ be a binary function class over $\mathcal{X}$ and $R:\mathbb{N}\rightarrow [0,1]$ a function with $\er_{\mathcal{D}}(n)\rightarrow 0$ as $n\rightarrow\infty$.\\
    Then $\mathcal{F}$ is uniformly learnable at rate $R$, if there exists a learning algorithm $L$ who after observing $n$ i.i.d. samples outputs a hypothesis $h\in \mathcal{F}$ such that for every distribution $\mathcal{D}$ over $\mathcal{X}$, there exist $C,c>0$ for which $\mathbb{E}(\er_{\mathcal{D}}(h))\leq C\cdot \er_{\mathcal{D}}(cn)$ for all $n\in \mathbb{N}$.
\end{definition}

Similarly to the characterization before for PAC-learnability (i.e. linear rates) in case of finite VC dimension and online learnability for finite Littlestone dimension, universal learnability admits a similar characterization of convergence rates through combinatorial  properties of the function class $\mathcal{F}$. More precisely, there is exactly three possible rates of convergence: exponential, linear, and arbitrarily slow rates. One can characterize which rate is possible for a certain function class, depending on finitness of new combinatorial properties they introduced, called Littlestone tree and VC-Littlestone dimension. For further details on this new approach we refer to \cite{bousquet2021TheoryUniversalLearning}.\\

\subsubsection{On Active Learning}\label{sec:active learning}
To highlight another area of learning theory that shares certain parallels with the results we will present later, we introduce active learning. The setting of \textit{pool-based} active learning has been studied in great detail by Steve Hanneke \cite{ hanneke2009TheoreticalFoundationsActive,hanneke2012ActivizedLearningTransforming,hanneke2014TheoryDisagreementBasedActive}. In this setting, the learner has access to a collection of unlabeled examples (i.e. a pool). The learner can sequentially select samples from this pool, requesting their labels until a predefined limit on the number of label requests is reached.\\
Given this collection of unlabeled data together with selected labeled examples, the learner's goal is to find a hypothesis with minimal error on unseen data. 
This stands contrary to passive learning (hence the name active), where the learner is supplied with an entire set of labeled examples from the beginning. Note that we have two problems at hand: first, we need to identify the most informative samples to request their labels and second, based on this we need to find the best classifier.\\

Extending the pool-based setting to the case when the learner receives samples one after another (with no control over the stream), one can also study sequential/stream-based active learning  \cite{cohn1994ImprovingGeneralizationActive,freund1997SelectiveSamplingUsing, dekel2012SelectiveSamplingActive} and online selective sampling \cite{huang2020DisagreementbasedActiveLearning}. 
Similar to the online setting discussed earlier in Section \ref{sec:learning}, unlabeled samples are provided sequentially, and the learner will generate in each round a prediction for their labels. 
However, instead of being shown the true label after every prediction, the learner must now decide whether or not he wants to observe the true label. In case of a label request, we incur a cost.
The perspective on the objective of the learner is slightly different between stream-based and online selective sampling \cite{hanneke2021GeneralTheoryOnline}: 
\begin{itemize}
    \item In stream-based active learning the objective is again for a given budget of label requests to guarantee to find a classifier with low error on unseen data (note that the learner does not need to give predictions here).
    \item In online selective sampling, the goal is to provide guarantees on the expected number of label requests and the expected number of misclassified predictions after $T$ samples have been processed. One can also study the trade-off between these two quantities \cite{hanneke2021GeneralTheoryOnline}.
\end{itemize}

We define the concept of label complexity, an important quantity in pool-based and stream-based active learning. We adopt the notion of unverifiable label complexity, as first presented in \cite{balcan2010TrueSampleComplexity} and highlight that in literature there exist variations of this concept. 
\begin{definition}[Label complexity \cite{hanneke2009TheoreticalFoundationsActive}]
    Let $\mathcal{F}$ be a binary function space over $\mathcal{X}$. \\
    We say an active learning algorithm $\mathcal{A}$ achieves label complexity $\Lambda(\cdot,\cdot,\cdot,\cdot)$ for confidence bound error, if for every true labelling function $f^*\in \mathcal{F}$, for every $\varepsilon,\delta >0$, every distribution $\mathcal{D}$ over $\mathcal{X}$ and every integer $n\geq\Lambda(\varepsilon,\delta,f^*,\mathcal{D})$, then by running $\mathcal{A}$ with a budget of $n$ label requests, we can return a classifier $f\in \mathcal{F}$ such that with probability at least $1-\delta$, we have $\er_{\mathcal{D}}(f)\leq \varepsilon$.\\

    
\end{definition}

There exist different strategies when deciding whether to request a label or not. We will focus on disagreement-based active learning \cite{cohn1994ImprovingGeneralizationActive,balcan2006AgnosticActiveLearning, hanneke2014TheoryDisagreementBasedActive}, as this shows strong parallels to the sequential learning algorithm we will implement. However, there exist also different approaches to active learning, which we will not discuss further, like query by committee (QBC) algorithms \cite{freund1997SelectiveSamplingUsing} or importance weighing algorithms \cite{beygelzimer2009ImportanceWeightedActive}.\\
The first algorithm employing a disagreement-based approach in the realizable setting was outlined by Cohn, Atlas, and Ladner in 1994, commonly referred to as CAL (see Algorithm \ref{alg:cal}), after the initials of its developers \cite{cohn1994ImprovingGeneralizationActive}. The algorithm's strength lies in its simplicity: we request the label of the sample if it cannot already be inferred from existing knowledge. In a sense this is the most basic idea of what an active learning algorithm should be capable of: if there is already enough information available, we do not request a label. However, we make no further attempt of guessing if we cannot be sure. \\

\begin{algorithm}
    \caption{CAL Algorithm: disagreement-based active learning in the realizable setting}\label{alg:cal}
    \begin{algorithmic}[1]
        \State Let $\mathcal{F}_0=\mathcal{F}$, $i=0$ and $n$ is the available label budget
        \For{$t=1,\dots$}
        \State Receive sample $x_t$
        \If{$x_t\in \DIS(\mathcal{F}_{i-1})$}
        \State Request true label $y_t$ and update $i\gets i+1$
        \State Update version space $\mathcal{F}_i\gets \mathcal{F}_{i-1}^{x_t\rightarrow y_t}$ 
        \If{$i=n$} 
        \State Return any $\hat{f}_n\in \mathcal{F}_n$
        \EndIf
        \EndIf
        \EndFor
    \end{algorithmic}
\end{algorithm}

When studying these kinds of active learning algorithm, literature is focused on establishing tight lower and upper bounds on the label complexity of an active learning algorithm, in order to understand rates of convergence for the true error. See for example Hanneke's work \cite{hanneke2009TheoreticalFoundationsActive} which majorly advanced this field of research  by introducing a novel complexity measure called \textit{disagreement coefficient}. This quantity plays a crucial role in characterizing worst-case convergence rates for active learning algorithms. Generally it is also of interest to study the cases for which active learning provides an improvement over passive learning. A breakthrough result by \cite{balcan2010TrueSampleComplexity} showed that active learning always provides asymptotically better label complexity, even achieving exponential rates.  
Building on this, Hanneke \cite{hanneke2009TheoreticalFoundationsActive, hanneke2012ActivizedLearningTransforming} provides a framework, which can transform almost any passive learning algorithm into an active learning framework, with asymptotically improved label complexity. \\

The concepts stemming from two contributions of Hanneke \cite{hanneke2009TheoreticalFoundationsActive,hanneke2012ActivizedLearningTransforming} will lend inspiration to the algorithms we will implement later in Section \ref{sec:results_real} and \ref{sec:results_agno}. Among other ideas, he developed so called acitivizers for passive learning algorithms in order to (asymptotically) strictly improve label complexity bounds when compared to the respective sample complexity of the passive learning algorithm. In the realizable setting, he managed to further improve label complexity by extending beyond the concept of disagreement to considering the shattering of more than one point. In order to understand this method better, we will focus on his results and this concept of shattering $k$-points. \\

First, we define the disagreement coefficient, which was introduced by Hanneke in 2007  \cite{hanneke2007BoundLabelComplexity}. This quantity will play an important role in classifying label complexity in active learning and selective classification problems. We also introduce the notion of the disagreement core, which is related to quantities described by Hanneke in \cite{hanneke2007BoundLabelComplexity, hanneke2011RatesConvergenceActive} and was first introduced by \cite{balcan2010TrueSampleComplexity} under the name of boundary.
\begin{definition}[Disagreement Coefficient  \cite{hanneke2012ActivizedLearningTransforming}]\label{def:disagree_coeff}
    Let $\varepsilon>0$ and $\mathcal{F}$ a binary function space over $\mathcal{X}$. For $r\in [0,1]$ and a measurable function $f:\mathcal{X}\rightarrow\mathcal{Y}$ we define $\mathcal{B}_{\mathcal{F},\mathcal{D}}(f,r)=\{h\in \mathcal{F}: P_{x\sim \mathcal{D}}(f(x)\neq h(x))\leq r\}$. If $\mathcal{F}$ and $\mathcal{D}$ are clear from the context, we will abbreviate $\mathcal{B}(f,r)=\mathcal{B}_{\mathcal{F},\mathcal{D}}(f,r)$\\
    The disagreement coefficient of a classifier $f\in \mathcal{F}$ with respect to the distribution $\mathcal{D}$ over $\mathcal{X}$ is defined as 
    $$\theta_f(\varepsilon)=\sup_{r>\varepsilon}\frac{P_{x\sim\mathcal{D}}(x\in \DIS(\mathcal{B}_{\mathcal{F},\mathcal{D}}(f,r))}{r}\vee 1.$$
    We write $\theta_f=\theta_f(0)$.
\end{definition}

\begin{definition}[Disagreement Core \cite{hanneke2012ActivizedLearningTransforming}]
    Let $\mathcal{F}$ be a binary function class over $\mathcal{X}$.\\
    We define the disagreement core with respect to $\mathcal{F}$ and distribution $\mathcal{D}$ over $\mathcal{X}$ as 
    $$\partial_{\mathcal{F},\mathcal{D}}f=\lim_{r\rightarrow0}\DIS(\mathcal{B}_{\mathcal{F},\mathcal{D}}(f,r)).$$
    If $\mathcal{F}$ and $\mathcal{D}$ are clear from the context, we will equivalently write $\partial f=\partial_{\mathcal{F},\mathcal{D}}f$.
\end{definition}

It is easy to show the following result on the label complexity of the CAL Algorithm \ref{alg:cal} (see e.g. \cite{hanneke2009TheoreticalFoundationsActive}) which bounds label complexity in terms of the disagreement coefficient.
\begin{theorem}\label{theo:label_comp_cal}
    Let $\mathcal{F}$ be a binary hypothesis space over $\mathcal{X}$, $\mathcal{D}$ some distribution over $\mathcal{X}$, and $f^*\in \mathcal{F}$ is the target we want to learn, i.e. $\er_{\mathcal{D}}(f^*)=0$. Moreover, $\VCdim(\mathcal{F})=d$.\\
    Then, the CAL Algorithm achieves label complexity $\Lambda$, such that 
    $$\Lambda(\varepsilon,\delta,f^*,\mathcal{D}) \lesssim \theta_{f^*}(\varepsilon)(d\ln(\theta_{f^*}(\varepsilon)) +\log_2(\frac{\log_2(1/\varepsilon)}{\delta}))\log_2(1/\varepsilon).$$
    
\end{theorem}

The key idea of proving a label complexity bound like this, is to construct a set of conditionally i.i.d. samples given the disagreement region after a certain number of label requests. Through sample complexity results of the passive learning literature, we can then bound the number of i.i.d. samples needed to guarantee that we halve the disagreement region. This is shown through removing those hypothesis with large error. The final bound on the label complexity is just a question of how often we need to halve the region of disagreement, to guarantee an error less than $\varepsilon$ with probability greater than $1-\delta$. A proof of this can be found in \cite{hanneke2009TheoreticalFoundationsActive} and \cite{hanneke2014TheoryDisagreementBasedActive}.\\

An analysis of the label complexity in terms of $\theta_f(\varepsilon)$ gives rise to an understanding of the asymptotic dependence of $\theta_f(\varepsilon)$ on $\varepsilon$. Furthermore, depending on the type of asymptotic dependence on $\varepsilon$, we can achieve better label complexity rates. For example a behavior of $\theta_f(\varepsilon)=\mathcal{O}(1)$, i.e. $\theta_f<\infty$ provides strongest label complexity results, as can be seen in Theorem \ref{theo:label_comp_cal} when we can upper bound $\theta_{f^*}(\varepsilon)$ by the constant $\theta_{f^*}$. However, it is also interesting to study a behaviour like $\theta_f(\varepsilon)=o(1/\varepsilon)$, as it can be shown that this is necessary for the CAL algorithm to provide improvement over passive learning algorithms \cite{hanneke2012ActivizedLearningTransforming}. The following lemma relates the asymptotic behaviour of the disagreement coefficient with the disagreement core.

\begin{lemma}[\cite{hanneke2014TheoryDisagreementBasedActive}] \label{lem:o(1/eps)}
    For a $f\in \mathcal{F}$ the following relation holds: $\theta_f(\varepsilon)=o(1/\varepsilon)$ if and only if $P_{x\sim\mathcal{D}}(x\in \partial f)=0$.
\end{lemma}

When simple disagreement-based methods, i.e., shattering of one point, no longer yield improvement (this is the case when $\theta_f(\varepsilon)=\Omega(1/\varepsilon)$) the concept of shattering more than one point becomes relevant. By the above Lemma \ref{lem:o(1/eps)}, this is exactly the case when $P_{x\sim\mathcal{D}}(x\in \partial f)>0$. The generalization to this setting was conceived by Hanneke in \cite{hanneke2009TheoreticalFoundationsActive,hanneke2012ActivizedLearningTransforming} and is the basis of this work. Albeit the context of this being active learning and asymptotic label complexity improvements over passive learning, we will manage to adopt this idea to derive bounds for sequential classification. We present Hanneke's improvements obtained through this generalization. This requires the introduction of the order $k$ disagreement coefficient and the shatter core for shattering $k$ points.\\


\begin{definition}[k-dim shatter core \cite{hanneke2012ActivizedLearningTransforming}]
    Let $\mathcal{F}$ be a binary function class over $\mathcal{X}$ and $k\in \mathbb{N}$.\\
    We define the $k$-dimensional shatter core with respect to $\mathcal{F}$ and distribution $\mathcal{D}$ over $\mathcal{X}$ as 
    $$\partial_{\mathcal{F},\mathcal{D}}^k f=\lim_{r\rightarrow0} \{S\in \mathcal{X}^k:\mathcal{B}_{\mathcal{F},\mathcal{D}}(f,r) \text{ shatters } S\}.$$
    As before, if $\mathcal{F}$ and $\mathcal{D}$ are clear from the context, we will equivalently write $\partial^k f=\partial_{\mathcal{F},\mathcal{D}}^k f$.
\end{definition}

\begin{definition}[order k disagreement coefficient]
    Let $\varepsilon>0$ and $\mathcal{F}$ a binary function space over $\mathcal{X}$ and $k\in \mathbb{N}$.\\
    The order-$k$ disagreement coefficient of a classifier $f\in \mathcal{F}$ with respect to the distribution $\mathcal{D}$ over $\mathcal{X}$ is defined as 
    $$\theta_f^{(k)}(\varepsilon)=\sup_{r>\varepsilon}\frac{P_{S=\{x_1,\dots, x_k\} \sim \mathcal{D}^{\otimes k}}(\mathcal{B}_{\mathcal{F},\mathcal{D}}(f,r) \text{ shatters } S)}{r}\vee 1.$$
    We furthermore define $\Tilde{d}_f=\min\{k\in \mathbb{N}: P_{S\sim \mathcal{D}^{\otimes k}}(S\in \partial^k f)=0 \}$.\\
    We abbreviate $\theta_f^{(k)}=\theta_f^{(k)}(0)$.
\end{definition}

Note that for $f\in \mathcal{F}$ with $\VCdim(\mathcal{F})=d$, we have $\Tilde{d}_f\leq d+1$, so this quantity is well defined, for a finite VC dimension. Furthermore, we have $\theta_f^{(\Tilde{d}_f)}(\varepsilon)\leq \theta_f(\varepsilon)$ for all $\varepsilon \geq 0$ and by using Lemma \ref{lem:o(1/eps)} it is always the case that $\theta_f^{(\Tilde{d}_f)}(\varepsilon)=o(1/\varepsilon)$.\\

The Shattering Algorithm was defined by Hanneke as an activizer for a passive learning algorithm \cite{hanneke2012ActivizedLearningTransforming} and revised as a general active learning strategy in \cite{hanneke2014TheoryDisagreementBasedActive}. It is of the same general structure as the CAL Algorithm, but differs in the decision of whether or not to request a label. As a reminder: for the CAL Algorithm, the label is requested at time $t$, if the current version space $\mathcal{F}_i$ $(0\leq i\leq n)$ shatters the sample $x_t$ and if not, the correct label can be inferred. For the shattering algorithm, we discuss the case, when $\Tilde{d}_f$ is known, and we set $k=\Tilde{d}_f$ (it is possible to handle the case when this is not known, through iterating through all possible values of $k$). Furthermore, we assume that the distribution $\mathcal{D}$ over $\mathcal{X}$ is also known. Note that in the case of i.i.d. samples, it is also possible to handle the case for which $\mathcal{D}$ is not known, by an estimator which is only using unlabeled data.\\

The following determines whether or not a label is requested:
$$P_{S=\{x_1,\dots, x_{k-1}\}\sim \mathcal{D}^{\otimes k-1}}(S\cup {x_t} \text{ is shattered by } \mathcal{F}_i| S \text{ is shattered by } \mathcal{F}_i)$$
If this quantity exceeds $1/2$ a label is requested and the hypothesis space is updated accordingly. Otherwise, the hypothesis space is updated with a prediction. In this setting, we predict with the label $\hat{y}_t$, which maximizes the following probability
$$P_{S=\{x_1,\dots, x_{k-1}\}\sim \mathcal{D}^{\otimes k-1}}(S \text{ is not shattered by } \mathcal{F}_i^{x_t\rightarrow 1-\hat{y}_t}| S \text{ is shattered by } \mathcal{F}_i).$$\\

The following theorem quantifies, the improvement in label complexity for the Shattering Algorithm:
\begin{theorem}
    Let $\mathcal{F}$ be a binary hypothesis space over $\mathcal{X}$, $\mathcal{D}$ some known distribution over $\mathcal{X}$, and $f^*\in \mathcal{F}$ is the target we want to learn, i.e. $\er_{\mathcal{D}}(f^*)=0$. Moreover, $\VCdim(\mathcal{F})=d$ and $k=\Tilde{d}_{f^*}$.\\
    Then, the Shattering Algorithm achieves label complexity $\Lambda$, such that 
    $$\Lambda(\varepsilon,\delta,f^*,\mathcal{D}) \lesssim \theta_{f^*}^{(\Tilde{d}_{f^*})}(\varepsilon)(d\ln(\theta_{f^*}(\varepsilon)) +\log_2(\frac{\log_2(1/\varepsilon)}{\delta}))\log_2(1/\varepsilon)$$
\end{theorem}

When comparing this label complexity to the one achieved by CAL algorithm, we have replaced the first factor of $\theta_{f^*}(\varepsilon)$ by $\theta_{f^*}^{(\Tilde{d}_{f^*})}(\varepsilon)$. Like explained before, aside from constant factors, this is never worse than $\theta_{f^*}(\varepsilon)$ and is significantly better, when $\theta_{f^*}(\varepsilon)=\Omega(1/\varepsilon)$, as we always have $\theta_{f^*}^{(\Tilde{d}_{f^*})}(\varepsilon)=o(1/\varepsilon)$.\\

What we take as inspiration from this, is however the fact that we use a proxy of drawing $k-1$ points $\{x_1,\dots, x_{k-1}\}$ i.i.d. from $\mathcal{D}$ and then comparing the incoming point $x_t$ to this set. For this, we are not using the previous points $x_i$ for $0\leq i < t$ directly, and therefore we don't rely on the i.i.d. assumption in this step. The information gained from the previous points is only indirectly used through the current version space $\mathcal{F}_i$.\\

\subsubsection{On Abstentions}\label{sec:abst}

We furthermore explore a different type of algorithm design, which in certain aspects exhibits a close connection to the field of active learning: classification with an option to reject, i.e. allowing the learner to abstain from making a classification. This simple but powerful idea was first introduced by Chow in his seminal results in 1957 and 1970 \cite{chow1957OptimumCharacterRecognition,chow1970OptimumRecognitionError}. By allowing a classification algorithm to abstain from predictions for uncertain cases through incurring a cost, he was the first to analyze the trade-off between rejections and errors. He furthermore derived a Bayes-optimal rejection rule, which can be implemented as a  plug-in rule. This can be viewed as a confidence based rejection rule and in case the distribution of samples is known explicitly, this is optimal.\\
Other ways to learn reject option classifiers in the offline setting are through empirical risk minimization methods, which has been explored extensively in the literature, e.g. for Support Vector Machines \cite{grandvalet2008SupportVectorMachines} or boosting methods \cite{cortes2016LearningRejection}. Important to mention in this setting, is the work of \cite{bartlett2008ClassificationRejectOption}, which uses a surrogate loss function to study classification with an option to reject as a convex optimization problem (using the double hinge loss). Classification with an option to reject has become an interesting topic once more in recent years in the context of Neural Network Learning and Deep Learning   \cite{lakshminarayanan2017SimpleScalablePredictive, corbiere2019AddressingFailurePrediction, geifman2019SelectiveNetDeepNeural}. Understanding predictions through Uncertainty Quantification and gaining confidence in the model's output has become more important than ever, in the light of more and more complex models.\\

To circle back to the study of abstentions in a sequential setting, we introduce the 'Knows What It Knows' algorithm \cite{li2008KnowsWhatIt} (in short: KWIK). This approach combines the PAC style of learning with the online setting we described earlier in Section \ref{sec:learning}. The goal is that in a sequence of adversarial chosen inputs, we try to establish a bound on the number of abstentions, whilst not allowing any misclassifications. Like the name of this algorithm already suggests: the learner is self-aware and can classify regions of certainty (what it knows) and uncertainty (what it does not know). This is also referred to as reliable learning \cite{rivest1988LearningComplicatedConcepts}. If we relax this approach to allowing a bounded number, we get a model that sits right in between online learning \cite{littlestone1988LearningQuicklyWhen} (for no abstentions) and the KWIK framework \cite{li2008KnowsWhatIt} (for no misclassifications).
It was shown in \cite{sayedi2010TradingMistakesDon} that by allowing a \textit{bounded number of misclassifications}, it is possible for certain infinite problem classes to achieve more desirable abstention bounds. This can be generalized to infinite problem classes and the bounds on expected abstentions of the algorithm, for a guaranteed number of misclassifications, can be expressed in terms of a quantity called Extended Littlestone dimension \cite{zhang2016ExtendedLittlestoneDimension}. \\ 

Another approach to designing and analyzing algorithms with respect to their expected abstention and misclassification errors, draws inspiration from active learning. This was conceived by El-Yaniv and Wiener \cite{el-yaniv2010FoundationsNoisefreeSelective} and they used the term \textit{selective classification} to describe the setting of classification with an option to reject. More specifically, they analyzed the trade-off between abstention error and misclassifications and they are the first to give confidence based guarantees on the respective true errors. \\

This can be formalized in the following way, according to \cite{el-yaniv2010FoundationsNoisefreeSelective}: a binary classifier is learnt through a function pair $(f,r)$ with $f\in \mathcal{F}$ a binary classifier and an independent rejection function $r:\mathcal{X}\rightarrow\mathcal{Y}$. The learner abstains (represented through the symbol $\perp$) for $r(x)=0$ and predicts according to some hypothesis $f(x)$ if $r(x)=1$:
\begin{equation*}
    (f,r)(x):= \begin{cases}
        \perp & \text{if } r(x)=0\\
        f(x) & \text{if } r(x)=1.
    \end{cases}
\end{equation*}

Based on this one can define the coverage and the risk of a classifier. 

\begin{definition}
    Let $(f,r)$ be a selective classifier with $f\in \mathcal{F}$ a binary function, $r:\mathcal{X}\rightarrow\mathcal{Y}$ and $\mathcal{D}$ a distribution over $\mathcal{X}$.\\
    We define the coverage of the classifier, as the expectation of the function $r$ with respect to the underlying distribution $\mathcal{D}$:
    $$\Phi(f,r)=\mathbb{E}(r(X)) \qquad \text{ where } \mathcal{X}\sim \mathcal{D}.$$
\end{definition}

\begin{definition}
    Let $(f,r)$ be a selective classifier with $f\in \mathcal{F}$ a binary function, $f^*\in\mathcal{F}$ the target, $r:\mathcal{X}\rightarrow\mathcal{Y}$, and $\mathcal{D}$ a distribution over $\mathcal{X}$.\\
    We define the risk of the classifier, as the average loss on the accepted samples (note that we use the 0-1 loss):
    $$R(f,r)=\frac{\mathbb{E}(r(X)\mathrm{1}_{\{f(X)\neq f^*(X)\}})}{\Phi(f,r)}\qquad \text{ where } \mathcal{X}\sim \mathcal{D}.$$
\end{definition}

The general objective of the learner is to produce a selective classifier $(f,r)$ with sufficiently low risk and sufficiently high coverage. This inherent trade-off can now be defined as the \textit{risk-coverage trade-off}. In order rigorously approach this, the authors first studied a \textit{perfect selective classifier} $(f,r)$, which is guaranteed to have zero risk $R(f,r)=0$ and highest possible coverage $\Phi(f,r)$. This is analogous to the objective of the KWIK algorithm \cite{li2008KnowsWhatIt}.\\

It is possible to show \cite{el-yaniv2010FoundationsNoisefreeSelective}, that the optimal strategy for issuing a prediction on a sample $x\in \mathcal{X}$ is to predict only if all classifiers in the version space \textit{agree}. This is referred to as a consistent selective strategy (CSS) and was shown to achieve monotonically increasing coverage (in the number of observed samples). Note that we can conversely state that the learner abstains if there exists functions which \textit{disagree}; this is the case exactly if the point falls in the disagreement of the current version space. This is the connection to disagreement-based active learning in the stream-based setting. \\

Building on this similarity, El-Yaniv and Wiener managed \cite{el-yaniv2012ActiveLearningPerfect} to derive general (target-independent) bounds on the coverage of the CSS through a novel complexity measure, called the compression set size. More generally, a relationship to the disagreement coefficient $\theta_f$ (see Definition \ref{def:disagree_coeff}) was established. One can show that exponential label complexity bounds for the CAL Algorithm  \ref{alg:cal} (which are characterized through $\theta$) are achievable if and only if perfect selective classification achieves fast coverage rates. And in a later collaboration between the authors and Hanneke \cite{wiener2015CompressionTechniqueAnalyzing, hanneke2014TheoryDisagreementBasedActive} they used the techniques from selective classification to give tighter coverage guarantees in active learning settings. This establishes the close connection between the field of active learning and selective classification in the realizable setting.\\

The trade-off between risk and coverage is similar to what we will observe later when studying the expected abstention error in relation to the expected misclassification in the setting of sequential learning. Furthermore, we observe a connection between the active learning framework and classification with the option to reject. We will employ this later through allowing the algorithm to abstain on certain predictions. Note that the method to use abstention to improve classification in the presence of an adversary, has been also studied for the transductive setting in \cite{goldwasser2020PerturbationsLearningGuarantees,kalai2021OptimallyAbstainingPrediction}, which is similar to having a clean-label adversary.\\

\subsubsection{On Clean-label Attacks}

Lastly, we explore the concept of clean-label attacks which has been introduced in 2018 as a type of data poisoning attack in the work of Shafahi \cite{shafahi2018PoisonFrogsTargeted}. Since then, it has been studied extensively in the Machine Learning and Deep Learning community \cite{suciu2018WhenDoesMachine,turner2018CleanLabelBackdoorAttacks, gupta2023AdversarialCleanLabel}. The rise of larger and more complex models with intricate prediction strategies, not directly obvious to users, increases the potential for adversarial attacks and manipulations that might go unnoticed.\\

The vulnerability of Deep Learning methods to adversarial attacks has been thoroughly explored in the literature, particularly focusing on so called \textit{evasion attacks} \cite{szegedy2013IntriguingPropertiesNeural, biggio2013EvasionAttacksMachine, goodfellow2014ExplainingHarnessingAdversarial}.
Here, one assumes that the adversary is attacking at test time, through alteration of the target instance, in order to achieve a misclassification. An example of this would be an adversary \textit{flipping} the true label to enforce a misclassification. In evasion attacks like this however, the adversary is assumed to have more control, than would be realistic in real-life settings, like having access to the model or  target class. \\

Clean-label attacks however represent a more realistic way of thinking about adversarial attacks. An adversary injecting clean-label samples, is influencing the model already in the training phase, with the goal to target a specific instance in the testing phase. In this setting, the adversary does not have access to the true labelling function from the hypothesis space $f^*\in \mathcal{F}$ and thus cannot interfere through flipping labels. But the adversary is able to inject arbitrary examples $x\in\mathcal{X}$, \textit{not} sampled from the true distribution $\mathcal{D}$ over $\mathcal{X}$. The label is then still generated according to the true labelling function $f^*$, making it harder for an expert observer to detect adversarial injections. Note that in general, this does not affect overall performance, as we still learn correct information about $f^*$, however the attack can be \textit{targeted}, to enforce misclassification of some specific examples. \\

To illustrate the risks associated with clean-label attacks the context of real-life applications, consider a large language model (LLM) designed for sentiment analysis, i.e. extracting opinions or specific emotions from text. For instance, a company aiming to foster positive brand associations or a political party seeking to influence public opinion could employ a clean-label attack. 
In this scenario, the company or political party acts as an adversary, intending to manipulate predictions, say a positive or negative association with the brand name/political party. Rather than gaining access to the model and altering existing predictions, which could be challenging if robust security measures are in place, the adversary might flood the internet with explicitly positive articles related to their cause. A LLM would in turn classify these articles as positive, if they are not straight up lies (clean-labels). This creates an over-representation of positive associations with the targeted name and a model trained on this classified data is more likely to associate positive sentiments with the brand name/political party. \\

This setting has first been studied from a theoretical perspective under the name of p-tampering \cite{mahloujifar2017BlockwisePTamperingAttacks,mahloujifar2020LearningPtamperingPoisoning}, which is related to a malicious noise model by Valiant \cite{valiant1985LearningDisjunctionConjunctions} where an adversary with some probability $p$ can choose an incoming sample with its label. In the $p$-tampering model, we allow an adversary to corrupt the test sample with probability $p$, but the label is then assigned according to the true labelling function. \\

We consider however a more powerful adversary, which can inject an out of distribution sample at any time, so not with a specific probability $p$. In the setting we consider, it will not be revealed to the learner which examples were adversarial, also not in hindsight. The theoretical aspects of this has been studied in \cite{blum2021RobustLearningCleanlabel}. 

\newpage

\section{Results from Original Paper in the Realizable Setting}\label{sec:results_real}
Before diving into the more general results of the paper \cite{goel2023AdversarialResilienceSequential}, we give a short motivating example of a simple disagreement-based learner, in the realizable setting with a clean-label adversary. \\
We consider the hypothesis class of thresholds $\mathcal{F}=\{\mathbbm{1}_{\{\cdot\leq a\}}:a\in [0,1]\}$ on the interval $[0,1]$. A simple disagreement-based learner (similar to the CAL Algorithm in active learning \ref{alg:cal}), chooses to predict when all classifiers from the current version space agree on its label and abstains, whenever the point falls into the disagreement region. 
It is clear that we never misclassify a sample, however the following holds for the expected number of abstentions:
\begin{itemize}
    \item In the fully stochastic setting, a classifier will abstain from predicting in expectation at most $2\log(T)$ times.
    \item In the fully adversarial setting, an adversary can sample from the disagreement region every time, thus forcing abstention linear in $T$.
\end{itemize}

With our approach in the presence of a clean-label adversary, we can achieve stochastic like guarantees for the abstention error, when we only incur a cost, if we abstain on an i.i.d. sample.

\begin{theorem}\label{theo:mis_abs_thresh}
    Let $\mathcal{F}=\{\mathbbm{1}_{\{\cdot\leq a\}}:a\in [0,1]\}$ be the class of thresholds on $[0,1]$ and $T\in\mathbb{N}^+$. \\
    We get the following bounds for a disagreement-based learner with time horizon $T$ with clean-label injections
    $$\MCE=0\quad \mathbb{E}(\AbE)\leq 2\log(T).$$   
\end{theorem}

\begin{proof}
    Misclassification error: Whenever we choose to predict, we predict according to the true labelling function $f^*\in\mathcal{F}$, as we only predict when all classifiers agree. Furthermore, we never eliminate the target. \\
    Abstention error: Assume the number of i.i.d. samples observed until time $t$ is exactly $i-1$. An exchangeability argument shows that the probability of a new i.i.d. point falling in the disagreement region of the other $i-1$ points is at most $2/i$. Summing over the time horizon $T$ gives the desired bound. 
\end{proof}

Note that disagreement-based classifiers like this have been extensively studied in the literature as perfect selective classifiers, as we have already introduced in Section \ref{sec:abst}. Those hypothesis classes with a sublinear abstention rate, can be characterized through a quantity called the \textit{star number} $\mathfrak{s}$ (see \cite{hanneke2014MinimaxAnalysisActive,hanneke2016RefinedErrorBounds} also for a definition), where the VC dimension of a class always serves as a lower bound for $\mathfrak{s}$. 
In the stochastic setting, we can achieve a sublinear number of abstentions in $T$ if and only if the star number is finite, in which case it is at most $\mathfrak{s}\log(T)$. One can show \cite{hanneke2014MinimaxAnalysisActive} that for the class of thresholds one has $\mathfrak{s}=2$, however already for indicators of intervals where the VC dimension is 2, the star number is infinite. This motivates an approach that moves beyond disagreement-based methods; to give guarantees for more general function classes.\\

\subsection{General VC Dimension with Known Distribution}

Motivated by the simplicity and efficiency of disagreement-based methods, \cite{goel2023AdversarialResilienceSequential} adopts this approach and develops it further to move from simple disagreement (which is the shattering of one point) to the shattering of $k$-points, where $k$ is at most the VC dimension of the hypothesis class $\mathcal{F}$. We will see that this approach enables us to give stochastic like guarantees, for even more complex problem classes:\\
We show that for a class with finite VC dimension, it is feasible to achieve sublinear misclassification error at the expense of making few abstentions, even in the setting of adversarial clean-label injections (if we allow abstention on adversarial points without incurring a cost). \\

First, we describe and motivate the approach of using the probability of shattering $k$ points before giving guarantees of the learner's misclassification and abstention error.\\
Algorithm \ref{alg:gen_vc} takes inspiration from the active learning algorithm based on shattering that Hanneke developed in \cite{hanneke2009TheoreticalFoundationsActive,hanneke2012ActivizedLearningTransforming} to improve asymptotic error guarantees of disagreement-based active learning. For an overview see Section \ref{sec:active learning}. We adopt his ideas in order to go beyond disagreement-based methods and to give guarantees for an even broader class of functions. Furthermore, he uses a proxy which is not directly dependent on the previously observed data, but leverages knowledge of the underlying distribution. This is very relevant in our setting, as we cannot assume that previously observed samples are i.i.d. Nevertheless, we can use the proxy to apply our knowledge of the underlying true sample distribution in order to decide whether to abstain or predict.\\

We will describe how Algorithm \ref{alg:gen_vc} operates.
The algorithm keeps track of a state variable $k$, which we denote as the 'level' of the algorithm. This variable is initialized by the VC dimension $d$  of the hypothesis space $\mathcal{F}$ (since, by definition, the largest set of points that can be shattered by $\mathcal{F}$ has a magnitude of $d$). Now in each round $t\in \{1,\dots, T\}$ whenever $k=0$, we have a disagreement-based learner: one abstains if the incoming point $\hat{x}_t$ falls in the disagreement region of the current hypothesis space and else predicts according to the consistent label.\\

\begin{algorithm}[H]
    \caption{Learning for general VC dimension $d$ when distribution $\mathcal{D}$ of $X$ is known}\label{alg:gen_vc}
    \begin{algorithmic}[1]
        \State Set $k=d$ and $\mathcal{F}_0=\mathcal{F}$
        \For{$t=1,\dots,T$}
        \State Receive $\hat{x}_t$
        \If{$k>0$}
        \If{$\min(\rho_k(\mathcal{F}_{t-1}^{\hat{x}_t\rightarrow 0}),\rho_k(\mathcal{F}_{t-1}^{\hat{x}_t\rightarrow 1}))\geq 0.6\rho_k(\mathcal{F}_{t-1})$} \label{alg:step_abstain}
        \State Predict $\hat{y}_t=\perp$
        \Else 
        \State Predict $\hat{y}_t=\arg \max _{j\in \{0,1\}}\rho_k(\mathcal{F}_{t-1}^{\hat{x}_t\rightarrow j})$\label{alg:step_predict}
        \EndIf        
        \State Receive true label $y_t$ and update $\mathcal{F}_{t}\gets\mathcal{F}_{t-1}^{\hat{x}_t\rightarrow y_t}$
        \If{$\rho_k(\mathcal{F}_{t})\leq \alpha_k$}
        \State $ k\gets k-1$
        \EndIf
        \Else
        \If{$\hat{x}_t\in \mathcal{S}_1(\mathcal{F}_{t-1})$}
        \State Predict $\hat{y}_t=\perp$
        \Else
        \State Predict $\hat{y}_t=f(\hat{x}_t)$ for some $f\in \mathcal{F}_{t-1}$
        \EndIf
        \State Receive true label $y_t$ and update $\mathcal{F}_{t}\gets\mathcal{F}_{t-1}^{\hat{x}_t\rightarrow y_t}$
        \EndIf
        \EndFor
    \end{algorithmic}
\end{algorithm}

For $k>0$ the learner goes through the first If loop in Step \ref{alg:step_abstain} of the algorithm. We compare the probability of shattering $k$ points with respect to two different function classes. We are interested in $\mathcal{F}_{t-1}$, which constitutes the function class $\mathcal{F}$ limited to all points that have been previously processed — those for which we know the true labels. We compare the class of functions $f\in\mathcal{F}_{t-1}$ which satisfy $f(\hat{x}_t)=0$ to the class where we restrict $f\in\mathcal{F}_{t-1}$ to satisfy  $f(\hat{x}_t)=1$. 
If both times the shattering probability is high (compared to shattering probability of $\mathcal{F}_{t-1}$) , then we \textit{abstain}. If this is not the case (i.e. the shattering probability of at least one of the classes is low) then we \textit{predict}, according to Step \ref{alg:step_predict}, the label which maximizes the probability of shattering $k$ points. \\
After receiving the true label $y_t$, we update the hypothesis class to $\mathcal{F}_{t}$ by restricting it to only contain functions $f\in \mathcal{F}_{t-1}$ for which $f(\hat{x}_t)=y_t$. Lastly, if the probability of shattering $k$ points by this function class is below some threshold $\alpha_k$, we reduce the level $k$ by one. \\

The decision to abstain or predict relies on proving that for $\hat{x}_t$ drawn independently from $\mathcal{D}$, the event where both probabilities (mapping to 0 and 1) exceed the threshold is small and can be upper-bounded, enabling us to upper-bound $\mathbb{E}(\AbE)$. This argument is based on the observation that one can relate the probabilities of shattering $k$ points with both classes $\mathcal{F}_{t-1}^{\hat{x}_t\rightarrow 0}$ \textit{and} $\mathcal{F}_{t-1}^{\hat{x}_t\rightarrow 1}$ to shattering $k+1$ points. Furthermore, we have lower and upper bounds on $\rho_k(\mathcal{F})$ and $\rho_{k+1}(\mathcal{F})$ respectively through $\alpha_{k+1}$ and $\alpha_k$, where one can choose these thresholds to allow for the best trade-off between $\mathbb{E}(\MCE)$ and $\mathbb{E}(\AbE)$. To derive a bound for the $\mathbb{E}(\MCE)$, by design of the algorithm one has that $\rho_k(\mathcal{F}_{t-1})$ reduces by a constant fraction, if a misclassification occurs. For each level $k$ this can only happen a certain number of times, as if $\rho_k(\mathcal{F}_{t-1})$ falls below $\alpha_k$, we go down one level, therefore allowing us to derive a bound for the number of misclassifications.\\

Finally, we can give the following guarantees on the expected error after $T$ samples in Theorem \ref{theo:bounds_gen_vc}. A detailed proof based on the above described outline, can be found in \cite{goel2023AdversarialResilienceSequential}.
\begin{theorem}[\cite{goel2023AdversarialResilienceSequential}]\label{theo:bounds_gen_vc}
    Let $\mathcal{F}$ be a hypothesis class with $\VCdim(\mathcal{F})=d<\infty$. Then, in the adversarial clean-label injection model with abstentions and time horizon $T$, Algorithm \ref{alg:gen_vc} with thresholds $\alpha_k=T^{-k}$ in each level $k\in \{1,\dots, d\}$ gives the following guarantee
        \begin{align*}
        \mathbb{E}(\MCE) &\leq d^2\log(T)\\
        \mathbb{E}(\AbE)&\leq 6d.
    \end{align*}
\end{theorem}

Algorithm \ref{alg:gen_vc} iterates through different levels $k$, starting with level $k=d$ where $d$ is the VC dimension of $\mathcal{F}$. Here, $k$ serves as a measure of the complexity of the hypothesis space. As the algorithm progresses this complexity is reduced, by learning more from the labeled points. To be more precise, we measure the complexity of the hypothesis space by how likely it is for the function space $\mathcal{F}_{t-1}$ to shatter $k$ points, randomly drawn from the distribution $\mathcal{D}$. This approach is similar to disagreement-based learning, but instead of focusing on whether or not we shatter one point, we want to study the probability of shattering $k$ points.\\
We manage to show that this generalization is a real improvement over disagreement-based learning. Sublinear error rates are possible for classes with VC dimension $d<\infty$ when the distribution $\mathcal{D}$ of the samples $x_t\in \mathcal{X}$ is known and the learner robust to a clean-label adversary.\\
The probability of shattering $k$-points randomly drawn from $\mathcal{D}$ and using this as a measure for certainty in our prediction, will give rise to two very important qualities:
\begin{itemize}
    \item  Due to the algorithm's design, this concept will allow keeping track of the learner's mistakes. In this scenario, at level $k$, $\rho_k(\mathcal{F}_{t-1})$  decreases by a constant. Since we can lower bound this quantity, there is a limit to how many times this reduction can occur before transitioning to a lower level. This implies that we have learned significantly from the mistakes, resulting in a reduction of complexity of the current hypothesis space. 
    \item When deciding whether to predict or abstain from a given point $\hat{x}_t$, the decision is not based directly on previous points. This is important because this set of points does not constitute an i.i.d. sample from the distribution $\mathcal{D}$ due to the adversarial injections. However, we compare the incoming point to a randomly drawn set of $k$ i.i.d. points. This serves as a proxy to circumvent the issue posed by the non i.i.d. sample whilst still leveraging the knowledge of the true distribution $\mathcal{D}$. The idea of relying on a proxy to decide predictions will be a recurring theme throughout the entire paper.
\end{itemize}

\subsection{Structure based algorithms}
In the setting where the distribution of $\mathcal{X}$ is unknown, the previous approach based on probability of shattering fails to work, as this quantity can no longer be calculated explicitly. One obvious approach would be to find a good estimate for this, however the adversarial injected out-of-distribution samples would corrupt this estimate. \\
Thus we change gears and exploit structural properties of certain problem classes to derive how different predictions influence the disagreement region. We start with a structural property of problem classes of VC dimension 1 and extend this approach to axis-aligned rectangles in $d$ dimensions, a problem class of VC dimension $2d$.

\subsubsection{VC dimension 1}

We start with providing a general algorithm for which we can prove bounds on expected misclassification and abstention error for problems of VC dimension 1. The proof we use is specifically based on a structural property of classes with VC dimension 1 and thus cannot be easily extended to arbitrary VC dimension, even though the algorithm also seems sensible in these cases. However, problems of VC dimension 1 serve as an important base cases or examples (like threshold problems in one dimension) and an understanding of these simple examples is essential.\\
The algorithm relies heavily on a concept we call \textit{leave-one-out disagreement estimate}, which will serve as a quantity to keep track of the number of points from our current sample $x\in S$ which would fall in the disagreement region of the hypothesis class if we restricted it to a set which disagrees with a chosen reference function $f$. The choice of such a reference function $f$ might seem arbitrary at first but its importance is revealed later, as the Structural Theorem \ref{theo:structure} is based on this quantity. Note that this comparison function $f$ is fixed and independent of the true labelling function $f^*$.
\begin{definition}
    Let $\mathcal{F}$ be a hypothesis class, $f\in \mathcal{F}$ be a fixed reference function and $S$ a set of realizable data points. We define the leave-one-out disagreement estimate $\Gamma$ as
    $$\Gamma(S,\mathcal{F},f)=\{x| \exists y: (x,y)\in S \text{ and } x\in \mathcal{S}_1(\mathcal{F}|_{S_f\setminus (x,y)})\}$$
    where $S_f=\{(x,y)\in S: f(x)\neq y\}$. Furthermore denote $\gamma(S,\mathcal{F},f)=|\Gamma(S,\mathcal{F},f)|$. We also define $\gamma_t:=\gamma(S_t,\mathcal{F},f)=|\Gamma(S_t,\mathcal{F},f)|$. We omit the dependence on $f$ it is clear from the context. 
\end{definition}

In order to understand better what this definition entails, we provide a simple example of a problem class with VC dimension 1; thresholds on $[0,1]$.
\begin{example}
    Let the hypothesis class consist of threshold classifiers on the interval $[0,1]=\mathcal{X}$ and $\mathcal{F}=\{\mathbbm{1}_{\{\cdot\leq a\}}:a\in [0,1]\}$ and we can choose $f$ to be the all zeros function $f(x)=0$ for all $x\in [0,1]$ (we will see later why this choice is possible). We are given a data set $S=\{(x_1,1),\dots,(x_d,1),(x_{d+1},0),\dots, (x_n,0) \}$ of $n$ labelled points which we assume to be ordered $0\leq x_1\leq \dots \leq x_d\leq x_{d+1}\leq \dots \leq x_n\leq 1$. Then the set $S_f=\{(x,y)\in S: f(x)\neq y\}=\{(x,y)\in S: 0\neq y\}=\{(x_1,1),\dots,(x_d,1)\}$ consists of the $d$ positive points. We can calculate the corresponding leave-one-out disagreement regions
    \begin{itemize}
        \item $\mathcal{S}_1(\mathcal{F}|_{S_f\setminus(x,y)})=(x_d,1]$ for all $(x,y)\in S\setminus\{(x_d,1)\}$ and 
        \item $\mathcal{S}_1(\mathcal{F}|_{S_f\setminus(x_d,1)})=(x_{d-1},1]$
    \end{itemize}
    as only if we leave out the largest positive point, the disagreement region changes. Therefore $\Gamma(S,\mathcal{F},f)=\{x_d, \dots, x_n\}$ and $\gamma(S,\mathcal{F},f)=n-d$ in this case. 
\end{example}

Now we can proceed to the construction of a learning algorithm. Similarly to the known distribution case, we will propose a label for an incoming point $\hat{x}_t$ using the reference function $f(\hat{x}_t)\in \{0,1\}$. Then we calculate the quantity $a_0=\Gamma(S_{t-1},\mathcal{F}^{\hat{x}_t\rightarrow f(\hat{x}_t)},f)$. This will serve as a prediction rule in the sense that if $a_0$ exceeds a certain threshold we choose to predict, else we opt to abstain. If we make a mistake, the size of $\Gamma(S_t,\mathcal{F},f)$ will go down. 

\begin{algorithm}
    \caption{Structure based algorithm in unknown distributions}\label{alg:vcdim1}
    \begin{algorithmic}[1]
        \State Let $f\in \mathcal{F}$ be a reference function (independent of $f^*$) and $\alpha$ the abstention threshold
        \State Set $\mathcal{F}_0=\mathcal{F}$ and $S_0=\emptyset$
        \For{$t=1,\dots, T$}
        \State Receive $\hat{x}_t$
        \State Let $a_0=|\Gamma(S_{t-1},\mathcal{F}^{\hat{x}_t\rightarrow f(\hat{x}_t)},f)|$ 
        \If{$\hat{x}_t\notin \mathcal{S}_1(\mathcal{F}_{t-1})$} 
        \State Predict $\hat{y}_t=\Tilde{f}(\hat{x}_t)$ according to some $\Tilde{f}\in \mathcal{F}_{t-1}$ 
        \ElsIf{$a_0\geq \alpha$}
        \State Predict $\hat{y}_t=f(\hat{x}_t)$
        \Else
        \State Predict $\hat{y}_t=\perp$
        \EndIf        
        \State Receive true $y_t$ and update $S_t\gets S_{t-1}\cup (\hat{x}_t,y)$ and $\mathcal{F}_t\gets\mathcal{F}_{t-1}^{\hat{x}_t\rightarrow y_t}$
        \EndFor
    \end{algorithmic}
\end{algorithm}

We will be able to prove performance guarantees for VC dimension 1 because of the Structural Theorem stated in \cite{ben-david2015NotesClassesVapnikChervonenkis}: he relates classes of VC dimension 1 to a certain type of tree ordering which induces linearity.  
This characterization was also used in \cite{blum2021RobustLearningCleanlabel} to show that a problem with VC dimension 1 under clean-label attacks is robustly learnable. In general VC dimensions one cannot use this explicit structure and thus one would need to argue differently to provide performance guarantees for Algorithm \ref{alg:vcdim1}.\\

In \cite{ben-david2015NotesClassesVapnikChervonenkis} Ben-David showed that if the functions of a hypothesis class $\mathcal{F}$ can be linearly ordered, then the VC dimension of $\mathcal{F}$ is 1. However the condition of a linear order can be relaxed to a partial one, which induces a tree, which gives equivalence.
\begin{definition}
    Given a domain $\mathcal{X}$ and a partial order $\prec$ on $\mathcal{X}$, we say that a set $I\subseteq \mathcal{X}$ is an initial segment if for all $x\in I$ and $y\in \mathcal{X}$ with $y\prec x$ we have $y\in I$. \\
    The partial order $\prec$ defines a tree ordering if for every $x\in \mathcal{X}$ the initial segment $I$ is linearly ordered, i.e. for all $x,y\in I$ then $x\prec y$ or $y\prec x$.
\end{definition}

\begin{theorem}[\cite{ben-david2015NotesClassesVapnikChervonenkis}]\label{theo:structure}
    Let $\mathcal{F}$ be a function class over some domain $\mathcal{X}$. Then the following are equivalent:
    \begin{enumerate}
        \item $\VCdim(\mathcal{F})\leq 1$
        \item There exists some tree ordering over $\mathcal{X}$ and a hypothesis $f\in \mathcal{F}$ such that every element of the new function class $$\mathcal{F}_f=\{\mathbbm{1}_{\{h(\cdot)\neq f(\cdot)\}}:h\in \mathcal{F}\}$$ is an initial segment with respect to this ordering relation. 
    \end{enumerate}
\end{theorem}

Note the existence of the reference function $f\in \mathcal{F}$ which is chosen specifically such that the hypothesises can be transformed to initial segments is the justification for our approach. Furthermore the tree ordering in Theorem \ref{theo:structure} on $\mathcal{X}$ can be explicitly constructed by the relation 
$$\leq^\mathcal{F}_f= \{(x,y):\forall h \in \mathcal{F}\quad h(y)\neq f(y)\rightarrow h(x)\neq f(x)\}$$
which by Lemma 5 in \cite{ben-david2015NotesClassesVapnikChervonenkis} indeed defines a tree ordering.\\
\begin{example}\label{ex:vc1tree}
We give an example of a binary function class which is partially ordered.\\
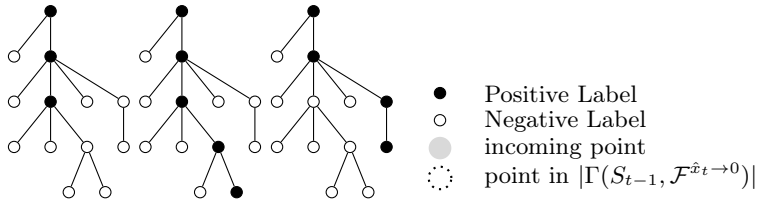
\begin{figure}[H]
    \centering
\begin{tikzpicture}[x=\XS,y=\XS*0.4,
        every node/.style={draw=black, fill=white ,circle,inner sep=1.5pt},
        rotate=270, ]
\node[fill=black] (1) at (0,0) {} ;
\node[fill=black] (2) at (1,0) {}  ;
\node (3) at (1,-2) {} ;
\node (4) at (2,4) {} ;
\node (5) at (2,2) {} ;
\node[fill=black] (6) at (2,0) {} ;
\node (7) at (2,-2) {} ;
\node (8) at (3,4) {} ;
\node (9) at (3,2) {} ;
\node (10) at (3,0) {} ;
\node (11) at (3,-2) {} ;
\node (12) at (4,3) {} ;
\node (13) at (4,1) {};
\draw (1) -- (2);
\draw (1) -- (3);
\draw (2) -- (4);
\draw (2) -- (5);
\draw (2) -- (6);
\draw (2) -- (7);
\draw (4) -- (8);
\draw (6) -- (9);
\draw (6) -- (10);
\draw (6) -- (11);
\draw (9) -- (12);
\draw (9) -- (13);
\end{tikzpicture}
\begin{tikzpicture}[x=\XS,y=\XS*0.4,
        every node/.style={draw=black, fill=white ,circle,inner sep=1.5pt},
        rotate=270, ]
\node[fill=black] (1) at (0,0) {} ;
\node[fill=black] (2) at (1,0) {}  ;
\node (3) at (1,-2) {} ;
\node (4) at (2,4) {} ;
\node (5) at (2,2) {} ;
\node[fill=black] (6) at (2,0) {} ;
\node (7) at (2,-2) {} ;
\node (8) at (3,4) {} ;
\node[fill=black] (9) at (3,2) {} ;
\node (10) at (3,0) {} ;
\node (11) at (3,-2) {} ;
\node[fill=black] (12) at (4,3) {} ;
\node (13) at (4,1) {};
\draw (1) -- (2);
\draw (1) -- (3);
\draw (2) -- (4);
\draw (2) -- (5);
\draw (2) -- (6);
\draw (2) -- (7);
\draw (4) -- (8);
\draw (6) -- (9);
\draw (6) -- (10);
\draw (6) -- (11);
\draw (9) -- (12);
\draw (9) -- (13);
\end{tikzpicture}
\begin{tikzpicture}[x=\XS,y=\XS*0.4,
        every node/.style={draw=black, fill=white ,circle,inner sep=1.5pt},
        rotate=270, ]
\node[fill=black] (1) at (0,0) {} ;
\node[fill=black] (2) at (1,0) {}  ;
\node (3) at (1,-2) {} ;
\node[fill=black] (4) at (2,4) {} ;
\node (5) at (2,2) {} ;
\node (6) at (2,0) {} ;
\node (7) at (2,-2) {} ;
\node[fill=black] (8) at (3,4) {} ;
\node (9) at (3,2) {} ;
\node (10) at (3,0) {} ;
\node (11) at (3,-2) {} ;
\node (12) at (4,3) {} ;
\node (13) at (4,1) {};
\draw (1) -- (2);
\draw (1) -- (3);
\draw (2) -- (4);
\draw (2) -- (5);
\draw (2) -- (6);
\draw (2) -- (7);
\draw (4) -- (8);
\draw (6) -- (9);
\draw (6) -- (10);
\draw (6) -- (11);
\draw (9) -- (12);
\draw (9) -- (13);
\end{tikzpicture}
\hspace{0.2cm}
\begin{tikzpicture}[x=\XS,y=\XS*1.2]
\node[anchor=north, fill=black, draw=black ,circle, inner sep=1.5pt] at (0.5,0.875) {};
\node[anchor=west] at (1.25,0.75) {\footnotesize Positive Label};
\node[anchor=north, fill=white, draw=black ,circle, inner sep=1.5pt] at (0.5,0.375) {};
\node[anchor=west] at (1.25,0.25) {\footnotesize Negative Label};
\node[anchor= north, draw=none, circle, thick , dotted,fill=gray!30, inner sep=3pt] (9) at (0.5,-0.025) {};
\node[anchor=west] at (1.25,-0.25) {\footnotesize incoming point};
\node[anchor= north, draw=black, circle, thick , dotted, fill=white, inner sep=3pt] (9) at (0.5,-0.575) {};
\node[anchor=west] at (1.25,-0.75) {\footnotesize point in $|\Gamma(S_{t-1},\mathcal{F}^{\hat{x}_t\rightarrow 0})|$};
\end{tikzpicture}
    \caption{Examples of different function classes on $\mathcal{X}$}
    \label{fig:function_tree}
\end{figure}

Assume we are given a set $\mathcal{X}$ consisting of 13 points which are partially ordered according to the tree given in Figure \ref{fig:function_tree}. We are considering the function class $$\mathcal{F}=\{f_p:\mathcal{X}\rightarrow\{0,1\}|\quad  p \text{ path on } \mathcal{X} \text{ and } f(x)=1 \text{ for } x\in p\}$$ where the number of functions in $\mathcal{F}$ is the number of all possible paths on $(\mathcal{X},\prec)$ containing the root. For three examples of such functions see Figure \ref{fig:function_tree}, where the black nodes correspond to a positive labels and the white nodes to negative labels. By the structure Theorem \ref{theo:structure} these functions are all initial segments with respect to a reference function $f(x)=0$ for all $x\in \mathcal{X}$ and thus $\VCdim(\mathcal{F})=1$. \\
    
\begin{minipage}{0.76\textwidth}
We let Algorithm \ref{alg:vcdim1} process two incoming points in order to show a correct classification from the disagreement region and a misclassification. Let the true function $f^*$ be according to the first example. The quantity $\gamma_t$ helps us to keep track of the points in the estimated disagreement region at time $t$.\\
What happens is depicted in the 5 trees on the right. We sequentially receive new points $\hat{x}_t$ and run through the decisive steps of the algorithm in order to make predictions $\hat{y}_t\in \{0,1,\perp\}$. As $T=13$, we have $\alpha=\sqrt{T/\log(T)}<3$.
\begin{enumerate}
        \item Starting at $t=11$ we have the set $S_{10}$ of 10 labelled points which we have observed so far. Note that $\gamma_{10}=|\Gamma(S_{10},\mathcal{F})|=8$. 
        \item The new point $\hat{x}_{11}$ is placed accordingly within the tree. Note that the point falls in the disagreement region $\hat{x}_{11}\in \mathcal{S}_1(\mathcal{F}_{10})$. Thus we need to calculate $a_0=|\Gamma(S_{10},\mathcal{F}^{\hat{x}_{11}\rightarrow 0},0)|=6$. As $a_0>\alpha$ we predict $\hat{y}_{11}=0$.
        \item The prediction is correct, which adds a new negative point to the tree $S_{11}=S_{10}\cup\{(\hat{x}_{11},0)\}$ and $\gamma_{11}=|\Gamma(S_{10},\mathcal{F})|=9$.
        \item We receive a new point $\hat{x}_{12}$, falling in the disagreement region $\hat{x}_{12}\in \mathcal{S}_1(\mathcal{F}_{11})$. Again we calculate $a_0=|\Gamma(S_{11},\mathcal{F}^{\hat{x}_{12}\rightarrow 0},0)|=4$ and as $a_0>\alpha$ we predict $\hat{y}_{12}=0$.
        \item This prediction was false and the true label is $y_{12}=1$ which we use to update $S_{11}$ to get $S_{12}$. Note that now $\gamma_{12}=6$. So the size of our estimated leave-one-out disagreement region reduced due to a misclassification. 
    \end{enumerate}
\end{minipage}
\begin{minipage}{0.2\textwidth}
    
\centering

\begin{tikzpicture}[x=\XS,y=\XS*0.4,
        every node/.style={draw=black, fill=white ,circle,inner sep=1.5pt},
        rotate= 270]
\node[fill=white, draw=none] (lab1)   at (0.5,3.5) {\footnotesize{1.}} ;
\node[fill=black] (1) at (0,0) {} ;
\node[fill=black] (2) at (1,0) {}  ;
\node (3) at (1,-2) {} ;
\node (4) at (2,4) {} ;
\node (7) at (2,-2) {} ;
\node (8) at (3,4) {} ;
\node (10) at (3,0) {} ;
\node (11) at (3,-2) {} ;
\node (12) at (4,3) {} ;
\node (13) at (4,1) {};
\draw (1) -- (2);
\draw (1) -- (3);
\draw (2) -- (4);
\draw (2) -- (6);
\draw (2) -- (7);
\draw (4) -- (8);
\draw (2) -- (12);
\draw (2) -- (13);
\draw (2) -- (10);
\draw (2) -- (11);
\end{tikzpicture}

\begin{tikzpicture}[x=\XS,y=\XS*0.4,
        every node/.style={draw=black, fill=white ,circle,inner sep=1.5pt},
        rotate=270, ]
\node[fill=white, draw=none] (lab2)   at (0.5,3.5) {\footnotesize{2.}} ;
\node[fill=black] (1) at (0,0) {} ;
\node[fill=black] (2) at (1,0) {}  ;
\node (3) at (1,-2) {} ;
\node (4) at (2,4) {} ;
\node (7) at (2,-2) {} ;
\node (8) at (3,4) {} ;
\node[draw=none,fill=gray!30, inner sep=3pt] (9) at (3,2) {} ;
\node (10) at (3,0) {} ;
\node (11) at (3,-2) {} ;
\node (12) at (4,3) {} ;
\node (13) at (4,1) {};
\draw (1) -- (2);
\draw (1) -- (3);
\draw (2) -- (4);
\draw (2) -- (6);
\draw (2) -- (7);
\draw (4) -- (8);
\draw (9) -- (12);
\draw (9) -- (13);
\draw (2) -- (9);
\draw (2) -- (10);
\draw (2) -- (11);

\draw[thick, dotted] (3,0) circle (3.5pt);
\draw[thick,dotted](1,0) circle (3.5pt);
\draw[thick,dotted](3,-2) circle (3.5pt);
\draw[thick,dotted] (3,4) circle (3.5pt);
\draw[thick,dotted] (2,4) circle (3.5pt);
\draw[thick,dotted] (2,-2) circle (3.5pt);

\end{tikzpicture}

\begin{tikzpicture}[x=\XS,y=\XS*0.4,
        every node/.style={draw=black, fill=white ,circle,inner sep=1.5pt},
        rotate=270, ]
\node[fill=white, draw=none] (lab3)   at (0.5,3.5) {\footnotesize{3.}} ;
\node[fill=black] (1) at (0,0) {} ;
\node[fill=black] (2) at (1,0) {}  ;
\node (3) at (1,-2) {} ;
\node (4) at (2,4) {} ;
\node (7) at (2,-2) {} ;
\node (8) at (3,4) {} ;
\node (9) at (3,2) {} ;
\node (10) at (3,0) {} ;
\node (11) at (3,-2) {} ;
\node (12) at (4,3) {} ;
\node (13) at (4,1) {};
\draw (1) -- (2);
\draw (1) -- (3);
\draw (2) -- (4);
\draw (2) -- (6);
\draw (2) -- (7);
\draw (4) -- (8);
\draw (9) -- (12);
\draw (9) -- (13);
\draw (2) -- (9);
\draw (2) -- (10);
\draw (2) -- (11);
\end{tikzpicture}

\begin{tikzpicture}[x=\XS,y=\XS*0.4,
        every node/.style={draw=black, fill=white ,circle,inner sep=1.5pt},
        rotate=270, ]
\node[fill=white, draw=none] (lab4)   at (0.5,3.5) {\footnotesize{4.}} ;
\node[fill=black] (1) at (0,0) {} ;
\node[fill=black] (2) at (1,0) {}  ;
\node (3) at (1,-2) {} ;
\node (4) at (2,4) {} ;
\node[draw=none, fill=gray!30, inner sep=3pt] (6) at (2,0) {} ;
\node (7) at (2,-2) {} ;
\node (8) at (3,4) {} ;
\node (9) at (3,2) {} ;
\node (10) at (3,0) {} ;
\node (11) at (3,-2) {} ;
\node (12) at (4,3) {} ;
\node (13) at (4,1) {};
\draw (1) -- (2);
\draw (1) -- (3);
\draw (2) -- (4);
\draw (2) -- (6);
\draw (2) -- (7);
\draw (4) -- (8);
\draw (6) -- (9);
\draw (6) -- (10);
\draw (6) -- (11);
\draw (9) -- (12);
\draw (9) -- (13);
\draw[thick,dotted](1,0) circle (3.5pt);
\draw[thick,dotted] (3,4) circle (3.5pt);
\draw[thick,dotted] (2,4) circle (3.5pt);
\draw[thick,dotted] (2,-2) circle (3.5pt);
\end{tikzpicture}

\begin{tikzpicture}[x=\XS,y=\XS*0.4,
        every node/.style={draw=black, fill=white ,circle,inner sep=1.5pt},
        rotate=270, ]
\node[fill=white, draw=none] (lab5)  at (0.5,3.5) {\footnotesize{5.}} ;
\node[fill=black] (1) at (0,0) {} ;
\node[fill=black] (2) at (1,0) {}  ;
\node (3) at (1,-2) {} ;
\node (4) at (2,4) {} ;
\node[fill=black] (6) at (2,0) {} ;
\node (7) at (2,-2) {} ;
\node (8) at (3,4) {} ;
\node (9) at (3,2) {} ;
\node (10) at (3,0) {} ;
\node (11) at (3,-2) {} ;
\node (12) at (4,3) {} ;
\node (13) at (4,1) {};
\draw (1) -- (2);
\draw (1) -- (3);
\draw (2) -- (4);
\draw (2) -- (6);
\draw (2) -- (7);
\draw (4) -- (8);
\draw (6) -- (9);
\draw (6) -- (10);
\draw (6) -- (11);
\draw (9) -- (12);
\draw (9) -- (13);
\end{tikzpicture}


\end{minipage}

\end{example}

\begin{example}\label{ex:vc1thresh}
    To provide another example, we look at the threshold functions on  $\mathcal{X}=[0,1]$: $$\mathcal{F}=\{f:\mathcal{X}\rightarrow \{0,1\}| f(x)=\mathbbm{1}_{\{x\leq a\}}: a\in [0,1]\}.$$ 
    Every element $f\in \mathcal{F}$ is an initial segment w.r.t. the comparison function $f=0$ and thus $\VCdim(\mathcal{F})=1$ according to the Structure Theorem \ref{theo:structure}. Furthermore, as the set $(\mathcal{X},\leq)$ can be totally ordered, we know that for $x,y\in \mathcal{X}$, $x\leq y$ and $f(x)=0$ it follows that $f(y)=0$ for any $f\in \mathcal{F}$.\\
    In this case the Algorithm \ref{alg:vcdim1} is simply the disagreement-based learner from Section \ref{sec:results_real}. The learner will always abstain from predicting if the new point falls in the disagreement region: for $\hat{x}_t$, $1\leq t\leq T$ if $\hat{x}_t\in \mathcal{S}_1(\mathcal{F}_{t-1})$ then only the biggest positive point in $S_{t-1}$ $x^*_{S_{t-1}}$ can fall in $\Gamma(S_{t-1},\mathcal{F}^{\hat{x}_t\rightarrow 0})$ as all other points $y\in S_{t-1}$ with $\hat{x}_t < y$ will no longer be in the disagreement region when $\hat{x}_t\rightarrow 0$ and all points $z\in S_{t-1}$ with $z < x^*_{S_{t-1}}$ will never fall in the leave-one-out disagreement region as $f(x^*_{S_{t-1}})=1$.
    Therefore $a_0=|\Gamma(S_{t-1},\mathcal{F}^{\hat{x}_t\rightarrow 0},0)|\leq 1$ and if $T\geq 2$, we will always abstain as $a_0<\alpha$.
\end{example}

Now we are equipped to show the error guarantees for Algorithm \ref{alg:vcdim1} in VC dimension 1.
\begin{theorem}
    Let $\mathcal{F}$ be a hypothesis class over some domain $\mathcal{X}$ with $\VCdim(\mathcal{F})=1$. Then the learner as described in Algorithm \ref{alg:vcdim1} with time horizon $T$ achieves the following guarantees for a parameter $\alpha=\sqrt{T/\log(T)}$
    \begin{align*}
        \MCE &\leq \sqrt{T\log(T)}\\
        \mathbb{E}(\AbE)&\leq \sqrt{T\log(T)}
    \end{align*}
\end{theorem}

Compared to the simple disagreement-based learner for one dimensional thresholds, we pay in the misclassification error to reduce the $\mathbb{E}(\AbE)$ by a factor of $\alpha$.\\
In order to prove these bounds, we introduce use $\gamma_t=\gamma(S_t,\mathcal{F})=|\Gamma(S_t,\mathcal{F})|$ to keep track of the number of points in the estimated disagreement region at time $t$. We use the fact that if we make a mistake $\gamma_t$ will decrease by at least $\alpha$. This can be observed in the example \ref{ex:vc1tree}.

\begin{lemma}\label{lem:mis_error}
    For any $t\in \{1,\dots, T\}$ we have
    $$\gamma_t\leq \gamma_{t-1}-\alpha\cdot \mathbbm{1}_{\{\text{Misclassification at time $t$}\}}+1$$
\end{lemma}

\begin{proof}
    It is easy to see that in any round with no mistake (i.e. a correct prediction or abstention) we have $\gamma_t\leq \gamma_{t-1}+1$, as at most $\hat{x}_t$ might be added to get $\Gamma(S_t,\mathcal{F})$.\\
    The only way to make a wrong prediction in a round $t$ is if $a_0\geq \alpha$, predicting $\hat{y}_t=f(\hat{x}_t)$ but the true label would actually be the opposite $y_t=1-f(\hat{x}_t)$. In this case the following holds
    \begin{equation}\label{eq:misclas_eq1}
        |\Gamma(S_t,\mathcal{F})|\leq |\Gamma(S_{t-1},\mathcal{F}^{\hat{x}_t\rightarrow 1-f(\hat{x}_t)})|+1.
    \end{equation}
    This is true because we have $S_t=S_{t-1}\cup \{(\hat{x}_t,1-f(\hat{x}_t))\}$ and as $y_t=1-f(\hat{x}_t)$ we also have $(S_t)_f=(S_{t-1})_f\cup \{(\hat{x}_t,1-f(\hat{x}_t))\}$ and therefore 
    \begin{align*}
        |\Gamma(S_{t-1},\mathcal{F}^{\hat{x}_t\rightarrow 1-f(\hat{x}_t)})|&=|\{x:\exists y (x,y)\in S_{t-1} \text{ and } x\in \mathcal{S}_1(\mathcal{F}|_{(\hat{x}_t,1-f(\hat{x}_t))\cup (S_{t-1})_f\setminus(x,y)})\}|\\
        &=|\{x:\exists y (x,y)\in S_{t-1} \text{ and } x\in \mathcal{S}_1(\mathcal{F}|_{(S_t)_f\setminus(x,y)})\}|
    \end{align*}
    where compared to $|\Gamma(S_t,\mathcal{F})|$ only at most one point $\hat{x}_t$ might be added. \\
    Furthermore note that in general VC dimension 1 the following holds
    \begin{equation}\label{eq:misclas_eq2}
        |\Gamma(S_{t-1},\mathcal{F}^{\hat{x}_t\rightarrow1-f(\hat{x}_t)})|+|\Gamma(S_{t-1},\mathcal{F}^{\hat{x}_t\rightarrow f(\hat{x}_t)})|\leq |\Gamma(S_{t-1},\mathcal{F})|
    \end{equation}
    as in the first two summands one projects $\hat{x}_t$ to 0 and the other to 1 and the corresponding two sets of points in the estimated disagreement region are disjoint (if they were not disjoint, then we would have found two points that can be shattered thus contradicting VC dimension 1).\\
    Putting these two inequalities together, in case of a mistake, the following holds
    \begin{align*}
        \gamma_t=|\Gamma(S_t,\mathcal{F})|\overset{\eqref{eq:misclas_eq1}}{\leq} &|\Gamma(S_{t-1},\mathcal{F}^{\hat{x}_t\rightarrow 1-f(\hat{x}_t)})|+1\\
        \overset{\eqref{eq:misclas_eq2}}{\leq}&|\Gamma(S_{t-1},\mathcal{F})|- |\Gamma(S_{t-1},\mathcal{F}^{\hat{x}_t\rightarrow f(\hat{x}_t)})|+1\\
        \leq &|\Gamma(S_{t-1},\mathcal{F})|-\alpha +1=\gamma_{t-1}-\alpha +1
    \end{align*}
    where the last inequality is due to the condition of predicting $a_0\geq \alpha$. Thus giving $$\gamma_t\leq \gamma_{t-1}-\alpha\cdot \mathbbm{1}_{\{\text{Misclassification at time $t$}\}}+1.$$
\end{proof}

\begin{lemma}
    Algorithm \ref{alg:vcdim1} has $\MCE \leq \frac{T}{\alpha}$.
\end{lemma}
\begin{proof}
    We have $\MCE=\sum_{t=1}^T\mathbbm{1}_{\text{Misclassification at time $t$}}$. Thus rearranging the result from Lemma \ref{lem:mis_error} and summing gives the telescopic sum 
    $$\MCE\leq\frac{1}{\alpha}\sum_{t=1}^T(\gamma_{t-1}-\gamma_t+1)=\frac{1}{\alpha}(T+\gamma_0-\gamma_T)\leq \frac{T}{\alpha}$$
    as $\gamma_0=0$ due to $S_0=\emptyset$ and $0\leq \gamma_T \leq T$.

\end{proof}

\begin{definition}
    Let $\mathcal{F}$ be a hypothesis class over the domain $\mathcal{X}$, $f\in \mathcal{F}$ a reference function and $S$ a realizable data set. We say a point $x$ with label $y$ $(x,y)\in S$ is attackable with respect to a data set $S$, if there exists a sequence of realizable adversarial examples $A_x$ such that when the history is $S\cup A_x\setminus\{(x,y)\}$, the algorithm would abstain on $x$. In terms of Algorithm \ref{alg:vcdim1} this is equivalent to $x\in \mathcal{S}_1(\mathcal{F}|_{S\cup A_x\setminus\{(x,y)\}})$ and $|\Gamma(S\cup A_x, \mathcal{F}^{x\rightarrow f(x)},f)|<\alpha$.
\end{definition}

\begin{lemma}\label{lem:attack_points}
    Let $\mathcal{F}$ be a hypothesis class over the domain $\mathcal{X}$ with $\VCdim(\mathcal{F})=1$ and $f\in \mathcal{F}$ the reference function such that Theorem \ref{theo:structure} holds. Let $S$ be a realizable data set.\\ Then the number of attackable points is at most $\alpha$.
\end{lemma}

\begin{proof}
    According to Theorem \ref{theo:structure}, the reference function $f\in \mathcal{F}$ ensures that every element of $\mathcal{F}_f=\{\mathbbm{1}_{\{h(\cdot)\neq f(\cdot)\}}:h\in \mathcal{F}\}$ is an initial segment of the tree over $\mathcal{X}$ with respect to the partial order $\leq^\mathcal{F}_f$. 
    Thus we can process all functions $h\in \mathcal{F}$ to be XOR-ed with this reference function (as $f$ is fixed) to form initial segments. Therefore, we can assume w.l.o.g. that $f=0$ is the all-zeros function and work directly with a hypothesis class $\mathcal{F}$ which already consists of initial segments. \\
    Now also the true hypothesis $f^*\in\mathcal{F}$ corresponds to a path $p\subset\mathcal{X}$ and a threshold $x^*\in p$ such that 
    $$f^*(x)=1 \Leftrightarrow x\in p \text{ and } x\leq^\mathcal{F}_f x^*$$
    which holds by definition of an initial segment of a tree ordering. \\
    By the definition of an attackable point, we only need to consider points $x$ in the disagreement region of $S$ $\{(x,y)\in S: x\in \mathcal{S}_1(\mathcal{F}|_{S\setminus(x,y)})\}$. Therefore we focus on 
    $\{(x,y)\in S: x^*_S\prec x \text{ with } x^*_S=\max_{(x,y)\in S: y=1} x\}$: 
    all points in the subtree of the maximal positive point in $S$ as the labels of all other incoming points could be inferred. \\
    Note that an adversary adding negative points in the subtree of the deepest one only increases $|\Gamma(S\cup A_x,\mathcal{F})|$ and cannot decrease $|\Gamma(S\cup A_x,\mathcal{F}^{x\rightarrow 0})|$ thus if we remove all negative points 
    \begin{equation}\label{eq:remove_neg}
    \mathcal{S}_1(\mathcal{F}|_{S\cup A_x})\subset\mathcal{S}_1(\mathcal{F}|_{(S\cup A_x)_f}).
    \end{equation}
    Next, we want to bound the number of attackable points. We define for any node $u\in S$ the corresponding closest ancestor on the path $p$ by $\pos(u)\in p$.\\

    Let $u\in S$ be an attackable point with minimal $\pos(u)$ such that $x^*_S \prec\pos(u)$, ensuring $u$ is in the subtree of the deepest positive point. Then we can make the two following claims about $u$:
    \begin{enumerate}
        \item if $u$ is attackable and positive then it can only be $x_S^*$ (due to the linear order of positive points, only the biggest might fall in the leave one out disagreement region of $S$) \label{enum:attack_pos}
        \item if $u$ is attackable and negative then the number of other attackable points is less than $\alpha$. \label{enum:attack_neg}
    \end{enumerate}
    In order to show claim \ref{enum:attack_neg}, we assume that there exists another attackable point $(v,y_v)\in S$ which has $\pos(u)\prec\pos(v)$ by assumption. Then $v$ is no descendant of $u$ as else $v\notin\mathcal{S}_1(\mathcal{F}|_{S\setminus(v,y_v)})$ which would be a contradiction to $v$ being attackable. Furthermore, we can show 
    \begin{equation}\label{eq:vinu}
        v\in \mathcal{S}_1(\mathcal{F}|_{(u,0)\cup(S\cup A_u)_f\setminus(v,y_v)})
    \end{equation} 
    and as $v\in S$, this implies $v\in \Gamma(S\cup A_u,\mathcal{F}^{u\rightarrow 0})$. As $u$ was attackable we can bound $|\Gamma(S\cup A_u,\mathcal{F}^{u\rightarrow 0})|<\alpha$, therefore bounding the number of attackable points. \\
    In order to show equation \eqref{eq:vinu}, first, we know that there exists a set of adversarial examples $A_v$ such that $v\in \mathcal{S}_1(\mathcal{F}|_{S\cup A_v\setminus(v,y_v)})$ by assumption of $v$ being attackable. Using relation \eqref{eq:remove_neg} we know $v\in \mathcal{S}_1(\mathcal{F}|_{(S\cup A_v)_f\setminus(v,y_v)})$. Moreover, the set of adversarial examples $A_u$ corresponding to attacking $u$ does not contain positive points between $\pos(u)$ and $\pos(v)$ as if this was the case this would remove $u$ from the disagreement region $\mathcal{S}_1(\mathcal{F}|_{S\cup A_u\setminus (u,0)})$, making it no longer attackable. Therefore, we have the following implication $\mathcal{S}_1(\mathcal{F}|_{(S\cup A_v)_f\setminus(v,y_v)})\subset\mathcal{S}_1(\mathcal{F}|_{(S\cup A_u)_f\setminus(v,y_v)})$. Now lastly, as $v$ is no descendant of $u$, sending $u$ to zero has no impact on $v$ being in the disagreement region. Therefore giving equation \eqref{eq:vinu}.\\
    Putting together claims \ref{enum:attack_pos} and \ref{enum:attack_neg}, we see that the number of attackable examples is $\leq \alpha$.
\end{proof}

\begin{remark}
    The proof of the previous Lemma \ref{lem:attack_points} corrects a mistake from \cite{goel2023AdversarialResilienceSequential} when bounding the number of attackable points. 
\end{remark}

\begin{lemma}
    Algorithm \ref{alg:vcdim1} has $\mathbb{E}(\AbE)\leq\alpha \log(T)$.
\end{lemma}

\begin{proof}
    Let $n$ be the number of i.i.d. points in $S_T$. Note that if we abstain on an i.i.d. point, this point was attackable. By Lemma \ref{lem:attack_points} we know that the number of attackable examples for a dataset $S_t$ from time step $t\in \{1,\dots,T\}$, is at most $\alpha$. Assume that we have observed $i-1$ i.i.d. points so far, then the probability to abstain on the $i$-th point is bounded by $\alpha/i$ due to exchangeability. Thus
    \begin{align*}
        \mathbb{E}(\AbE)&=\mathbb{E}(\sum_{t=1}^T \mathbbm{1}_{\{c_t=0 \text{ and } \hat{y}_t=\perp\}})\\
        &=\mathbb{E}(\sum_{i=1}^n \mathbbm{1}_{\{\hat{y}_{t_i}=\perp\}})\\
        &=\sum_{i=1}^n P(\hat{y}_{t_i}=\perp)\leq \sum_{i=1}^n \frac{\alpha}{i}\leq \alpha\log(T)
    \end{align*}
    where $t_i$ is the timestep in which the incoming point is the $i$-th i.i.d. point. 
\end{proof}

\begin{remark}
    Algorithm \ref{alg:vcdim1} has undergone substantial revisions compared to its initial version presented in the earlier version of the paper \cite{goel2023AdversarialResilienceSequential}, which were based on this thesis. In order to underline the specific choices that have been made, we mention what has changed since the initial formulation of the algorithm.\\ 
    Originally, the algorithm was able to not only predict according to $f(\hat{x}_t)$ but made predictions within $\{0,1\}$ (without explicit coherence to the reference function $f$). The two quantities $a_0=|\Gamma(S_{t-1},\mathcal{F}^{\hat{x}_t\rightarrow 0},f)|$ and $a_1=|\Gamma(S_{t-1},\mathcal{F}^{\hat{x}_t\rightarrow 1},f)|$ were defined and then the learner chose to predict if $\max(a_0,a_1)\geq \alpha$ with the label $\{0,1\}$ which maximized the argument. \\
    This however led to contradictions in the proof of Lemma \ref{lem:mis_error} as the quantity $\gamma_t$ does not necessarily show the behaviour needed in order to bound the misclassification error in the same way as we do now. Specifically the quantity $\gamma_t$ does not decrease by at least $\alpha$ in case of a misclassification. An easy example for one dimensional thresholds showed this (see Example \ref{ex:vc1thresh}), as if there are $>\alpha$ negative points, then the algorithm will always predict 1 (as $a_0\leq 1$ and $a_1\geq \alpha$) while maintaining the same size of $\gamma_t$.\\
    In order to remedy this, the algorithm was simplified to only predict according to the reference function, if the threshold $\alpha$ is exceeded. Through this and the way $\Gamma$ is defined, we are able to prove the bound on the number of misclassifications rigorously by exploiting that the points which disagree with the reference function $S_f=\{(x,y):y\neq f(x)\}$ define the disagreement region in $\Gamma$. Furthermore the proof in Lemma \ref{lem:attack_points} needed to be adapted to the new algorithm. It is based on the same arguments (bounding the number of attackable points), but one has to be more careful when establishing which points are included in $\Gamma$ and which are not. With respect to this, we want to stress the importance of establishing an intuition what the leave-one-out disagreement estimate entails through easy examples as given in Example \ref{ex:vc1tree} and \ref{ex:vc1thresh}, as it is easy to confuse the disagreement region $\mathcal{S}_1(\mathcal{F}_{t-1})$ with $\Gamma(S_{t-1},\mathcal{F},f)$ and to think of which points are excluded when looking at $\Gamma(S_{t-1},\mathcal{F}^{\hat{x}_t\rightarrow f(\hat{x}_t)},f)$.\\

\end{remark}

\newpage

\subsubsection{Axis-aligned rectangles in $p$ dimensions}

The class of axis-aligned rectangles in $\mathcal{X}=\mathbb{R}^p$ is defined as a function of the following form, parametrized by $a_1,b_1,\dots,a_p,b_p\in \mathbb{R}$ with $a_i\leq b_i$ for all $i\in \{1,\dots,p\}$:
\begin{align*}
f_{(a_1,b_1,\dots,a_p,b_p)}\colon \quad \mathbb{R}^p &\longrightarrow \{0,1\}\\
    x &\longmapsto
    \begin{cases}
        1 & \text{if } a_i\leq (x)_i \leq b_i \quad \forall i \in \{1, \dots , p\}\\
        0 & \text{else }
    \end{cases}
\end{align*}
Note that the function class of axis-aligned rectangles in $\mathbb{R}^p$ has VC dimension $2p$. 

The following Algorithm \ref{alg:axis-aligned} exploits the structure of axis-aligned rectangles in $p$ dimensions to facilitate learning. For a visual description of how this algorithm decides to predict, see Figure \ref{fig:axis-aligned}.

\begin{algorithm}[h]
    \caption{Structure based algorithm for axis-aligned rectangles in dimension $p$}\label{alg:axis-aligned}
    \begin{algorithmic}[1]
        \State Set $a_1,\dots , a_p=-\infty$, $b_1,\dots,b_p=+\infty$, $\mathcal{F}_0=\mathcal{F}$ and $S_0=\emptyset$
        \For{$t=1,\dots,T$}
        \State Receive $\hat{x}_t\in \mathcal{X}\subset \mathbb{R}^p$
        \If{$y_s=0 \quad \forall s<t$} \label{alg:line:cond1}
        \State $\hat{y}_t=0$

        \ElsIf{$\hat{x}_t\notin \mathcal{S}_1(\mathcal{F}_{t-1})$} \label{alg:line:cond2}
        \State $\hat{y}_t=f(\hat{x}_t)$ for any $f\in \mathcal{F}_{t-1}$

        \ElsIf{$\exists s_1,\dots,s_\alpha<t$ and $\exists i_1,\dots, i_\alpha\in \{1,\dots,p\}$ s.t.: \par $(\hat{x}_{s_j})_{i_j}\in [(\hat{x}_t)_{i_j},a_{i_j})\cup(b_{i_j},(\hat{x}_t)_{i_j}] \quad \forall j\in \{1,\dots,\alpha\}$} \label{alg:line:cond3}
        \State $\hat{y}_t=0$

        \Else
        \State $\hat{y}_t=\perp$
        \EndIf        
        \State Receive true label $y_t$
        \State Update $S_t\gets S_{t-1}\cup\{(\hat{x}_t,y_t)\}$ and $\mathcal{F}_t\gets \mathcal{F}_{t-1}^{\hat{x}_t\rightarrow y_t}$
        \State Update $a_1,\dots,a_p,b_1,\dots,b_p$ s.t. \par
     $a_i=\min\{(x)_i:(x,1)\in S_t\}$ and $b_i=\max\{(x)_i:(x,1)\in S_t\}$
        \EndFor
    \end{algorithmic}
\end{algorithm}

\begin{theorem}\label{theo:axis_guarantee}
    Let $p,T \in \mathbb{N}^+$ and $\mathcal{F}$ be the class of axis-aligned rectangles in $\mathbb{R}^p$.\\
    Then Algorithm \ref{alg:axis-aligned} with $\alpha=\sqrt{T/\log{T}}$ satisfies 
    \begin{align*}
        \MCE&\leq p\sqrt{T\log(T)}\\
        \mathbb{E}(\AbE)&\leq 2p\sqrt{T\log(T)}+2p\log(T)
    \end{align*}
\end{theorem}

\begin{proof} We first bound the misclassification and then the abstention error. \\
\textit{Misclassification error}: Misclassifications in Algorithm \ref{alg:axis-aligned} are possible for the loops in line \ref{alg:line:cond1} or line \ref{alg:line:cond3}, as in the loop of line \ref{alg:line:cond2} we always predict according to the correct consistent label. \\ 
In the first loop of line \ref{alg:line:cond1} at most one error might occur throughout the run time of the algorithm, whenever the first positive example is received. In the loop of line \ref{alg:line:cond3} a misclassification in round $t$ occurs, if $\hat{x}_t$ is a positive example with more than $\alpha$ points $\hat{x}_s$, $s<t$, for which there exist a coordinate $1\leq i\leq p $ such that $(\hat{x}_{s})_{i}\in [(\hat{x}_t)_{i},a_{i})\cup(b_{i},(\hat{x}_t)_{i}]$. Thus, after the true label $y_t=1$ is received and $a_j,b_j$ are updated for all $1\leq j \leq p$, the coordinates $(\hat{x}_s)_i$ now lie in $[a_i,b_i]$. For each point this is possible at most $p$ times and thus can also serve as an indicator to predict 0 at most $p$ times. Until time $T$ the algorithm can therefore make at most $\frac{Tp}{\alpha}+1$ misclassifications. \\

\textit{Abstention error}: Let $\hat{x}_t$ be the $n$-th i.i.d. sample such that $x_1,\dots,x_n$ are all i.i.d. examples observed with $\hat{x}_t=x_n$. We abstain on $\hat{x}_t$ only if
we have observed one positive example prior \textit{and} $\hat{x}_t\in \mathcal{S}_1(\mathcal{F}_{t-1})$ \textit{and} $\hat{x}_t$ has less than $\alpha$ other points $\hat{x}_s$, $s<t$, such that $(\hat{x}_{s})_{i}\in [(\hat{x}_t)_{i},a_{i})\cup(b_{i},(\hat{x}_t)_{i}]$.\\
This implies that if we abstain on $\hat{x}_t$, then $\hat{x}_t$ falls in the disagreement region of the version space induced by the other $x_1,\dots,x_{n-1}$ i.i.d. points and there are less than $\alpha$ other i.i.d. points $x_s$, $s<n$, with coordinates $(x_{s})_{i}\in [(x_n)_{i},a_{i})\cup(b_{i},(x_n)_{i}]$. Furthermore we fix the rectangle corresponding to the true hypothesis $[a_1^*,b_1^*]\times \dots \times [a_p^*,b_p^*]$ (note this contains all rectangles of the algorithm for all $t\in \{1,\dots,T\}$). Thus, if we abstain on $\hat{x}_t$, then there are also less than $\alpha$ other i.i.d. points $x_s$, $s<n$, with coordinates $(x_{s})_{i}\in [(x_n)_{i},a_{i}^*)\cup(b_{i}^*,(x_n)_{i}]$.\\
We differentiate between the point $\hat{x}_t$ being positive or negative and exploit that in this case these events can easily be bounded. Note that
\begin{align} \label{eq:AbE_bound}
    \begin{split}
    P(\hat{y}_t=\perp)&= P(\hat{y}_t=\perp \text{and } y_t=1 )+ P(\hat{y}_t=\perp \text{and } y_t=0 )\\
    &\leq P(\hat{x}_t \in \mathcal{S}_1(\mathcal{F}_{\{x_1,\dots,x_{n-1}\}}) \text{ and } y_t=1) + \\
    &P(\exists s_1,\dots,s_\alpha<n, \exists i_1,\dots, i_\alpha\in \{1,\dots,p\}: \\&(x_{s_j})_{i_j}\in [(\hat{x}_t)_{i_j},a_{i_j}^*)\cup(b_{i_j}^*,(\hat{x}_t)_{i_j}] \quad \forall j\in \{1,\dots,\alpha\} \text{ and }y_t=0)
    \end{split}
\end{align}
First, we use exchangeability to bound the probability that a positive $\hat{x}_t=x_n$ falls in the disagreement region of the other $x_1,\dots,x_{n-1}$ points. Note that this implies that $x_n$ lies outside of the smallest rectangle containing the positive points of $x_1,\dots,x_{n-1}$. If we exchange $x_n$ with some $x_s$, $s\leq n$, then among all those possibilities, $x_s$ falls in the disagreement region of the version space induced by the others at most $2p$ times (for every coordinate only the minimal or maximal positive point has the possibility to fall in the disagreement region). Therefore by exchangeability $$P(\hat{x}_t \in \mathcal{S}_1(\mathcal{F}|_{\{x_1,\dots,x_{n-1}\}}) \text{ and } y_t=1)\leq \frac{2p}{n}.$$
For the second term, we will also use exchangeability of the i.i.d. samples to bound the probability that for a negative example $\hat{x}_t$, there are less than $\alpha$ examples $x_s$, $s<n$, with $(x_{s})_{i}\in [(\hat{x}_t)_{i},a_{i}^*)\cup(b_{i}^*,(\hat{x}_t)_{i}]$. If this is the case then in \textit{every dimension} $i \in \{1,\dots,p\}$ the point has less than $\alpha$ coordinates of other points falling in $[(\hat{x}_t)_{i},a_{i}^*)\cup(b_{i}^*,(\hat{x}_t)_{i}]$. If $\hat{x}_t=x_n$ is swapped with another $x_s$, $s<n$, then at most $2\alpha$ swaps will also allow for less than $\alpha$ points between itself and the boundary  (exactly the $\alpha$ biggest coordinates smaller than $a_i^*$ and the $\alpha$ smallest coordinates bigger than $b_i^*$).  \\
Notice that negative points $\hat{x}_t$ with $(\hat{x}_t)_i\in [a_i^*,b_i^*]$ have 0 points falling in the interval between itself and the boundary (as this interval is empty). However, as when $\hat{x}_t$ is negative $\exists j \in \{1,\dots,p\}$ such that $(\hat{x}_t)_j\notin [a_j^*,b_j^*]$ and therefore in this dimension $j$, we need less than $\alpha$ coordinates suggesting to predict 0 and these are again only the $\alpha$ biggest smaller than $a_j^*$ and the $\alpha$ smallest bigger than $b_j^*$.\\
Therefore we can bound the second term by $\frac{2\alpha p}{n}$. Adding both bounds, we get
\begin{align*}
    \mathbb{E}(\AbE)=\mathbb{E}(\sum_{t=1}^T\mathbbm{1}_{\{f(\hat{x}_t)=\perp\}})&=\mathbb{E}(\sum_{t=1}^n\mathbbm{1}_{\{f(x_t)=\perp\}})\\
    &=\sum_{t=1}^n P(f(x_t)=\perp)\\
    &\leq \sum_{t=1}^n \frac{2p(\alpha +1)}{t}\\
    &\leq 2p(\alpha +1)\log(n)\\
    &\leq 2p(\alpha +1)\log(T)    
\end{align*}
where in the first inequality we used the bounds we derived for equation \eqref{eq:AbE_bound} and in last inequality, we used that the number of i.i.d. samples received so far $n$ is less than $T$. Setting $\alpha=\sqrt{T/\log(T)}$ gives the desired result.

\end{proof}

\begin{remark}
    The proof of the previous Theorem \ref{theo:axis_guarantee} is a more detailed version of the proof from \cite{goel2023AdversarialResilienceSequential} in order to give a rigorous argumentation for each claim. 
\end{remark}

\begin{remark}
    The way this algorithm is structured reflects in a way Algorithm $\ref{alg:vcdim1}$ for the case of VC dimension 1 and an unknown distribution, in the sense that if we chose to predict when the incoming point $\hat{x}_t$ falls in the disagreement region, then we only predict 0 (in Algorithm 2 it is according to $f(\hat{x}_t)=0$ which through a transformation of the function class according to \ref{theo:structure} w.l.o.g. can be set to 0). Furthermore, we only predict in this case if we have found enough 'evidence': for axis-aligned rectangles this is according to line \ref{alg:line:cond2} and for VC dimension 1 this is according to an estimate of the disagreement region $|\Gamma(S_{t-1},\mathcal{F}^{\hat{x}_t\rightarrow f(\hat{x}_t)},f)|$. In case of a misclassification in both cases, the points which served as evidence are removed for future consideration. \\
    This shows the structural significance of the Algorithm \ref{alg:vcdim1} which, when adapted, works also for the case of axis-aligned rectangles where $\VCdim(\mathcal{F})>1$. \\
    Again for both uses of the algorithm, a strong intuition is required and precise formulation in the corresponding proofs. 
\end{remark}

\begin{figure}
    \centering

\begin{minipage}{0.48\textwidth}
    
\begin{tikzpicture}
\begin{axis}[%
    name=plot1,
    axis lines=middle,
    xlabel= x,
    ylabel=y,
    xmin=-5, xmax=5,
    ymin=-5, ymax=5,
    xtick=\empty, ytick=\empty
]

\addplot [only marks, mark=+,mark options={very thick}] table {
0 2
1 1
3 0.5
4 0
};

\addplot [only marks, mark=-,mark options={very thick}] table {
-1 3
-2 4
-3 -1
2 -4
-4 -1.75
3 4
};

\addplot [only marks, mark= o] table {
-2 -2
};

\shade[top color=black, bottom color=black, opacity=0.1] (-3,-4) rectangle (5,3);
\shade[top color=black, bottom color=black, opacity=0.1] (-5,-1) rectangle (-3,3);
\shade[top color=black, bottom color=black, opacity=0.1] (-1,3) rectangle (5,4);

\draw[thick, dashed,fill=white] (0,0) rectangle (4,2);

\node[fill=white, draw=none]   at (-4,4) {\footnotesize{1.}} ;

\end{axis}

\end{tikzpicture}
\end{minipage}
\begin{minipage}{0.4\textwidth}
\begin{tikzpicture}
\shade[top color=black, bottom color=black, opacity=0.1] (0,0.6) rectangle (1,0.9);
\node[anchor=west] at (1.25,0.75) {\footnotesize Disagreement Region};
\draw[thick, dashed,fill=white] (0,0.1) rectangle (1,0.4);
\node[anchor=west] at (1.25,0.25) {\footnotesize Smallest Rectangle Positive Points};
\node[anchor=south] at (0.5,1) {\textbf{+}};
\node[anchor=west] at (1.25,1.25) {\footnotesize Positive Point};
\node[anchor=south] at (0.5,1.5) {\textbf{-}};
\node[anchor=west] at (1.25,1.75) {\footnotesize Negative Point};
\node[anchor=south] at (0.5,2) {$\circ$};
\node[anchor=west] at (1.25,2.25) {\footnotesize New unlabeled Point};
\draw[pattern=north west lines, pattern color=black, opacity=0.3,draw=none] (0,-0.4) rectangle (1,-0.1);
\node[anchor=west] at (1.25,-0.25) {\footnotesize $[(x)_i,a_i)\cup(b_i,(x)_i]$};
\end{tikzpicture}
\end{minipage}

\begin{tikzpicture}
\begin{axis}[
    name=plot2,
    axis x line=center,
    axis y line=none,
    xlabel= x,
    xmin=-5,xmax=5,
    ymin=-5, ymax=5,
    xtick=\empty, ytick=\empty,
    extra x ticks={0,4},
    extra x tick labels={$a_1$,$b_1$},
    ]
\addplot [only marks, mark=+,mark options={very thick}] table {
0 2
1 1
3 0.5
4 0
};

\addplot [only marks, mark=-,mark options={very thick}] table {
-1 3
-2 4
-3 -1
2 -4
-4 -1.75
3 4
};

\addplot [only marks, mark= o] table {
-2 -2
};

\draw[->] (0,2-0.25) -- (0, 0);
\draw[->] (1,1-0.25) -- (1, 0);
\draw[->] (3, 0.5-0.25) -- (3, 0);
\draw[->] (-1,3-0.25) -- (-1, 0);
\draw[->] (-2,4-0.25) -- (-2, 0);
\draw[->] (-3,-1+0.25) -- (-3, 0);
\draw[->] (2,-4+0.25) -- (2, 0);
\draw[->] (-4,-1.75+0.25) -- (-4, 0);
\draw[->] (3,4-0.25) -- (3, 0);

\draw[pattern=north west lines, pattern color=black, opacity=0.3, draw=none] (-2,-5) rectangle (0,5);

\node[fill=white, draw=none]   at (-4,4) {\footnotesize{2.}}; 

\end{axis}
\end{tikzpicture}
\begin{tikzpicture}
\begin{axis}[
    name=plot3,
    axis x line=none,
    axis y line=center,
    every axis x label/.style={at=(current axis.right of origin),anchor=west},
    ylabel= y,
    xmin=-5,xmax=5,
    ymin=-5, ymax=5,
    xtick=\empty, ytick=\empty,
    extra y ticks={0,2},
    extra y tick labels={$a_2$,$b_2$},
    ]
    \addplot [only marks, mark=+,mark options={very thick}] table {
0 2
1 1
3 0.5
4 0
};

\addplot [only marks, mark=-,mark options={very thick}] table {
-1 3
-2 4
-3 -1
2 -4
-4 -1.75
3 4
};

\addplot [only marks, mark= o] table {
-2 -2
};

\draw[->] (1-0.25,1) -- (0, 1);
\draw[->] (3-0.25, 0.5) -- (0, 0.5);
\draw[->] (4-0.25,0) -- (0, 0);
\draw[->] (-1+0.25,3) -- (0, 3);
\draw[->] (-2+0.25,4) -- (0, 4);
\draw[->] (-3+0.25,-1) -- (0, -1);
\draw[->] (2-0.25,-4) -- (0, -4);
\draw[->] (-4+0.25,-1.75) -- (0, -1.75);
\draw[->] (3-0.25,4) -- (0, 4);


\draw[pattern=north west lines, pattern color=black, opacity=0.3,draw=none] (-5,-2) rectangle (5,0);

\node[fill=white, draw=none]   at (-4,4) {\footnotesize{3.}}; 

\end{axis}
\end{tikzpicture}

\begin{minipage}{0.48\textwidth}
\begin{tikzpicture}

\begin{axis}[%
    name=plot4,
    axis lines=middle,
    xlabel= x,
    ylabel=y,
    xmin=-5, xmax=5,
    ymin=-5, ymax=5,
    xtick=\empty, ytick=\empty
]

\addplot [only marks, mark=+,mark options={very thick}] table {
0 2
1 1
3 0.5
4 0
-2 -2
};

\addplot [only marks, mark=-,mark options={very thick}] table {
-1 3
-2 4
-3 -1
2 -4
-4 -1.75
3 4
};

\shade[top color=black, bottom color=black, opacity=0.1] (-3,-4) rectangle (5,3);

\draw[thick, dashed,fill=white] (-2,-2) rectangle (4,2);

\draw (-2,0) -- (4,0);
\draw (0,-2) -- (0,2);

\node[fill=white, draw=none]   at (-4,4) {\footnotesize{4.}};

\end{axis}

\end{tikzpicture}
\end{minipage}

    \caption{An example of a classification problem for axis-aligned rectangles. Given a set of 10 points which are already assigned a positive or negative label a new point is received by the algorithm (see plot 1). The algorithm decides to predict (see plot 2 and 3) - as there are 4 points which convince us to predict 0 (take $\alpha =\sqrt{11/\log(11)}<3$). This would be a misclassification, as the new point is actually a postive point. See plot 4 for the updated disagreement region and rectangle containing the positive points.}
    \label{fig:axis-aligned}
\end{figure}
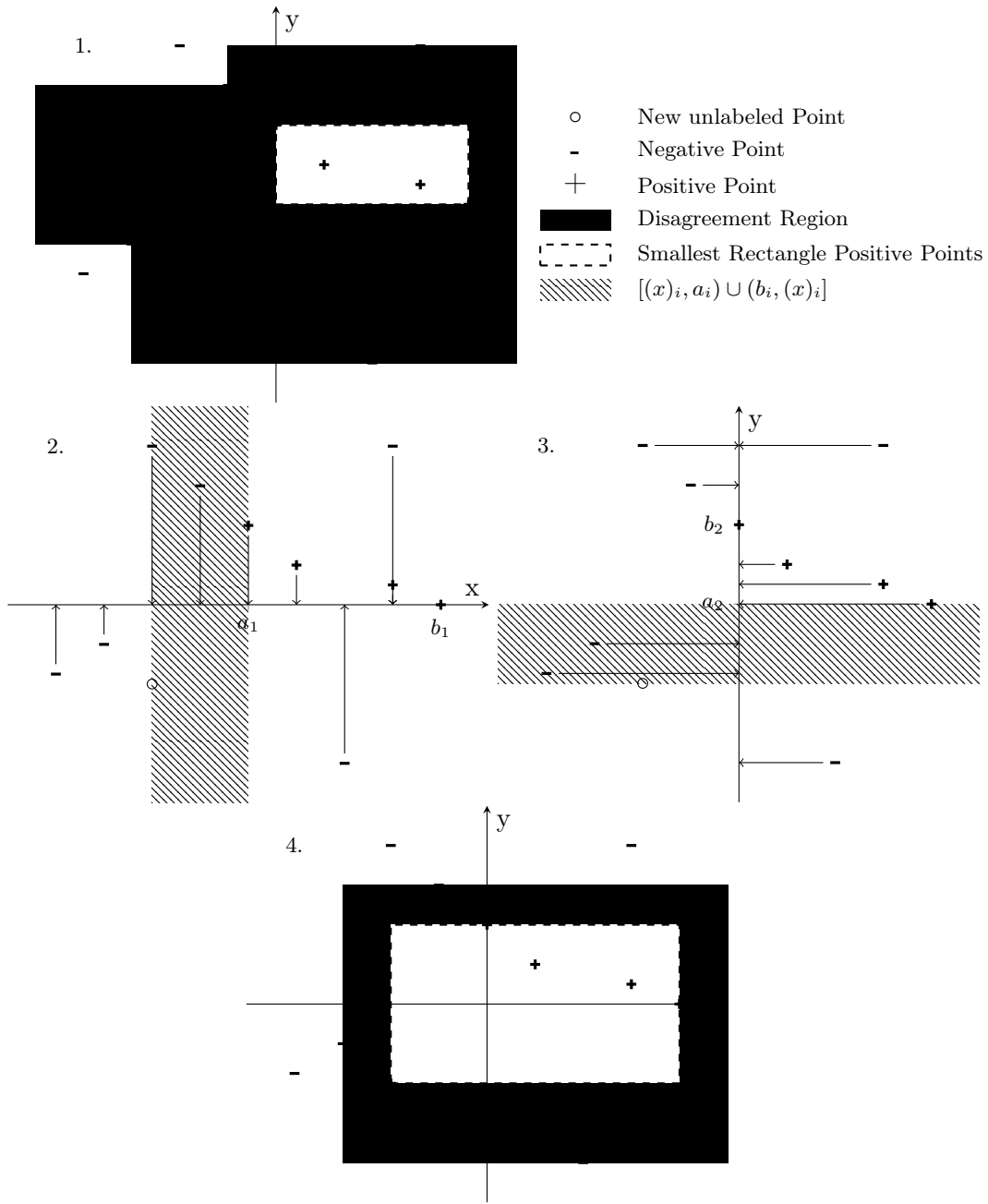

\newpage

\chapter{Agnostic Setting}\label{chap:agno}

\section{Literature for the Agnostic Setting}\label{sec:lit_agno}
\subsection{Setting}\label{sec:setting_agno}

We transition from the realizable case to the \textit{agnostic case}. This contains the realizable scenario but extends it to an even broader context where we can account for a misspecified model, random labels or inaccurate labels due to noise. The term 'agnostic' denotes a state of not knowing and this represents noise or inaccuracies in real data classification tasks better where data comes from real-life applications. The realizable setting is however easier to work with.\\

To recall, we aim to find the best possible classifier from the hypothesis space $\mathcal{F}=\{f:\mathcal{X}\rightarrow \{0,1\}\}$ through samples sequentially presented to the algorithm $(x_i,y_i)\in \mathcal{X}\times\{0,1\}$.
However in this (non-adversarial) setting, we now we have data points where the samples $x_i$ \textit{and} their corresponding labels $y_i$ are drawn i.i.d. from some fixed but generally unknown joint probability distribution $\mathcal{P}_{\mathcal{X},\mathcal{Y}}$ over $\mathcal{X}\times \mathcal{Y}$. Given the sequential nature of the online learning problem, we primarily focus on the marginal distribution over $\mathcal{X}$ (again denoted by $\mathcal{D}$), from which samples $x_i$ are drawn i.i.d., and the conditional label distribution $\eta(x)=P(Y=+1|X=x)$ where $(X,Y)\sim \mathcal{P}_{\mathcal{X},\mathcal{Y}}$. Therefore, the label 1 is assigned to a sample $x_i$ with probability $\eta(x_i)$ and the label 0 with probability $1-\eta(x_i)$. Note that in the realizable setting we had the existence of a function $f^*$ from the hypothesis space $\mathcal{F}$ such that $P(Y=f^*(x)|X=x)=1$.\\
In this light we also need to re-evaluate what it means to find 'the best classifier' from the hypothesis space, as we can no longer guarantee the existence of a function with true error zero. We adapt the notion of true error in the agnostic setting $$\er_{\mathcal{P}_{\mathcal{X},\mathcal{Y}}}(f):=P_{(X,Y)\sim \mathcal{P}_{\mathcal{X},\mathcal{Y}}}(f(X)\neq Y).$$ 
The idea is still to find the classifier which minimizes the true risk, but we acknowledge that this classifier may generate inaccurate predictions, while aiming to minimize the overall number of mistakes. The noise rate of a given hypothesis space $\mathcal{F}$ is defined as $\nu(\mathcal{F})=\inf_{f\in\mathcal{F}}\er_{\mathcal{P}_{\mathcal{X},\mathcal{Y}}}(f)$ and describes the minimum mistakes any classifier will make. One can show, that for a closed set $\mathcal{F}$ this minimum exists, i.e. for some $f^*\in \mathcal{F}$ we have $\er_{\mathcal{P}_{\mathcal{X},\mathcal{Y}}}(f^*)=\nu(\mathcal{F})$ \cite{hanneke2012ActivizedLearningTransforming}. This $f^*$ is the Bayes optimal classifier. The goal of the learning task will be to find such a classifier.\\ 

We will discuss the case for which the learner knows the \textit{true marginal distribution} of samples $x\in \mathcal{X}$ and wants to learn the hypothesis $f^*$ which best describes the true conditional label distribution in the sense that it has least true error. This is similar to what was done in \cite{goel2023AdversarialResilienceSequential} for the realizable setting. Note that the learner does not know the true conditional label distribution, or else we could simply find this classifier $f^*$ in an optimization problem.\\ 

The question that remains now, is how to introduce an adversary who is injecting clean-labels in the agnostic setting when there no longer is such a thing as a labelling function. This is a new and more general viewpoint of clean-label injections which has not been addressed in the literature prior \cite{blum2021RobustLearningCleanlabel}.
In the realizable setting, when an adversary injects a sample $\hat{x}_t$ in round $t\in \{0,\dots,1\}$, we learn something about the hypothesis class through to the label assigned by the underlying true labelling function $f^*(\hat{x}_t)$. Now we want to generalize this for the case when labels are random: if an adversary injects a sample $\hat{x}_t$ coming from an unknown distribution, then after making a prediction the learner will observe the label $y_t$ which was sampled from the correct conditional label distribution $\eta(\hat{x}_t)$. We can no longer assume that the samples $x_t\in\mathcal{X}$ presented to the learner are drawn i.i.d.. However the true labels are generated according to the correct label distribution $P(Y|X=x_t)$. For this approach to be well defined, an adversary is not allowed to sample from a region where the marginal distribution of $\mathcal{X}$ has probability mass 0. \\

In this setting with noise, we redefine the objective of reducing the expected misclassification and abstention error for time horizon $T\in \mathbb{N}^+$. We assume $\hat{x}_t$ is the (potentially adversarial) input in round $t$, $\hat{y}_t\in \{0,1,\perp\}$ the prediction the algorithm produces and $y_t\in\{0,1\}$ the true label. Then we define the errors to be 
{\allowdisplaybreaks
\begin{align*}
    \MCE&=\sum_{t=1}^T \mathbbm{1}_{\{\hat{y}_t=1-y_t\}}-\mathbbm{1}_{\{f^*(\hat{x}_t)=1-y_t\}}\\
    \AbE&=\sum_{t=1}^T\mathbbm{1}_{\{x_t \text{ was drawn i.i.d. from }\mathcal{D} \text{ and } \hat{y}_t=\perp\}}
\end{align*}}
Note that we have adapted the definition of misclassification error, as it makes sense to count the number of mistakes with respect to the true target $f^*$ and not only when we disagree with the true (but noisy) label $y_t$. Therefore, giving us an opportunity to learn the target $f^*$ and not the noise.  

\subsubsection{Noise models}\label{sec:noise}
In order to guarantee bounds on misclassification and abstention error, we will impose assumptions on the distribution $\mathcal{P}_{\mathcal{X},\mathcal{Y}}$ to characterize the noise. The conditions go from most general to most restrictive assumption on the noise 
\begin{itemize}
    \item Agnostic \cite{haussler1992DecisionTheoreticGeneralizations}: the setting as described above, without any more assumptions on $\mathcal{P}_{\mathcal{X},\mathcal{Y}}$
    \item Tsybakov-Mammen \cite{mammen1999SmoothDiscriminationAnalysis,tsybakov2004OptimalAggregationClassifiers}: For $f$ an unknown target function and $\hat{x}_t\in \mathcal{X}$, the label associated to $\hat{x}_t$ is given by $f(\hat{x}_t)$ with probability $1-\eta(\hat{x}_t)$ or $1-f(\hat{x}_t)$ with probability $\eta(\hat{x}_t)$ where $\eta(x)$ is an unknown function which satisfies the Tsybakov noise condition:
    \begin{equation}\label{eq:tsyba_noise}
        \exists\mu,\alpha\in [0,\infty) \text{ s.t. }\forall \varepsilon>0\quad P_{X\sim \mathcal{D}}(|\eta(X)-1/2|\leq\varepsilon)\leq \mu\cdot \varepsilon^\alpha
    \end{equation}
    This can be interpreted as having a true labelling function (as in the realizable case), with an adversary flipping the labels with a certain probability $\eta(x)$ which can be arbitrarily close to $1/2$ and which is unknown to the learner.
    \item Massart (Bounded Noise) \cite{massart2006RiskBoundsStatistical, sloan1988TypesNoiseData, sloan1992CorrigendumTypesNoise}: For $f$ an unknown target function and $\hat{x}_t\in \mathcal{X}$, the label associated to $\hat{x}_t$ is given by $f(\hat{x}_t)$ with probability $1-\eta(x)$ or $1-f(\hat{x}_t)$ with probability $\eta(x)$ where $\eta(x)\leq \eta < 1/2$, i.e. $|\eta(x)-1/2|\geq c$ for some $c>0$. \\
    This can be interpreted as having a true labelling function (as in the realizable case), with an adversary flipping the labels with a probability $\eta(x)$ which is bounded away from $1/2$. This is a special case of the Tsybakov Noise condition.
    \item Random Classification Noise \cite{angluin1988LearningNoisyExamples}: This is a special case of Massart Noise where $\eta(x)=\eta$ $\forall x\in \mathcal{X}$. \\
    Here an adversary flips labels with a fixed probability of $\eta<1/2$.
\end{itemize}

The setting of a clean-label adversary in a classification task with noise, is described in Protocol \ref{alg:protocol_agn}.

\begin{algorithm}[h]
\floatname{algorithm}{Protocol}
    \caption{For Sequential Predictions in the Setting with Noise}\label{alg:protocol_agn}
    \begin{algorithmic}[]
    \State Nature selects a distribution $\mathcal{D}$ over $\mathcal{X}$ and a labelling function $f^*\in \mathcal{F}$
    \State The learner does not have access to $f^*$ but has access to $\mathcal{D}$
    \State Depending on the noise, we have for all $x\in \mathcal{X}$ $\eta(x)=\eta$ (RCN), $|\eta(x)-1/2|\geq c>0$ (Massart) or an $\eta(x)$ for which the Tsybakov noise condition \ref{eq:tsyba_noise} holds
    \For{$t=1,\dots, T$}
    \State \textbf{Nature} draws $x_t\sim \mathcal{D}$ independently of previous samples 
    \State \textbf{Adversary (Clean-Label Attack)}: either injects an adversarial example from an unknown distribution $\hat{x}_t$ or does nothing $\hat{x}_t=x_t$
    \State \textbf{Nature} determines true label $f^*(\hat{x}_t)$
    \State \textbf{Adversary (Noise)}: \textit{independently} of whether the example $\hat{x}_t$ was injected or not, the adversary sets $y_t=f^*(\hat{x}_t)$ with probability $\eta(\hat{x}_t)$ or flips the label $y_t=1-f^*(\hat{x}_t)$ with probability $1-\eta(\hat{x}_t)$
    \State \textbf{Learner} is presented with $\hat{x}_t$, outputs $\hat{y}_t\in \{0,1,\perp\}$ and receives potentially noisy label $y_t$
    \EndFor
    \end{algorithmic}
\end{algorithm}

\subsection{Related Work}

\subsubsection{On Agnostic Learning}

Agnostic learning was first quantified in the PAC style setting by Haussler in 1992 \cite{haussler1992DecisionTheoreticGeneralizations}. He was the first to generalize the definition of PAC-learning, responding to the critique that the realizable assumption is too restrictive and does not represent real-life learning scenarios, where labels are not deterministically determined through the input. Maintaining minimal assumptions on the origin of the noise, we just assume the existence of a distribution $\mathcal{P}_{\mathcal{X},\mathcal{Y}}$ over $\mathcal{X}\times\mathcal{Y}$ which draws samples and labels.

\begin{definition}[Agnostic PAC-learning \cite{shalev-shwartz2014UnderstandingMachineLearning}]
    Let $\mathcal{F}$ be a binary hypothesis class over $\mathcal{X}$ and $\mathcal{P}_{\mathcal{X},\mathcal{Y}}$ a probability distribution over $\mathcal{X}\times \mathcal{Y}$.\\
    We call a hypothesis class $\mathcal{F}$ agnostic Probably Approximately Correct (agnostic PAC) learnable, if there exists a function $m_\mathcal{F}:(0,1)\times (0,1)\rightarrow \mathcal{N}$ and a learner $L$ such that the following holds:
    For any $\varepsilon, \delta\in (0,1)$ and any probability distribution $\mathcal{P}_{\mathcal{X},\mathcal{Y}}$ over $\mathcal{X}\times \mathcal{Y}$, when running the learning algorithm on $m \geq m_\mathcal{F}(\varepsilon,\delta)$ i.i.d. samples drawn from  $\mathcal{P}_{\mathcal{X},\mathcal{Y}}$ the algorithm returns a hypothesis $h$. For this $h$ we have with probability of at least $1-\delta$ (over the randomness of the samples), that for the expected true risk the following bound holds: $\er_{\mathcal{P}_{\mathcal{X},\mathcal{Y}}}(h)\leq \min_{h'\in\mathcal{F}}\er_{\mathcal{P}_{\mathcal{X},\mathcal{Y}}}(h')+\varepsilon$.
\end{definition}

In the realizable setting, this reduces to the definition of PAC learning given in Definition \ref{def:pac}, as then there exists a classifier $f^*\in\mathcal{F}$ with $\min_{h'\in\mathcal{F}}\er_{\mathcal{P}_{\mathcal{X},\mathcal{Y}}}(h')=\er_{\mathcal{P}_{\mathcal{X},\mathcal{Y}}}(f^*)=0$. Note however, that in the agnostic setting, we do not give an \textit{absolute} bound of the error of the hypothesis, but we quantify the \textit{relative} deviation from the best classifier in the class. \\

We can also give a characterization of sample complexity, for a hypothesis class to be agnostic PAC-learnable:
\begin{theorem}[Fundamental Theorem of Statistical Learning Part II]\label{theo:fund_stat_learn_part2}
 Let $\mathcal{F}$ be a binary hypothesis class over $\mathcal{X}$. Then additionally to the equivalences in Theorem \ref{theo:fund_stat_learn}, the following are equivalent:
\begin{enumerate}
    \item The hypothesis class $\mathcal{F}$ is PAC-learnable
    \item The hypothesis class $\mathcal{F}$ is agnostic PAC-learnable
\end{enumerate}
Furthermore, if $\mathcal{F}$ is agnostic PAC-learnable (with finite VC dimension $d$), we have for the sample complexity there exist absolute constants $C_1,C_2$ such that:
$$C_1\frac{d+log(1/\delta)}{\varepsilon^2}\leq 
m_\mathcal{F}\leq C_2\frac{d+\log(1/\delta)}{\varepsilon^2}.$$
\end{theorem}

Furthermore, we give a result on uniform convergence bounds due to Vapnik \cite{vapnik1971UniformConvergenceRelative}, which can be used to derive confidence thresholds for the deviation between empirical error of a hypothesis $f\in\mathcal{F}$ and its true error.
\begin{lemma}[Uniform convergence bounds from \cite{bousquet2004IntroductionStatisticalLearning}]\label{lem:unif_conv}
Let $\mathcal{F}$ be a binary function class with $\VCdim(\mathcal{F})=d<\infty$ and $\mathcal{P}_{\mathcal{X},\mathcal{Y}}$ a distribution over $\mathcal{X}\times\mathcal{Y}$. \\
For any $0<\delta<1$ with probability at least $1-\delta$ over the choice of sample $S_n=\{(x_1,y_1),\dots,(x_n,y_n)\}$ drawn i.i.d. from $\mathcal{P}_{\mathcal{X},\mathcal{Y}}$ any function $f\in\mathcal{F}$ satisfies
$$\er_{\mathcal{P}_{\mathcal{X},\mathcal{Y}}}(f)-\er_{S_n}(f)\leq \sigma(n,\delta,d).$$
Similarly, $\er_{S_n}(f)-\er_{\mathcal{P}_{\mathcal{X},\mathcal{Y}}}(f)\leq \sigma(n,\delta,d)$ for the same assumptions.\\ $\sigma(n,\delta,d)$ is given by $\sigma(n,\delta,d):=2\sqrt{\frac{2d\log(2ne/d)+\log(2/\delta)}{n}}$.

\end{lemma}

Similar to the generalization of PAC learning to agnostic PAC learning, online learning can also be extended to the non-realizable case. The error is now quantified using the concept of \textit{Regret}, which reflects how much the learner regrets, in hindsight, not following the predictions of the optimal hypothesis:

\begin{definition}[Regret]
    Let $\mathcal{F}$ be a binary hypothesis class over $\mathcal{X}$ and time horizon $T\in \mathbb{N}^+$.\\
    Given an online learning algorithm $L$ who sequentially predicts on samples $\{(x_1,y_1),\dots,(x_T,y_T)\}$ with predictions $\hat{y}_1,\dots,\hat{y}_T$, we define the Regret at time $T$ with respect to the best classifier in $\mathcal{F}$ as 
$$\text{Regret}_L(\mathcal{F},T)=\sum_{i=1}^T \mathbbm{1}_{\{\hat{y}_i=1-y_i\}}-\min_{f\in\mathcal{F}}\sum_{i=1}^T\mathbbm{1}_{\{f(x_i)=1-y_i\}}.$$ 

\end{definition}

Like agnostic PAC learning, this measures the relative error compared to the best classifier, whereas in the realizable setting, we were able to bound the absolute error. We say that a hypothesis class is agnostic online learnable, if we achieve Regret sublinear in the number of examples $T$. We assume, that the true labels $y_i$ cannot depend on the learner's prediction, else vanishing regret would not be achievable \cite{cover1965BehaviorSequentialPredictors}. Moreover, the learning algorithm has to be \textit{randomized} \cite{shalev-shwartz2014UnderstandingMachineLearning}. \\

One of the first to address this problem with a randomized algorithm was Littlestone and Warmuth in 1994, introducing the concept of the weighted majority algorithm for finite hypothesis classes
\cite{littlestone1994WeightedMajorityAlgorithma}. Similar models have been studied extensively \cite{cesa-bianchi1997HowUseExpert,cesa-bianchi1996OnlinePredictionConversion}.
The theoretical analysis of similar algorithms aims to find a combinatorial quantity (like Littlestone dimension or generalizations of it) to quantify optimal error bounds. Discussions on this can be found in e.g.\cite{ben-david2009AgnosticOnlineLearning} or   \cite{filmus2023OptimalPredictionUsing}. Even more in the agnostic setting (due to its relevance w.r.t. real-life data), there is a broad field of research trying to establish learning rates for the optimization problem with linear and kernel based methods \cite{cesa-bianchi2011OnlineLearningNoisy,natarajan2013LearningNoisyLabels}, also investigating different loss functions, feedback communicated back to the learner (full information setting or bandit) and assumptions on noise models. 

\subsubsection{On Agnostic Active Learning}

The literature on active learning contains results and methods closely related to the approach we will discuss later in Section \ref{sec:results_agno}. The most important aspect, in which sequential/online active learning differs from our setting, is that the learner can only observe the true label, if it is requested. Moreover, usually one studies active learning in the stochastic setting, assuming that input and labels are drawn i.i.d. from $\mathcal{P}_{\mathcal{X},\mathcal{Y}}$.\\

To give an overview over the field of agnostic active learning, we start with the \textit{pool based} setting, as this was the first active learning algorithm to be introduced for the general agnostic setting. This was conceived by  Balcan et al in 2006 \cite{balcan2006AgnosticActiveLearning}, who derived the $A^2$ (for Agnostic Active) Algorithm, with the aim to generalize the disagreement-based CAL Algorithm (see Algorithm \ref{alg:cal}). It is easy to see, that in the case of noisy or random labels, even the best classifier in the class might issue a wrong prediction on this instance. Thus relying on the label of one sample and removing all classifiers whose label predictions disagree is not a sensible approach. However, we still want to reliably reduce the region of disagreement without removing the best classifier.\\
In order to achieve this, the $A^2$ Algorithm does the following: in round $i$ of the algorithm, we request $S_i$  labels of samples from the current region of disagreement. Calculating high-confidence upper and lower bounds of $\er_{\mathcal{P}_{\mathcal{X},\mathcal{Y}}}(h)$, the true error of a hypothesis $h$, based on the samples $S_i$ can be done through results on Uniform convergence (similar to Lemma \ref{lem:unif_conv}). Then those hypothesis are eliminated for which the calculated lower bound is greater than the minimal upper bound, as they are have a significantly higher error rate compared to the best ones in the class. Note that we choose the number of label requests $S_i$ such that with high probability with each update we halve the current region of disagreement.\\

They were the first to study the most general setting of agnostic active learning, using only the assumption of i.i.d. sampled data. For additional assumptions on the problem class, the noise and distribution, they even showed exponential rates in label complexity, i.e. $\mathcal{O}(\log_2(1/\varepsilon))$ label requests to find a classifier with error less than $\varepsilon$. A comprehensive analysis, quantifying the label complexity of this algorithm in terms of the disagreement coefficient $\theta_f$, was first conducted in 2007 by Hanneke \cite{hanneke2007BoundLabelComplexity}. This work laid the foundation for subsequent analyses of active learning algorithms, which can be often discribed in terms of $\theta_f$.\\

The $A^2$ algorithm sparked the development of active learning algorithms in the general agnostic setting requiring minimal assumptions on the distribution.
It was followed by the algorithm presented by Dasgupta et al. in \cite{dasgupta2007GeneralAgnosticActive} (referred to as DHM after the authors). Their approach ensures computational feasibility and leverages the labels of samples which can be inferred and match the target with high probability ('pseudo labels'), thus representing a semi-supervised approach in active learning.\\
The DHM Algorithm keeps track of two different sets of labeled samples $S$ (where we have inferred the label) and $T$ (where the label was explicitly requested). In order to decide whether to request or infer the label of a sample $x_t$, we study the classifier $f$ which is consistent with labels in $S$ and has minimal error on $T$. If the difference in empirical error on $S\cup T$ between the classifier $f^0$ restricted to $(x_t,0)$ compared to $f^1$ restricted to $(x_t,1)$ is large, we can confidently infer the label. We assign the label with minimal empirical error and add the sample to $S$. Else, we request the label and add it to $T$. Again the bounds for the difference in empirical error for some classifier are derived from uniform convergence results as presented in Lemma \ref{lem:unif_conv}. This comparison of performance for classifiers when restricting it to either choice of label in $\mathcal{Y}$ is similar to what we do in the realizable setting in algorithm \ref{alg:gen_vc}. Again, label complexity bounds in terms of the disagreement coefficient have been studied in \cite{dasgupta2007GeneralAgnosticActive} and \cite{hanneke2011RatesConvergenceActive}.\\

Taking inspiration from this DHM Algorithm and disagreement-based learning, Hanneke adapted in \cite{hanneke2012ActivizedLearningTransforming} his shattering algorithm to also improve guarantees in the agnostic setting. This is the algorithm we discussed in Section \ref{sec:active learning} and on which the Algorithm from Goel et al. is based upon. The main difference to the realizable setting is to collect a set of labeled points (inferred or requested) and update the hypothesis space to eliminate classifiers which perform significantly worse. How badly they perform is captured by the empirical error compared to that of the minimal empirical error that is achievable by a classifier in the version space. We will adopt this notion later in Section \ref{sec:results_agno} for more details. Classical results on uniform convergence (using data-dependent Rademacher complexities or Lemma \ref{lem:unif_conv}) give guarantees that we do not eliminate the best classifier whilst shrinking the region of disagreement by a considerable amount. \\
The choice of predicting or inferring is based on the probability of shattering, same as in the realizable case.\\

He assumed Tsybakov-Mammen noise condition (see Section \ref{sec:noise}) and managed to give the following guarantees:

\begin{theorem}
    Let $\mathcal{F}$ be a binary hypothesis space over $\mathcal{X}$, $\mathcal{P}_{\mathcal{X},\mathcal{Y}}$ some distribution over $\mathcal{X}\times \mathcal{Y}$, and $f^*\in \mathcal{F}$ is the target (Bayes classifier) we want to learn, i.e. $\er_{\mathcal{P}_{\mathcal{X},\mathcal{Y}}}(f^*)=\min_{f\in\mathcal{F}}\er_{\mathcal{P}_{\mathcal{X},\mathcal{Y}}}(f)=\nu(\mathcal{F})$. Moreover, $\VCdim(\mathcal{F})=d$, $k=\Tilde{d}_{f^*}$ (see Section \ref{sec:active learning}) and the Tsybakov-Mammen noise condition holds with parameters $\kappa=\frac{1+\alpha}{\alpha}$ and $\mu$.\\
    Then, there exist a constant $c>0$ which is dependent on the distribution $\mathcal{P}_{\mathcal{X}}$ and $\mathcal{F}$, such that for all $\varepsilon,\delta\in (0,e^{-3})$ and integer $n$ with 
     $$n\geq c \theta_{f^*}^{k}(\varepsilon^{1/\kappa})\varepsilon^{2/\kappa -2}\log_2(1/(\varepsilon\delta))$$
    we can give the following guarantee for the Shattering Algorithm: The classifier $f_n\in\mathcal{F}$ which is the output of the Shattering Algorithm with $n$ label requests achieves with probability at least $1-\delta$ that $\er_{\mathcal{P}_{\mathcal{X},\mathcal{Y}}}(f_n)\leq \nu(\mathcal{F})+\varepsilon$. 
\end{theorem}

A more recent line of work has managed to show exponential rates in pool-based active learning, by allowing the algorithm to abstain \cite{puchkin2021ExponentialSavingsAgnostic}. This brings us back to the concept of selective classification, a concept closely linked to active learning, as we have previously established. We can establish an equivalence between requesting a label in active learning, whenever it cannot be inferred with confidence, and abstaining on a prediction when we cannot predict with confidence, also in the agnostic setting.

The first to study the problem of selective classification in the agnostic setting from a theoretical viewpoint were El-Yaniv and Wiener in 2011 \cite{wiener2011AgnosticSelectiveClassification} who introduced a simple extension of their CSS classifier to the agnostic setting.\\
We introduce the following definition of pointwise-competitive: A selective classifier $(f,g)$ is called \textit{pointwise-competitive}, if for every $x\in\mathcal{X}$ such that $g(x)=1$ (i.e. we choose to predict), we have that $f(x)=f^*(x)$ where $f^*$ is a true risk minimizer \cite{wiener2015AgnosticPointwiseCompetitiveSelective}.\\
In the realizable setting, this was trivially the case, as we predicted with the consistent hypothesis, thus never making a mistake. This becomes relevant in the agnostic setting, where we with our predictions we want to match the target $f^*$ and not the noise. \\

The approach to agnostic selective classification is tied to the methods we have seen in agnostic active learning: classifiers are eliminated from the hypothesis space, if they have higher empirical error compared to the best in the version space. Guarantees to not eliminating the best classifier $f^*$, while achieving high coverage, can be issued through guarantees of uniform convergence (Lemma \ref{lem:unif_conv}). Choosing to predict when all classifiers agree on the predicted label, can be shown to lead to a pointwise-competitive classifier, even in the agnostic setting \cite{wiener2015AgnosticPointwiseCompetitiveSelective}.\\

Note that in the El-Yaniv and Wiener's first studies of agnostic classification in \cite{wiener2011AgnosticSelectiveClassification,wiener2015AgnosticPointwiseCompetitiveSelective}, they make an assumption on $\mathcal{F}$ and the 0-1 loss to be of Bernstein type (for details see \cite{bartlett2004LocalComplexitiesEmpirical}), which can be removed by considering an improved version of the algorithm \cite{gelbhart2019RelationshipAgnosticSelective}.
A comprehensive analysis of the relationship between coverage bounds for a pointwise-competitive learner in selective classification, label complexity bounds for a disagreement-based learner in sequential active learning, and the disagreement coefficient in the agnostic setting can be found in \cite{gelbhart2019RelationshipAgnosticSelective}.\\

To mention in the context of abstentions, the 'Knows What It Knows' Algorithm can also be adapted for the agnostic setting \cite{szita2011AgnosticKWIKLearning} and to move beyond disagreement-based predictions, one can study confidence rated predictions using abstentions to get good bounds \cite{zhang2014DisagreementBasedAgnosticActive}. Lastly, the algorithm presented in \cite{desalvo2021OnlineActiveLearning}  shows guarantees for arbitrary loss functions through a surrogate and the use of pseudolabels (inferred labels) for  sequential active learning, thus combining active learning and semi-supervised learning. \\

To conclude, we present results in online selective sampling that are closely related to what we present in Section \ref{sec:results_agno} and the bounds we aim to derive. To reiterate, online active learning is similar to sequential active learning, however with the objective of establishing bounds on expected misclassification error and the expected number of queries \cite{hanneke2021GeneralTheoryOnline}. As mentioned before, there is a close connection (established through selective classification) between the number of queries a sequential active learner makes and the number of abstentions in selective classification (which is the objective we will focus on later).\\
In the work of Huang et al. \cite{huang2020DisagreementbasedActiveLearning,huang2020SEQUENTIALDECISIONMAKING}, they study a disagreement-based online selective sampling algorithm, subject to Tsybakov Noise. The algorithm operates as follows: similar to what we have seen before in pool based active learning, we wait until we have collected a sample of a fixed size and then eliminate those classifiers which make significantly more mistakes than the best classifier in the set. Corresponding guarantees to not eliminate the best classifier and a control by how much the disagreement region shrinks with each update of the version space can be given through the uniform convergence results from Lemma \ref{lem:unif_conv}. For time horizon $T$, with this method they achieve an expected misclassification bound of order $\mathcal{O}(1)$ and expected number of label requests of order $\mathcal{O}(dT^{\frac{2-2\alpha}{2-\alpha}}\log^2(T))$, where $\alpha$ is the parameter of Tsybakov Noise and $\VCdim(\mathcal{F})=d$. See \cite{cacciarelli2024ActiveLearningData} for a comprehensive study of sequential active learning algorithms. 

\newpage

\section{Extension to the Agnostic Setting}\label{sec:results_agno}

In the realizable setting, we could put our full trust in the label of a single sample, influencing the update of the hypothesis space after each new prediction. In contrast, in the agnostic setting, we encounter a new layer of uncertainty due to random labels. It is evident that updating the hypothesis space after every label is no longer a suitable approach.\\

As already discussed for agnostic active learning algorithms in Section \ref{sec:lit_agno} (when there is no clean-label adversary) there is a natural fix:
\begin{itemize}
    \item If there is information about the noise (like e.g., Random Classification Noise or Massart Noise), then we can leverage this information to collect a bigger sample and remove classifiers which perform 'badly' (we will quantify this later).
    \item An application of the uniform convergence property can give necessary guarantees to make sure we are not eliminating the target whilst reducing the version space by a necessary amount.
\end{itemize}

Following this approach, Hanneke has already shown in his work \cite{hanneke2012ActivizedLearningTransforming}, that going from disagreement-based to shattering, does not only provide improvements in label complexity for the realizable setting, but also in the agnostic setting (with Tsybakov noise condition) with \textit{stochastic data}. The central question, we ask ourselves now:

\begin{center}
    Is it possible to design an algorithm, which is robust to clean-label injections \textbf{and} label noise, utilizing the same methods as in Section \ref{sec:results_real} from \cite{goel2023AdversarialResilienceSequential}?
\end{center}

To explore this question, we first focus on a \textit{disagreement-based} learner for the simple class of thresholds. This can be seen as a first exploration of this new problem of learning in the presence of noise with a clean-label adversary. Pursuing this approach to generalize it to hypothesis classes with VC dimension $d$ and then following the same approach as with shattering seems feasible. However, guaranteeing desirable error rates, turns out to be a significantly harder task, due to the added complexity through noise.

\subsection{Disagreement-Based Learning of Thresholds}

In Algorithm \ref{alg:agnostic_disagree}, we describe the algorithm for a disagreement-based version of learning thresholds on the interval $[0,1]$ with an adversary injecting clean-label samples and Random Classification Noise. 
The idea remains the same as in the realizable setting: the learner chooses to 
\begin{itemize}
    \item Predict, if all hypothesis from the current version space agree with their prediction on the sample.
    \item Abstain, if the sample falls in the disagreement region of the current version space.
\end{itemize}
However, the question is now: how and at what point do we update the version space? Previously, this happened after every sample, as classifiers which assigned the wrong label were necessarily not the target. In the case of random labels, this no longer holds. We adopt the strategy of common agnostic (active) learning algorithms (see for example \cite{dasgupta2007GeneralAgnosticActive,hanneke2012ActivizedLearningTransforming}). This strategy involves initially collecting a sample of size $M$ and then eliminating classifiers based on their difference in empirical errors to the classifier from the current version which achieves minimal error on this set. Classifiers which deviate a lot and produce a lot of errors are then less likely to be the target $f^*$ (as this is the Bayes optimal classifier) and are eliminated.\\ 
Again, we are faced with a trade-off between misclassification and abstention error:
\begin{enumerate}
    \item On one hand, in order to minimize the misclassification error, with high probability, we want to make sure not to eliminate the target $f^*$, such that this function continues to contribute correct labels in subsequent rounds. \label{trade-off:Mis}
    \item On the other hand, in order to minimize the abstention error and thus the region of disagreement in general, with high probability, we want to make sure we eliminate classifiers making a lot of mistakes.\label{trade-off:Abs}
\end{enumerate}

To ensure, that we remove just enough classifiers that both the above points can be satisfied, will constitute the primary challenge in the proof. This will determine the choice of $\Delta$ in step \ref{alg_line:update} of the Algorithm.\\

We can give the following guarantees for misclassification and abstention error:

\begin{theorem}\label{theo:mis_abs_noisy}
    Let $\mathcal{F}=\{\mathbbm{1}_{\{\cdot\leq a\}}:a\in [0,1]\}$ be the class of thresholds on the interval $\mathcal{X}=[0,1]$ and $\mathcal{D}$ some known marginal distribution over $\mathcal{X}$. Furthermore, we have the true labelling function $f^*\in \mathcal{F}$ and Random Classification Noise (see Section \ref{sec:noise}) with parameter $\eta<1/2$.\\
    Running Algorithm \ref{alg:agnostic_disagree} with time horizon $T$, $M=\log(T)\frac{16}{(1-2\eta)^2}$ and $\Delta=\frac{1-2\eta}{2}$, we achieve the following upper bounds on expected misclassification and abstention error
    \begin{equation*}
        \mathbb{E}(\MCE)\leq 1 \quad \mathbb{E}(\AbE)\leq 6M\cdot \frac{\log(T/(12M) + 3/2)}{\log(3/2)}+1 
    \end{equation*}
\end{theorem}

Note that we get the following asymptotics in terms of $T$ when ignoring the noise terms: $\mathbb{E}(\MCE)\in \mathcal{O}(1)$ and $\mathbb{E}(\AbE)\in  {\mathcal{O}(\log^2(T))}$. Comparing this with results in the realizable setting (refer to Theorem \ref{theo:mis_abs_thresh}), shows that the bounds we achieve for the noisy problem are not significantly worse than those in the realizable setting.


\begin{algorithm}[H]
    \caption{Disagreement-based learning in the agnostic setting for thresholds $\mathcal{F}=\{\mathbbm{1}_{\{\cdot\leq a\}}:a\in [0,1]\}$ with known marginal distribution $\mathcal{D}$ over $\mathcal{X}=[0,1]$ }\label{alg:agnostic_disagree}
    \begin{algorithmic}[1]
        \State Set $S_0=\emptyset, i=1, \mathcal{F}_0=\mathcal{F}$
        \State Set $a_0=\min\{x\in \mathcal{X}:x\in \DIS(\mathcal{F}_0)\}$, $b_0=\max\{x\in \mathcal{X}:x\in \DIS(\mathcal{F}_0)\}$
        \State Choose $c_0^-$ and $c_0^+$ such that \par
        $P_{x\sim \mathcal{D}}(x\in [a_0,c_0^-])=P_{x\sim \mathcal{D}}(x\in [c_0^-,c_0^+])=P_{x\sim \mathcal{D}}(x\in [c_0^+,b_0])=\frac{\rho_1(\mathcal{F}_{0})}{3}$
        \State Set $D_0=[c_0^-,c_0^+]$
        \For{$t=1, \dots, T$}
        \State Receive $\hat{x}_t\in \mathcal{X}$
        \If{$\hat{x}_t\in \DIS(\mathcal{F}_{i-1})$}
        \State $\hat{y}_t=\perp$
        \Else 
        \State Predict $\hat{y}_t=f(\hat{x}_t)$ according to some $f\in\mathcal{F}_{i-1}$
        \EndIf        
        \State Receive the true label $y_t$ 
        \State Update $S_{i-1}\gets S_{i-1}\cup \{(\hat{x}_t,y_t)\}$
        \If{$|S_{i-1}\cap D_{i-1}|=M$} \label{alg_line:when_update}
        \State Set $D_{i-1}^M=S_{i-1}\cap D_{i-1}$
        \State Update version space such that error is minimal \par $\mathcal{F}_{i}\gets\{h\in \mathcal{F}_{i-1}: \text{er}_{D_{i-1}^M}(h)- \min_{f\in \mathcal{F}_{i-1}}\text{er}_{D_{i-1}^M}(f)\leq \Delta \}$ \label{alg_line:update}
        \State Set $S_i \gets \emptyset$
        \State Set $a_i=\min\{x\in \mathcal{X}:x\in \DIS(\mathcal{F}_i)\}$, $b_0=\max\{x\in \mathcal{X}:x\in \DIS(\mathcal{F}_i)\}$
        \State Choose $c_i^-$ and $c_i^+$ such that \par
        $P_{x\sim \mathcal{D}}(x\in [a_i,c_i^-])=P_{x\sim \mathcal{D}}(x\in [c_i^-,c_i^+])=P_{x\sim \mathcal{D}}(x\in [c_i^+,b_i])=\frac{\rho_1(\mathcal{F}_{i})}{3}$
        \State Set $D_i=[c_i^-,c_i^+]$
        \State Update $i\gets i+1$
        \EndIf
        \EndFor
    \end{algorithmic}
\end{algorithm}

Before we begin proving Theorem \ref{theo:mis_abs_noisy}, we present some results crucial for our argumentation. The proofs are omitted, as these are well-known results.

\begin{lemma}[Hoeffding's inequality]\label{lem:hoeff}
    Let $X_1,\dots,X_n$ be independent random variables with $a_i\leq X_i\leq b_i$, for all $i\in \{1,\dots,n\}$ almost surely.\\
    Then for all $t>0$, the following inequalities hold
    $$P(\sum_{i=1}^n (X_i - \mathbb{E}(X_i) )\geq t)\leq \exp(-\frac{2t^2}{\sum_{i=1}^n(b_i-a_i)^2})$$
    and 
    $$P(\sum_{i=1}^n (X_i - \mathbb{E}(X_i) )\leq  -t)\leq \exp(-\frac{2t^2}{\sum_{i=1}^n(b_i-a_i)^2}).$$
\end{lemma} 

\begin{lemma}[Union Bound]\label{lem:union_bound}
    Let $A_1,\dots,A_n$ be a collection of events.\\
    Then the following holds:
    $$P(\bigcup_{i=1}^n A_i)\leq \sum_{i=1}^n P(A_i).$$
\end{lemma}

The following two Lemmas \ref{lem:not_remove_target} and \ref{lem:drittel} will determine the choice of $\Delta$ and $M$.

\begin{lemma} \label{lem:not_remove_target}
    Let $\mathcal{F}=\{\mathbbm{1}_{\{\cdot\leq a\}}:a\in [0,1]\}$ be the class of thresholds on the interval $\mathcal{X}=[0,1]$ and $\mathcal{D}$ some known marginal distribution over $\mathcal{X}$. Furthermore, we have the true labelling function $f^*\in \mathcal{F}$ and Random Classification Noise (see Section \ref{sec:noise}) with parameter $\eta<1/2$. We have time horizon $T\in \mathbb{N}^+$ and a given round $i\in \mathbb{N}^+$.\\
    Then, the update step \ref{alg_line:update} in round $i$ of Algorithm \ref{alg:agnostic_disagree} does not remove the target hypothesis $f^*\in \mathcal{F}$ with probability at least $1-1/T^2$ if $\Delta>2\sqrt{\frac{\log(T)}{M}}+(2\eta-1)$. This means with probability at least $1-1/T^2$ the following holds
    $$\er_{D_{i-1}^M}(f^*)-\min_{f\in \mathcal{F}_{i-1}}\er_{D_{i-1}^M}(f)<\Delta.$$
\end{lemma}

\begin{proof}
    We set $f_M^*:=\arg\min_{f\in \mathcal{F}_{i-1}}$ and $D_M^f:=\DIS(\{f^*,f\})\cap D_{i-1}^M$ which denotes the points, on which some $f\in\mathcal{F}_{i-1}$ and $f^*$ disagree. \\
    First note for all $t\in \{0,\dots,T\}$ that given $\hat{x}_t$ the labels $\hat{y}_t$ are drawn i.i.d. according to 
    $$\hat{y}_t=\begin{cases} f^*(\hat{x}_t) & \text{ with probability } 1-\eta \\
    1-f^*(\hat{x}_t) & \text{ with probability } \eta.
    \end{cases}$$
    Thus, when $\hat{x}_t\in D_M^{f_M^*}$, i.e. $f^*(\hat{x}_t)\neq f_M^*(\hat{x}_t)$ (or more explicitly $f^*(\hat{x}_t)=1-f_M^*(\hat{x}_t)$) the following holds
    $$\mathbbm{1}_{\{f^*(\hat{x}_t)=1-\hat{y}_t\}}-\mathbbm{1}_{\{f_M^*(\hat{x}_t)=1-\hat{y}_t\}}=\begin{cases} 1 & \text{ with probability } \eta \\
    -1 & \text{ with probability } 1-\eta\end{cases}$$
    which is a bounded random variable with values in $\{-1,1\}$. 
    Note that we assume the samples $\hat{x}_t$ are given and the randomness is over the i.i.d. labels $y_t$ (due to the clean-label adversary).\\
    Given $\hat{x}_t\notin D_M^{f_M^*}$, this difference is equal to zero, as $f^*$ and $f_M^*$ agree on $\hat{x}_t$. We can calculate the expectation $\mathbb{E}(\mathbbm{1}_{\{f^*(\hat{x}_t)=1-\hat{y}_t\}}-\mathbbm{1}_{\{f_M^*(\hat{x}_t)=1-\hat{y}_t\}})=1\cdot \eta -1\cdot (1-\eta)=2\eta-1 <0$ for $\hat{x}_t\in D_M^{f_M^*}$. Now we can bound the probability
    {\allowdisplaybreaks
    \begin{align*}
        P(\er_{D_{i-1}^M}(f^*)&-\er_{D_{i-1}^M}(f_M^*)\geq\Delta) =P(\frac{1}{M}\sum_{(\hat{x}_t,\hat{y}_t)\in D_{i-1}^M}\mathbbm{1}_{\{f^*(\hat{x}_t)=1-\hat{y}_t\}}-\mathbbm{1}_{\{f_M^*(\hat{x}_t)=1-\hat{y}_t\}}\geq \Delta)\\
        =&P(\sum_{(\hat{x}_t,\hat{y}_t)\in D_{i-1}^M}\mathbbm{1}_{\{f^*(\hat{x}_t)=1-\hat{y}_t\}}-\mathbbm{1}_{\{f_M^*(\hat{x}_t)=1-\hat{y}_t\}}\geq M\Delta)\\
        =&P(\sum_{(\hat{x}_t,\hat{y}_t)\in D_M^{f_M^*}}\mathbbm{1}_{\{f^*(\hat{x}_t)=1-\hat{y}_t\}}-\mathbbm{1}_{\{f_M^*(\hat{x}_t)=1-\hat{y}_t\}}\geq M\Delta)\\
        =&P(\sum_{(\hat{x}_t,\hat{y}_t)\in D_M^{f_M^*}}\mathbbm{1}_{\{f^*(\hat{x}_t)=1-\hat{y}_t\}}-\mathbbm{1}_{\{f_M^*(\hat{x}_t)=1-\hat{y}_t\}}-|D_M^{f_M^*}|(2\eta-1)\geq M\Delta-|D_M^{f_M^*}|(2\eta-1))\\
        \leq& \exp(-\frac{2(M\Delta-|D_M^{f_M^*}|(2\eta-1))^2}{4|D_M^{f_M^*}|})
    \end{align*}}
    where the inequality is due to Hoeffding's inequality \ref{lem:hoeff}, as $|D_M^{f_M^*}|(2\eta-1)$ is the expectation of the sum and $M\Delta-|D_M^{f_M^*}|(2\eta-1)>0$. We get that this is less than $1/T^2$ by choosing a $\Delta$ such that (note $|D_M^{f_M^*}|\leq M$) $$\frac{2}{M}\sqrt{|D_M^{f_M^*}|\log(T)}+\frac{|D_M^{f_M^*}|}{M}(2\eta-1)\leq 2\sqrt{\frac{\log(T)}{M}}+(2\eta-1)\leq \Delta.$$
    The claim then follows by taking the complement. 
\end{proof}

\begin{lemma}\label{lem:drittel}
    Let $\mathcal{F}=\{\mathbbm{1}_{\{\cdot\leq a\}}:a\in [0,1]\}$ be the class of thresholds on the interval $\mathcal{X}=[0,1]$ and $\mathcal{D}$ some known marginal distribution over $\mathcal{X}$. Furthermore, we have the true labelling function $f^*\in \mathcal{F}$ and Random Classification Noise (see Section \ref{sec:noise}) with parameter $\eta<1/2$. We have time horizon $T\in \mathbb{N}^+$ and round $i\in \mathbb{N}^+$. Furthermore we define $D_M^f:=\DIS(\{f^*,f\})\cap D_{i-1}^M$ the set of samples in the middle third of the disagreement region, on which some $f\in \mathcal{F}$ and the target $f^*$ disagree.\\
    Then, for the choice of $M=\log(T)\frac{16}{(1-2\eta)^2}$ we have that with probability at least $1-1/T^2$ that the update step \ref{alg_line:update} in round $i$ of Algorithm \ref{alg:agnostic_disagree} removes all hypothesis $f\in \mathcal{F}$ which disagree with $f^*$ on more than $M/2$ samples. This means, that with probability at least $1-1/T^2$ we get that for all $f\in \mathcal{F}$ with $|D_M^f|>M/2$ that 
    $$\er_{D_{i-1}^M}(f)-\min_{f\in \mathcal{F}_{i-1}}\er_{D_{i-1}^M}(f)\geq\Delta.$$
    We claim that this reduces the region of disagreement by at least 1/3 of its size.
\end{lemma}

\begin{proof}
    First, we define $f_M^*:=\arg\min_{f\in \mathcal{F}_{i-1}}$ and note $\er_{D_{i-1}^M}(f^*)-\er_{D_{i-1}^M}(f_M^*)\geq 0 $ which holds by definition of the minimal empirical error. Thus, giving us an upper bound
    {\allowdisplaybreaks
    \begin{align*}
        P(\er_{D_{i-1}^M}(f)-&\er_{D_{i-1}^M}(f_M^*)<\Delta)\\
        &=P(\er_{D_{i-1}^M}(f)-\er_{D_{i-1}^M}(f^*)+\er_{D_{i-1}^M}(f^*)- \er_{D_{i-1}^M}(f_M^*)<\Delta)\\
        &<P(\er_{D_{i-1}^M}(f)-\er_{D_{i-1}^M}(f^*) <\Delta).
    \end{align*}}
    We continue by bounding the probability that the difference in empirical error of $f$ and $f^*$ is small. As in the proof of Lemma \ref{lem:not_remove_target}, we first look at how the random labels differ in the disagreement region.
    Again we assume the samples $\hat{x}_t$ are given and the randomness is over the i.i.d. labels $y_t$ (due to the clean-label adversary).
    This time we get for a given $\hat{x}_t\in D_M^f$ that 
     $$\mathbbm{1}_{\{f(\hat{x}_t)=1-\hat{y}_t\}}-\mathbbm{1}_{\{f^*(\hat{x}_t)=1-\hat{y}_t\}}=\begin{cases} 1 & \text{ with probability } 1-\eta \\
    -1 & \text{ with probability } \eta \end{cases}$$ 
    which is again a bounded random variable with values in $\{-1,1\}$ but with expectation $\mathbb{E}(\mathbbm{1}_{\{f(\hat{x}_t)=1-\hat{y}_t\}}-\mathbbm{1}_{\{f^*(\hat{x}_t)=1-\hat{y}_t\}})=1\cdot(1-\eta)-1\cdot\eta=1-2\eta>0$, given $\hat{x}_t\in D_M^f$ (else the difference is equal to 0). Similar as in the previous Lemma, we can now bound the probability of this difference being small with a Hoeffding  bound (see Lemma \ref{lem:hoeff}).
    {\allowdisplaybreaks
    \begin{align*}
        P(\er_{D_{i-1}^M}(f)&-\er_{D_{i-1}^M}(f_*)<\Delta) =P(\frac{1}{M}\sum_{(\hat{x}_t,\hat{y}_t)\in D_{i-1}^M}\mathbbm{1}_{\{f(\hat{x}_t)=1-\hat{y}_t\}}-\mathbbm{1}_{\{f^*(\hat{x}_t)=1-\hat{y}_t\}}<\Delta)\\
        =&P(\sum_{(\hat{x}_t,\hat{y}_t)\in D_{i-1}^M}\mathbbm{1}_{\{f(\hat{x}_t)=1-\hat{y}_t\}}-\mathbbm{1}_{\{f^*(\hat{x}_t)=1-\hat{y}_t\}}< M\Delta)\\
        =&P(\sum_{(\hat{x}_t,\hat{y}_t)\in D_M^{f}}\mathbbm{1}_{\{f(\hat{x}_t)=1-\hat{y}_t\}}-\mathbbm{1}_{\{f^*(\hat{x}_t)=1-\hat{y}_t\}}< M\Delta)\\
        =&P(\sum_{(\hat{x}_t,\hat{y}_t)\in D_M^{f}}\mathbbm{1}_{\{f(\hat{x}_t)=1-\hat{y}_t\}}-\mathbbm{1}_{\{f^*(\hat{x}_t)=1-\hat{y}_t\}}-|D_M^{f}|(1-2\eta)< M\Delta-|D_M^{f}|(1-2\eta))\\
        \leq& \exp(-\frac{2(M\Delta-|D_M^{f}|(1-2\eta))^2}{4|D_M^{f}|})\\
        \leq& \exp(-\frac{(M\Delta-M(1-2\eta))^2}{2M})=\exp(-M\frac{(\Delta-(1-2\eta))^2}{2})
    \end{align*}}
    Note, that in order to apply the Hoeffding bound, it must hold that $$M\Delta-|D_M^f|(1-2\eta)<0 \Leftrightarrow \Delta<\frac{|D_M^f|(1-2\eta)}{M}$$ which gives us a condition on the choice of $\Delta$. If we choose $\Delta=\frac{1-2\eta}{2}$ then the above condition is only satisfied, if $|D_M^f|>M/2$, which holds exactly for the hypothesis, we aim to eliminate. \\
    Plugging in this choice for $\Delta$ gives an upper bound of $\exp(-M\frac{(1-2\eta)^2}{8})$. Furthermore, for $M=\log(T)\frac{16}{(1-2\eta)^2}$ we get the desired $1/T^2$ and taking the complement establishes that with high probability, we eliminate all $f\in\mathcal{F}_{i-1}$ for which $|D_M^f|>M/2$.\\
    
    Note that the goal is to eliminate those hypothesis from the version space, which are far from the target and make a lot of mistakes. Due to the construction of $D_{i-1}$ being the middle third of the disagreement region and the hypothesis class being thresholds, the target's $f^*$ true threshold has to fall in the $\DIS(\mathcal{F}_{i-1})$ and at least $M/2$ samples are either on the left or on the right of this threshold. See Figure \ref{fig:thirds} . We want to eliminate those classifiers $f\in \mathcal{F}_{i-1}$, which have their threshold in the left- or rightmost third of the disagreement region, depending which one is further away from the true threshold. And those are exactly the ones which disagree with $f^*$ on more than $M/2$ points. Therefore, with each update, we eliminate at least those classifiers, which have their threshold in the left- or rightmost third, effectively reducing the region of disagreement by $1/3$ of its probability mass with high probability. This means with probability $1-1/T^2$ we get 
    \begin{equation*}       
    \rho_1(\mathcal{F}_{i-1})\frac{2}{3}\geq \rho_1(\mathcal{F}_{i}).
    \end{equation*}
    
\end{proof}

\begin{remark}
    We want to give a bit more intuition about the partitioning of the disagreement region in thirds in Algorithm \ref{alg:agnostic_disagree}. This particular design is such that we can guarantee that with each update, we eliminate a least one third of the disagreement region:\\
    Note that in round $i$ of the algorithm, the three intervals $[a_i,c_i^-], [c_i^-,c_i^+]$ and $[c_i^+,b_i]$ cover the disagreement region $\DIS(\mathcal{F}_{i-1})$ and they are chosen, as such to have the same probability mass. The learner waits with the update of $\mathcal{F}_{i-1}$, until $M$ points fell in the interval $[c_i^-,c_i^+]$ and based on the empirical errors on this set, eliminates hypothesis in Step \ref{alg_line:update}. Why do we require, that $M$ points (the set of labeled points $D_{i-1}^M$) fall in an interval which contains $1/3$ of the probability mass of the disagreement region?\\
    This is due to the fact, that we cannot guarantee that the samples are i.i.d. and thus cover the space according to the distribution $\mathcal{D}$. However, if we enforce that $M$ points fall in the interval $[c_i^-,c_i^+]$, due to the \textit{structure} of thresholds on the line, we can quantify, that a bad hypothesis disagrees with $f^*$ on at least $M/2$ samples. Like we argue in the proof of Lemma \ref{lem:drittel}, eliminating hypothesis like this, results in a reduction of the probability mass in the disagreement region by a third. \\
    Another option would be to have a fixed number of samples to be processed in every round, until we update, i.e. $|S_{i-1}|=M$ in Step \ref{alg_line:when_update}. However, then because of the adversarial examples, we cannot quantify easily by how much we shrink the probability mass in the region of disagreement with each update (note that in the case of i.i.d. samples, one can use uniform convergence). Introducing this cutting into thirds allows us to make the connection between removing bad predictors and the probability mass associated with removing them. \\
\end{remark}

\begin{figure}[H]
    \centering
\begin{tikzpicture}
  \begin{axis}[
    x=\XS*0.7,y=1.5cm,
    axis lines=middle, 
    axis y line=none,
    axis line style={-}, 
    xlabel={$x$}, 
    ylabel={$y$},
    xlabel style={at={(ticklabel* cs:1.0)}, anchor= west}, 
    ylabel style={at={(ticklabel* cs:1.0)}, anchor= south},
    ymin=-1, 
    ymax=2, 
    xmin=-2, 
    xmax=14, 
    xtick={0,4,8,12}, 
    xticklabels={$a_i$,$c_i^-$,$c_i^+$,$b_i$},
    major tick length=10pt, 
    minor tick length=5pt,
    tick style={black, thick},
  ]
    \addplot+[const plot, no marks, draw=black!70, thick] coordinates {(-2,0) (9,0) (9,1) (14,1)} node[below,pos=.85,black] {$f$};
   \addplot+[const plot, no marks, draw=black, thick] coordinates {(-2,0) (3,0) (3,1) (14,1)} node[above,pos=0.18,black] {$f^*$};

    \addplot [only marks,mark=*,mark size=2.5pt, mark options={fill=white}] table {
    4 1
    4.75 1
    5 1
    5.5 0
    6.25 1
    6.5 1
    7.0 1
    7.5 0
    7.75 1
};

    \draw[fill=black, opacity=0.1, draw=none] (8,0) rectangle (12,1);
     \draw [decorate,decoration={brace,amplitude=3pt,raise=2ex}]
  (0,1) -- (12,1) node[midway,yshift=+1.8em]{\small $\DIS(\mathcal{F}_{i-1})$};

  \end{axis}
  \end{tikzpicture}
  \begin{tikzpicture}
  \begin{axis}[
    x=\XS*0.7,y=1.5cm,
    axis lines=middle, 
    axis y line=none,
    axis line style={-}, 
    xlabel={$x$}, 
    xlabel style={at={(ticklabel* cs:1.0)}, anchor= west}, 
    ymin=-1, 
    ymax=2, 
    xmin=-2, 
    xmax=14, 
    xtick={0,4,8,12}, 
    xticklabels={$a_i$,$c_i^-$,$c_i^+$,$b_i$},
    major tick length=10pt, 
    minor tick length=5pt,
    tick style={black, thick},
  ]
    \addplot+[const plot, no marks, draw=black!50, thick] coordinates {(-2,0) (1,0) (1,1) (14,1)} node[below,pos=0.32,black] {$f$};
   \addplot+[const plot, no marks, draw=black, thick] coordinates {(-2,0) (8.5,0) (8.5,1) (14,1)} node[below,pos=0.83,black] {$f^*$};

    \addplot [only marks,mark=*,mark size=2.5pt, mark options={fill=white}] table {
    4 0
    4.75 0
    5 0
    5.5 1
    6.25 0
    6.5 0
    7.0 0
    7.5 1
    7.75 0
};

    \draw[fill=black, opacity=0.1, draw=none] (0,0) rectangle (4,1);
     \draw [decorate,decoration={brace,amplitude=3pt,raise=2ex}]
  (0,1) -- (12,1) node[midway,yshift=+1.8em]{\small $\DIS(\mathcal{F}_{i-1})$};

  \end{axis}
\end{tikzpicture}
\vspace{0.25cm}

  \begin{tikzpicture}
  \begin{axis}[
    x=\XS*0.7,y=1.5cm,
    axis lines=middle, 
    axis y line=none,
    axis line style={-}, 
    xlabel={$x$}, 
    xlabel style={at={(ticklabel* cs:1.0)}, anchor= west}, 
    ymin=-2, 
    ymax=2, 
    xmin=-2, 
    xmax=14, 
    xtick={0,4,8,12}, 
    xticklabels={$a_i$,$c_i^-$,$c_i^+$,$b_i$},
    major tick length=10pt, 
    minor tick length=5pt,
    tick style={black, thick},
  ]
    \addplot+[const plot, no marks, draw=black!50, thick] coordinates {(-2,0) (9,0) (9,1) (14,1)} node[below,pos=0.83,black] {$f$};
   \addplot+[const plot, no marks, draw=black, thick] coordinates {(-2,0) (6,0) (6,1) (14,1)} node[above,pos=0.24,black] {$f^*$};

    \addplot [only marks, mark=*, mark size=2.5pt,mark options={fill=white}] table {
    4 0
    4.75 0
    5 0
    5.5 1
    6.25 1
    6.5 1
    7.0 1
    7.5 0
    7.75 1
};

    \draw[fill=black, opacity=0.1, draw=none] (8,0) rectangle (12,1);    
    \draw [decorate,decoration={brace,amplitude=3pt,mirror,raise=4ex}]
  (0,0) -- (12,0) node[midway,yshift=-3em]{\small $\DIS(\mathcal{F}_{i-1})$};

  \end{axis}
  \end{tikzpicture}
  \begin{tikzpicture}
\node[anchor=east] at (1,0.2) {$\circ$};
\node[anchor=west] at (1.25,0.2) {\small Samples $\hat{x}_t$ in $[c_i^-,c_i^+]$ with noisy labels $y_t\in\{0,1\}$ collected in round $i$ };
\draw[fill=black, opacity=0.1, draw=none] (0,-0.4) rectangle (1,-0.1);
\node[anchor=west,align=left] at (1.25,-0.25) {\small Region where points with disagreement to $f^*$ \small is larger than $M/2$ points (out of $M$) };
\end{tikzpicture}
\vspace{0.25cm}
    \caption{An example of how the target threshold disagrees with 'bad' hypothesis on more than $M/2$ points, where $M=9$.
    From the examples of the first row, we see that if the target $f^*$ has its threshold within the first interval $[a_i,c_i^-]$, respectively the last $[c_i^+,b_i]$, then functions $f\in\mathcal{F}_{i-1}$ with their threshold in the last interval $[c_i^+,b_i]$, respectively the first  $[a_i,c_i^-]$, disagree with $f^*$ on \textit{all} $M=9$ points in $[c_i^-,c_i^+]$, i.e. $f(\hat{x}_t)\neq f(\hat{x}_t)$.\\
    In the second row we see an example, when the threshold of the target $f^*$ falls within the interval $[c_i^-,c_i^+]$. In this case, on at least one side of the threshold, there are than $M/2$ samples. For this example, we have $5>M/2$ samples right of the threshold. Thus, the goal is that all hypothesis with thresholds in the shaded area on the right, should be eliminated with high probability, as they would classify these $5$ points incorrectly (compared to the target). Note that 2 points are noisy (opposite label of $f^*$) in every example. }
    \label{fig:thirds}
\end{figure}

We are now ready to prove the main theorem, allowing us to bound the misclassification and abstention error: 

\begin{proof}[Proof of Theorem \ref{theo:mis_abs_noisy}]\label{proof:mis_abs_noisy}
    \textit{Misclassification error}: First, we note that for this choice of $M$ and $\Delta$, we get that $\Delta=\frac{1-2\eta}{2}>2\sqrt{\frac{\log(T)}{M}}+(2\eta-1)$, as $(2\eta-1)<0$.\\
    We have shown in Lemma \ref{lem:not_remove_target}, that for some fixed round $i$ we have
    \allowdisplaybreaks{\begin{align*}
        P(\er_{D_{i-1}^M}(f^*)-&\min_{f\in \mathcal{F}_{i-1}}\er_{D_{i-1}^M}(f)>\Delta)\\
        &\leq P(\er_{D_{i-1}^M}(f^*)-\min_{f\in \mathcal{F}_{i-1}}\er_{D_{i-1}^M}(f)>2\sqrt{\frac{\log(T)}{M}}+(2\eta-1))\leq 1/T^2
    \end{align*}}
    thus removing the target hypothesis in round $i$ with probability $1/T^2$.\\
    
    By a union bound (see Lemma \ref{lem:union_bound}), with probability at most $1/T$ there exists an $i\in\{1,\dots,T\}$ such that $\er_{D_{i-1}^M}(f^*)-\min_{f\in \mathcal{F}_{i-1}}\er_{D_{i-1}^M}(f)>\Delta$. Therefore, with probability $1-1/T$, we do not remove $f^*$ from the version space in any round $i$.\\
    Note, that in case $f^*\in \mathcal{F}_i$ for all $i$, we have $\MCE=0$, as if we choose to predict on any sample $\hat{x}_t$ in round $i$, then $\hat{x}_t\notin \DIS(\mathcal{F}_{i-1})$ and all functions in the version space agree on $\hat{x}_t$. We predict with some $f\in \mathcal{F}_{i-1}$ and as $f^*\in \mathcal{F}_{i-1}$ for all rounds $i$, $\hat{y}_t=f(\hat{x}_t)=f^*(\hat{x}_t)$. This gives for all $t$: $$\mathbbm{1}_{\{\hat{y}_t=1-y_t\}}-\mathbbm{1}_{\{f^*(\hat{x}_t)=1-y_t\}}=\mathbbm{1}_{\{f^*(\hat{x}_t)=1-y_t\}}-\mathbbm{1}_{\{f^*(\hat{x}_t)=1-y_t\}}=0$$
    which implies $\MCE=0$.
    As we can guarantee $f^*\in \mathcal{F}_i$ for all $i\in \{0,\dots,T\}$ with probability at least $1-1/T$, we get $P(\MCE=0)\geq 1-1/T$. \\
    Now a simple calculation using the law of total expectation shows (use abbreviation of misclassification error as \text{MisE}$_T$):
    \allowdisplaybreaks{
    \begin{align*}
        \mathbb{E}(\text{MisE}_T)&=P(\text{MisE}_T\leq 0)\cdot \mathbb{E}(\text{MisE}_T|\text{MisE}_T\leq 0) + 
        P(\text{MisE}_T> 0)\cdot \mathbb{E}(\text{MisE}_T|\text{MisE}_T> 0)\\
        &\leq  P(\text{MisE}_T >0)\cdot T= (1- P(\text{MisE}_T \leq 0))\cdot T\\
        &=(1- P(\text{MisE}_T=0)-P(\text{MisE}_T < 0))\cdot T\leq (1-(1-1/T))\cdot T=1/T\cdot T=1
    \end{align*}}
    where the first inequality follows as $\mathbb{E}(\text{MisE}_T|\text{MisE}_T\leq 0)\leq0$ and $\mathbb{E}(\text{MisE}_T|\text{MisE}_T> 0)\leq T$, as the misclassification error is at most $T$.\\
    
    \textit{Abstention error}:

    Given the observed samples $\{\hat{x}_1,\dots,\hat{x}_T\}$ until time $T$, we assume that among these, there have been $m$ samples drawn i.i.d. from $\mathcal{D}$: $\{\hat{x}_{i_1},\dots,\hat{x}_{i_m}\}$ with $0\leq m\leq T$ and $1\leq i_1\leq \dots \leq i_m\leq T$. For simplicity, we set $\Tilde{x}_j=\hat{x}_{i_j}$ for all $j\in\{1,\dots,m\}$ with the corresponding predictions $\Tilde{y}_j=\hat{y}_{i_j}$. Assume that when running the algorithm, we have undergone $N$ rounds of updates of $\mathcal{F}_i$, $N\leq T$. Denote the number of i.i.d. samples in round $i$ by $K_i$ for $i\in\{0,\dots,N\}$ and set $K_0=0$. Note, that $K_1,\dots, K_n$ as well as the number of rounds $N$ are random and depend on how many samples were needed, to get $M$ samples in the middle third of the disagreement region. \\

    Now we can rewrite:
    \allowdisplaybreaks{
    \begin{align*}
        \mathbb{E}(\AbE)&=\mathbb{E}(\sum_{t=1}^T\mathbbm{1}_{\{x_t \text{ was drawn i.i.d. from }\mathcal{D} \text{ and } \hat{y}_t=\perp\}})\\
        &=\mathbb{E}(\sum_{i=1}^m \mathbbm{1}_{\{\Tilde{y}_t=\perp\}})\\
        &=\mathbb{E}(\sum_{i=1}^N \sum_{j=1}^{K_i} \mathbbm{1}_{\{\Tilde{y}_{t_i +j}=\perp\}})\quad \text{ where } t_i=\sum_{l=1}^{i-1}K_l\\
        &=\mathbb{E}(\sum_{i=1}^N \sum_{j=1}^{K_i}    \mathbbm{1}_{\{ \Tilde{x}_{t_i +j} \in \DIS(\mathcal{F}_{i-1})\}})\\
        \end{align*}}

    To obtain an upper bound on the abstention error, we study its behaviour in the worst case, which is the setting for which all points are indeed i.i.d. (i.e. $m=T$) as we can abstain on adversarial points for free. To make this rigorous, we argue through a reduction argument. Assume we have two Algorithms $\mathcal{A}$ and $\mathcal{A}'$ which are given the same stream of data $\{\hat{x}_1,\dots, \hat{x}_T\}$. Assume the algorithm $\mathcal{A}'$ has access to an oracle which can identify which $m$ samples $\hat{x}_i$ were sampled i.i.d. from $\mathcal{D}$ and which were injected. The injected points are subsequently completely removed from consideration and $\mathcal{A}'$ is then equal to Algorithm \ref{alg:agnostic_disagree} but runs \textit{only} on the $m$ i.i.d. points, while $\mathcal{A}$ still processes all $T$ points. Now it holds
    \begin{equation}\label{eq:reduction}
    \mathbb{E}(\AbE \text{ of }\mathcal{A})\leq \mathbb{E}(\AbE \text{ of }\mathcal{A}').
    \end{equation}
    This is because of the following: As we only pay in abstention error if we abstain on i.i.d. points, both algorithms have abstention error at most $m$. However in $\mathcal{A}$, we have potentially more steps where we update the hypothesis space, as the adversarial points might only help to reach $M$ points in $[c_i^-,c_i^+]$ faster (or they have no influence at all).  With each update, the disagreement region shrinks by $1/3$ of its size and with that the possibility that an i.i.d point falls in $[c_i^-,c_i^+]$ also shrinks. Comparing the probability of $\Tilde{x}_j$ falling in the current disagreement region is thus smaller for Algorithm $\mathcal{A}$ than for $\mathcal{A}'$ for all $j\in\{1,\dots,m\}$. As we have $\mathbb{E}(\AbE)=\mathbb{E}(\sum_{i=1}^m \mathbbm{1}_{\{\Tilde{y}_t=\perp\}})=\sum_{i=1}^m P_{x\sim\mathcal{D}}(x \text{ falls in current disagreement region})$, the inequality from Equation \eqref{eq:reduction} holds. Furthermore, using $m=T$ gives an upper bound for the expected abstention error of $\mathcal{A}$, which is the quantity we want to bound.\\
    
    Next, we understand how $K_i$ determines the number of rounds $N$. Note that it has to hold that $\sum_{i=1}^N K_i=m = T$ (where we use $m=T$ by the previous argumentation). We want to find a lower bound on $K_i$ which guarantees that with high probability at least $M$ samples fell into the middle third in round $i$. 
    We assume that we wait until a \textit{fixed} amount of i.i.d. samples $\Tilde{K}_i$ is collected in each round. Then, we want to guarantee with high probability, that we induce an update of $\mathcal{F}_{i-1}$. By our choice of $M\in\mathcal{O}(\log(T))$, it is easy to see using a Hoeffding bound that given we wait until $\Tilde{K}_i=6M/\rho_1(\mathcal{F}_{i-1})$, then with probability $1-1/T^2$ there are $M$ points in $[c_i^-,c_i^+]$ (note that $\mathbb{E}(K_i)=3M/\rho_1(\mathcal{F}_{i-1})$).\\

    Assume that we wait until $\Tilde{K}_i=6M/\rho_1(\mathcal{F}_{i-1})$ i.i.d. samples are collected and then we perform the update. As we have seen, with high probability this guarantees that there are actually at least $M$ samples in $[c_i^-,c_i^+]$. By a union bound, we have on an event $E$ with probability $1-1/T$, that in all rounds there are at least $M$ samples in the middle third. Using Theorem \ref{lem:drittel} gives that on an event $G$ with probability $1-1/T$ (by a union bound) we have that $$\rho_1(\mathcal{F}_{i})\leq \frac{2}{3}\rho_1(\mathcal{F}_{i-1})\dots \leq (\frac{2}{3})^i$$ 
    for all rounds $i$, where we assume $\DIS(\mathcal{F}_0)=\mathcal{X}=1$. This gives on an event $E\cap G$ the following equation that we want to solve for the number of rounds $N$:
    \begin{equation}\label{eq:solve_N_2}
        \sum_{i=1}^N \frac{6M}{(\frac{2}{3})^i}=T.
    \end{equation}
     Solving this geometric sum for $N$ gives us $N=\frac{\log(T/(12M) + 3/2)}{\log(3/2)}$.\\
     On the event $(E\cap G)^C$ with probability less than $1/T$ we can find an upper bound for $N$ through using $\Tilde{K}_i\geq 6M$ the trivial bound $N\leq T/6M$.

This allows us to finally bound the abstention error by using the law of total expectation (note that $\Tilde{K}_i$ is fixed and thus not a random variable):
\allowdisplaybreaks{
    \begin{align*}
        \mathbb{E}(\sum_{i=1}^N& \sum_{j=1}^{\Tilde{K}_i} \mathbbm{1}_{\{\Tilde{y}_{t_i +j}=\perp\}})\\
        =\mathbb{E}(&\sum_{i=1}^N \sum_{j=1}^{\Tilde{K}_i} \mathbbm{1}_{\{\Tilde{y}_{t_i +j}=\perp\}}|E\cap G)\cdot P(E\cap G)\\
        &+\mathbb{E}(\sum_{i=1}^N \sum_{j=1}^{\Tilde{K}_i} \mathbbm{1}_{\{\Tilde{y}_{t_i +j}=\perp\}}|(E\cap G)^C)\cdot P((E\cap G)^C)\\
        &\leq \sum_{i=1}^{\frac{\log(T/12M + 3/2)}{\log(3/2)}}\sum_{j=1}^{\Tilde{K}_i} \mathbb{E}(\mathbbm{1}_{\{\Tilde{y}_{t_i +j}=\perp\}})+\sum_{i=1}^{T/(6M)}\sum_{j=1}^{\Tilde{K}_i} \mathbb{E}(\mathbbm{1}_{\{\Tilde{y}_{t_i +j}=\perp\}})\cdot 1/T\\
        &=\sum_{i=1}^{\frac{\log(T/12M + 3/2)}{\log(3/2)}}\Tilde{K}_i\cdot\rho_1(\mathcal{F}_{i-1})+\sum_{i=1}^{T/(6M)}\Tilde{K}_i\cdot\rho_1(\mathcal{F}_{i-1})\cdot 1/T\\
        &\leq \frac{\log(T/12M + 3/2)}{\log(3/2)}\cdot 6M + \frac{T}{6M}\cdot 6M\cdot \frac{1}{T}\\
        &= 6M\cdot \frac{\log(T/12M + 3/2)}{\log(3/2)}+1.
    \end{align*}
    }
where we used that on $E\cap G$, respective $(E\cap G)^C$, we have found an upper bound on $N$. Moreover, $P(E\cap G)\leq 1$ and $P((E\cap G)^C)\leq 1/T$.

\end{proof}

\begin{remark}
    The extension from Random Classification Noise to Massart Noise is straight-forward. To recall, the label of $\hat{x}_t$ is given by $f^*(\hat{x}_t)$ with probability $1-\eta(\hat{x}_t)$ and $1-f^*(\hat{x}_t)$ with probability $\eta(\hat{x}_t)$, where $\eta(x)\leq \eta<1/2$. Thus we can use the upper bound $$2\eta(x)-1\leq 2\eta-1$$ and use this throughout the proofs of Lemma \ref{lem:not_remove_target} and \ref{lem:drittel}, which will give the same bounds on the choice of $M$ and $\Delta$.
\end{remark}

\begin{remark}
    Using the Structure Theorem \ref{theo:structure} one could exploit the linear structure of the disagreement region for general VC classes 1 and derive a similar disagreement-based method, as we have presented in Algorithm \ref{alg:agnostic_disagree}. Note that we do rely on this linear order of the disagreement region quite heavily in this algorithm (especially in the partition into thirds and the guarantee of not eliminating $f^*$). An  extension to general VC dimension $d>1$ would require a different approach due to increased complexity.
\end{remark}

\newpage

\subsection{Beyond Disagreement-Based Learning for Thresholds }\label{sec:beyond_disagree}

The previous example of disagreement-based learning of thresholds in one dimension, provides us with an understanding of the added difficulties, introduced by a clean-label adversary. The primary difficulty lies in updating the version space. Following a similar strategy as demonstrated in \cite{goel2023AdversarialResilienceSequential}, transitioning from a basic disagreement-based learner to the framework presented in Algorithm \ref{alg:gen_vc}, is our goal in this agnostic setting as well. In Algorithm \ref{alg:agnostic_disagree_ver2} we draft a possible version of this algorithm, which still needs some adaptions in order to guarantee desirable bounds on misclassification and abstention error. This will be discussed below. 

\begin{algorithm}[H]
    \caption{Beyond disagreement-based learning in the agnostic setting for thresholds $\mathcal{F}=\{\mathbbm{1}_{\{\cdot\leq a\}}:a\in [0,1]\}$ with known marginal distribution $\mathcal{D}$ over $\mathcal{X}=[0,1]$}\label{alg:agnostic_disagree_ver2}
    \begin{algorithmic}[1]
        \State Set $S_0=\emptyset, i=1, \mathcal{F}_0=\mathcal{F}$
        \For{$t=1, \dots, T$}
        \State Receive $\hat{x}_t\in \mathcal{X}$

        \If{$\rho_1(\mathcal{F}_{i-1})\geq \alpha$}
        
        \If{$\min(\rho_1(\mathcal{F}_{i-1}^{\hat{x}_t\rightarrow 0}),\rho_1(\mathcal{F}_{i-1}^{\hat{x}_t\rightarrow 1}))\geq 0.6\rho_1(\mathcal{F}_{i-1})$} \label{alg:step_abstain_agno}
        \State Predict $\hat{y}_t=\perp$
        \Else 
        \State Predict $\hat{y}_t=\arg \max _{j\in \{0,1\}}\rho_1(\mathcal{F}_{i-1}^{\hat{x}_t\rightarrow j})$\label{alg:step_predict_agno}
        \EndIf         

        \Else 
        \If{$\hat{x}_t\in\DIS(\mathcal{F}_{i-1})$}
        \State Predict $\hat{y}_t=\perp$
        \Else
        \State Predict $\hat{y}_t=f(\hat{x}_t)$ according to some $f\in\mathcal{F}_{i-1}$
        \EndIf
        \EndIf
        
        \State Receive the true label $y_t$ 
        \State Update $S_{i-1}\gets S_{i-1}\cup \{(\hat{x}_t,y_t)\}$
        \If{$|S_{i-1}|=M$}
        \State Update version space such that error is minimal 
        \par $\mathcal{F}_{i}\gets\{h\in \mathcal{F}_{i-1}: \text{er}_{S_{i-1}}(h)- \min_{f\in \mathcal{F}_{i-1}}\text{er}_{S_{i-1}}(f)\leq \Delta \}$
        \State Set $S_i \gets \emptyset$
        \State Update $i\gets i+1$
        \EndIf
        \EndFor
    \end{algorithmic}

\end{algorithm}

This Algorithm differs from Algorithm \ref{alg:agnostic_disagree} in the abstention and prediction steps (see Line \ref{alg:step_abstain_agno} and \ref{alg:step_predict_agno}) when $\rho_1(\mathcal{F}_{i-1})$ is larger than some constant $\alpha$. This is based on the design of Algorithm \ref{alg:gen_vc}, which as shown in Section \ref{sec:results_real}, is in general VC dimension $d$ a robust learning algorithm. For learning thresholds on the interval $[0,1]$ the following holds:
\begin{itemize}
    \item If $\hat{x}_t \notin\DIS(\mathcal{F}_{i-1})$, then $\exists j\in\{0,1\}$ such that $\rho_1(\mathcal{F}_{i-1}^{\hat{x}_t\rightarrow 1-j})=0$, where $f(\hat{x}_t)=j$ for all $f\in \mathcal{F}_{i-1}$. This is because if all functions agree on $\hat{x}_t$, there exists no function $g\in\mathcal{F}_{i-1}$ such that $g(\hat{x}_t)=1-j$, therefore $\mathcal{F}_{i-1}^{\hat{x}_t\rightarrow 1-j}=\emptyset$. Thus, as long as $\rho_1(\mathcal{F}_{i-1}^{\hat{x}_t\rightarrow j})=\rho_1(\mathcal{F}_{i-1})>0=\rho_1(\mathcal{F}_{i-1}^{\hat{x}_t\rightarrow 1-j})$, we choose to predict with the consistent label and incur no misclassification error, if the target $f^*$ is in $\mathcal{F}_{i-1}$.
    \item If $\hat{x}_t \in \DIS(\mathcal{F}_{i-1})$ and $\rho_1(\mathcal{F}_{i-1}) \geq \alpha$, then due to the structure of thresholds on $[0,1]$, we have the following $$\DIS(\mathcal{F}_{i-1}^{\hat{x}_t\rightarrow 0})\cup \DIS(\mathcal{F}_{i-1}^{\hat{x}_t\rightarrow 1})=\DIS(\mathcal{F}_{i-1})\setminus\{\hat{x}_t\}$$
    as the first two quantities are disjoint intervals (due to the linear order on $[0,1]$). From this it follows that $$\rho_1(\mathcal{F}_{i-1}^{\hat{x}_t\rightarrow 0})+\rho_1(\mathcal{F}_{i-1}^{\hat{x}_t\rightarrow 1})\leq\rho_1(\mathcal{F}_{i-1})$$
    with equality if $P_{x\sim\mathcal{D}}(x=\hat{x}_t)=0$.
    Therefore, $\min(\rho_1(\mathcal{F}_{i-1}^{\hat{x}_t\rightarrow 0}),\rho_1(\mathcal{F}_{i-1}^{\hat{x}_t\rightarrow 1}))\leq 0.5\rho_1(\mathcal{F}_{i-1})$ and the condition in Step \ref{alg:step_abstain_agno} is \textit{never} satisfied. We will \textit{never} abstain and always choose to predict, according to the label $j$ which maximizes $\rho_1(\mathcal{F}_{i-1}^{\hat{x}_t\rightarrow j})$.\\
    \item If $\hat{x}_t \in \DIS(\mathcal{F}_{i-1})$ and $\rho_1(\mathcal{F}_{i-1}) < \alpha$ , by choosing $\alpha=1/T$, we can bound the expected abstention error by a constant.
    \end{itemize}
    
The way we designed Algorithm \ref{alg:agnostic_disagree_ver2} with $\alpha=1/T$, we have an absolute bound on the abstention error, as we never abstain as long as $\rho_1(\mathcal{F}_{i-1}) \geq \alpha$. Only when the probability for $x\sim\mathcal{D}$ to fall in $\DIS(\mathcal{F}_{i-1})$ is small, we have the possibility to abstain. In this case, as the probability for an i.i.d point to fall in $\DIS(\mathcal{F}_{i-1})$ is small, we can bound the expected abstention error: $\mathbb{E}(\AbE)\leq 1$.\\

The main difference to the realizable setting is the misclassification error. It is unclear, how we should choose to approach the step of updating the hypothesis space. The method, we derived earlier in Algorithm \ref{alg:agnostic_disagree}, to wait until we have collected a sample of size $M$ in the middle third of the disagreement region, would in this case be a weak point. A clean-label adversary could inject an arbitrary amount of samples either in $[a_i,c_i^-]$ or $[c_i^+,b_i]$, keeping us from updating the hypothesis space and thus denying us to learn from mistakes. In the realizable setting, after every misclassification, we were able to update the hypothesis space. Based on our prediction strategy, the probability of a point falling in the disagreement region reduces by a factor of at least $0.6$. \\

So the key this time seems to be, to collect a sample of fixed size $M$ or wait until we have made a fixed amount of misclassifications in the disagreement region (accounting for noise of course) such that after an updating step we can be certain to reduce the region of disagreement by a fixed amount, like it was possible in the realizable setting in Algorithm \ref{alg:gen_vc} or for disagreement-based learning in Algorithm \ref{alg:agnostic_disagree}. This remains to be subject of further considerations. \\

In general, this approach appears feasible, especially when considering a specific problem class like thresholds, where wan can use structural properties of the class. The trade-off between misclassification error and abstention error offers a powerful tool to protect against challenges in the classification problem. As clean-label injections still provide an opportunity to learn something about the problem class, we have to find ways to circumvent the distributional assumptions we usually impose.

\newpage

\chapter{Conclusion}

\section{Next Steps}
The extension to the agnostic setting was discussed as a possible next direction of the strategies presented in \cite{goel2023AdversarialResilienceSequential}. We have been the first to analyze this from a theoretical viewpoint in this thesis with the conception of an adversary injecting clean-label samples in the agnostic setting. Chapter \ref{chap:agno} can be considered as an initial exploration of this subject. In these sections, we discuss possible approaches for handling the update of the hypothesis space, which is a challenging aspect of such an algorithm. Based on these explorations, we can identify concrete next steps to derive an adversarial robust algorithm in the presence of noise:

\begin{itemize}
    \item Extension to more general noise models: In Algorithm \ref{alg:agnostic_disagree_ver2}, we heavily depend on the assumption of Random Classification Noise or Massart noise.\\
    Generally, it would desirable to obtain error bounds also in the case of Tsybakov (like for Hanneke's agnostic shattering algorithm \cite{hanneke2012ActivizedLearningTransforming}) or even more ambitiously the general agnostic case (like for the $A^2$ or DHM algorithm \cite{balcan2006AgnosticActiveLearning,dasgupta2007GeneralAgnosticActive}). 
    Deriving more general results appears to be a challenging task since we cannot rely on results of uniform convergence, due to the clean-label injections. Therefore, it might be necessary for the selection of $\Delta$ in the updating step, to be connected to noise parameters, in order to give guarantees.
    \item Going beyond disagreement-based learning: As was already discussed in Section \ref{sec:beyond_disagree}, it is conceivable that we can achieve similar guarantees for Algorithm \ref{alg:agnostic_disagree_ver2}. It is essential to define when and how to update the hypothesis class, in order to give good guarantees similar to the disagreement-based version of Algorithm \ref{alg:agnostic_disagree}. 
    \item Leveraging structural properties of more general hypothesis classes: We would hope to generalize Algorithm \ref{alg:agnostic_disagree} and \ref{alg:agnostic_disagree_ver2} in order to exploit structural properties of different hypothesis classes in a style similar to what was accomplished for thresholds. An extension to hypothesis classes of general VC dimension 1 and axis-aligned rectangles seems natural, as was done in Section \ref{sec:results_real}. Studying classes with finite VC dimension  would also allow us to use the proxy of the probability of shattering $k$ points in deciding whether to abstain or predict, instead of only simply using the disagreement region. 
    \item General classes with finite VC dimension: The overall goal would be to move away from the structural dependence of $\mathcal{F}$ and derive bounds for general VC dimension $d$, like for Algorithm \ref{alg:gen_vc}. The main issue is again that we cannot rely on uniform convergence results. 
\end{itemize}

\section{Conclusion}
In the first part of the thesis, we explore the context of the paper 'Adversarial Resilience in Sequential Prediction via Abstention' \cite{goel2023AdversarialResilienceSequential} with a thorough overview of previous results on active learning and disagreement-based methods, the use of abstention in classification tasks and previous study of adversarial clean-label injections. Based on this, we explore an algorithm from \cite{goel2023AdversarialResilienceSequential} based on the probability of shattering $k$ randomly drawn points, leveraging the knowledge of the underlying distribution. The algorithm is constructed in such a way, that in case of a false classification, we have high information gain. The learner is allowed to abstain in case he is uncertain. We attain desirable bounds in case of finite VC dimension for the expected number of misclassifications and abstentions on non-corrupted points.\\
In situations where the true distribution of samples is unknown, we conduct a thorough analysis of the methods outlined in \cite{goel2023AdversarialResilienceSequential} and present corrections to their argumentation. These methods are based on structural arguments of the hypothesis class to achieve similar bounds. More specifically, we focus on problem classes of general VC dimension 1 and axis-aligned rectangles, where we exploit their specific structure. \\

In the second part of the thesis, we aim to generalize the approach of \cite{goel2023AdversarialResilienceSequential} from the realizable to the more general  agnostic setting. The assumption of realizability, while simplifying theoretical analysis, does not accurately represent real-life data distribution, as perfect separability of data is rarely the case, thus motivating our analysis.  We introduce the concept of agnostic learning with clean-label injections, a topic that has not been theoretically studied in the existing literature (see \cite{blum2021RobustLearningCleanlabel}).
This exploration involves discussing previous results in agnostic active learning literature, reestablishing the connection to classification with an option to reject. For the case of known marginal distribution of samples, we succeed in deriving a disagreement-based learner for thresholds on an interval, which is robust towards a clean-label adversary subject to Random Classification Noise. This can be seen as a first step towards deriving more general strategies in the agnostic setting, similar to the method based on shattering from \cite{goel2023AdversarialResilienceSequential}.

\newpage


\newpage
\addcontentsline{toc}{chapter}{References}

\bibliographystyle{alpha}
\bibliography{ref}

\addcontentsline{toc}{chapter}{Declaration of Originality}
\includepdf{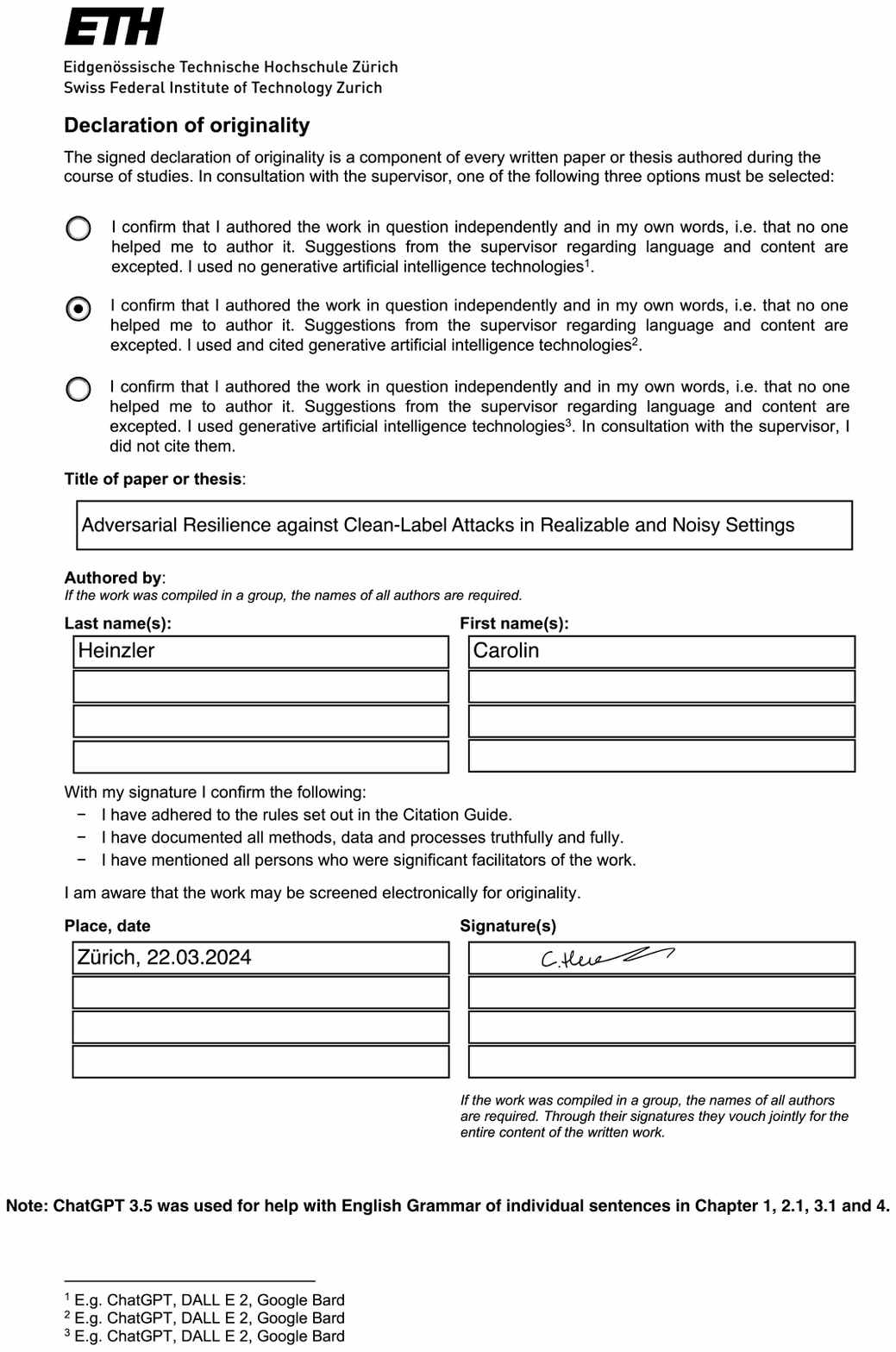}


\end{document}